\documentclass[10pt]{article} 

\usepackage[preprint]{tmlr}

\usepackage{amsmath,amsfonts,bm}

\def\eqref#1{equation~\ref{#1}}

\def\1{\bm{1}}

\DeclareMathAlphabet{\mathsfit}{\encodingdefault}{\sfdefault}{m}{sl}
\SetMathAlphabet{\mathsfit}{bold}{\encodingdefault}{\sfdefault}{bx}{n}

\usepackage{hyperref}
\usepackage{url}        
\usepackage{booktabs}       
\usepackage{amsfonts}       
\usepackage{nicefrac}       
\usepackage{microtype}      
\usepackage{xcolor}         
\usepackage{amsmath}
\usepackage{graphicx}
\usepackage{natbib}
\usepackage{wrapfig}  
\usepackage{bbding}
\usepackage{makecell}
\usepackage{amssymb}
\newtheorem{theorem}{Theorem}[section]

\newtheorem{corollary}[theorem]{Corollary}
\newtheorem{definition}[theorem]{Definition}
\newtheorem{assumption}[theorem]{Assumption}

\DeclareMathOperator*{\minimize}{\text{minimize}}

\usepackage{algorithm}
\usepackage{algpseudocode}
\usepackage{multicol}
\usepackage{tabularray}
\usepackage{graphicx}
\usepackage{caption}
\usepackage{subcaption}

\title{Learning Representations for Parameter-Free Retrieval in Continual Fine-Tuning of Pre-Trained Models}

\author{\name Hang Thi-Thuy Le \email hang.le1@wsu.edu \\
      \addr Washington State University
      \AND
      \name Long Minh Bui \email trongminhlong.bui@wsu.edu \\
      \addr Washington State University
      \AND
      \name Minh Hoang \email minhhoang@princeton.edu\\
      \addr Princeton University
      \AND
      \name Trong Nghia Hoang \email trongnghia.hoang@wsu.edu \\
      \addr Washington State University}

\def\month{MM}  
\def\year{YYYY} 
\def\openreview{\url{https://openreview.net/forum?id=XXXX}} 

\begin{document}

\maketitle

\begin{abstract}
Continual learning (CL) aims to adapt models to new tasks sequentially while preserving performance on earlier tasks whose data are no longer available.~Many state-of-the-art methods leverage a pre-trained backbone augmented with lightweight fine-tuning modules to balance representation adaptability and knowledge retention.~However, retrieval-based methods often retrieve fine-tuning modules at inference via continually learned retrieval functions that are themselves vulnerable to catastrophic forgetting. Alternatively, retrieval-free methods learn prediction mechanisms over a fixed pre-trained or pre-tuned representation space, avoiding forgetting induced by representation updates at the cost of adaptability.~To overcome this dilemma, we develop a novel continual learning framework that enables parameter-free retrieval by exploiting the distributional structures of task-adaptive representations.~Our key insight is that fine-tuning modules can be learned not only to acquire task-specific knowledge but also to induce representations that facilitate their own retrieval.~Specifically, each fine-tuning module is paired with a compact statistical signature of its induced embedding distribution, enabling likelihood-based retrieval without continually updating a shared retrieval function.~We theoretically characterize the conditions for accurate retrieval, showing that the per-input retrieval error decreases exponentially with representation separation, while the error of retrieving modules for single-task batches decreases exponentially with batch size under suitable conditions.~Guided by this analysis, we develop a continual fine-tuning method that promotes representation distinguishability while facilitating cross-task knowledge transfer.~Across a variety of class-incremental benchmarks with varying degrees of semantic shifts, our approach achieves state-of-the-art (SOTA) classification accuracy and the lowest average forgetting, while attaining highly competitive batch-wise retrieval accuracy without retraining or replaying past retrieval statistics.

\end{abstract}

\section{Introduction}
Continual learning (CL) aims to adapt a model to a sequence of tasks whose data arrive sequentially and cannot be fully revisited. Because each update on new data can overwrite knowledge acquired on earlier tasks, catastrophic forgetting is the central challenge~\citep{vandeven2022incremental, mccloskey1989catastrophic}. Classical approaches mitigate forgetting through regularization, which constrains updates to preserve behavior on earlier tasks~\citep{titsias2019functional,pan2020continual}, experience replay, which rehearses stored samples from past tasks~\citep{buzzega2020dark,cha2021co2l,chaudhry2019tiny}, or architectural expansion, which allocates dedicated capacity to each task~\citep{hung2019compacting,golkar2019continual}. However, these methods often update large portions of the model or rely on storing and revisiting past data, which raises memory and privacy concerns and limits their scalability to long task streams. 

The emergence of pre-trained foundation models~\citep{devlin2019bert, dosovitskiy2020image} has fundamentally reshaped this setting. Instead of learning representations from scratch, modern CL approaches keep a pre-trained backbone frozen and adapt it with lightweight, parameter-efficient modules such as prompts~\citep{lester2021prompt} or low-rank adapters~\citep{hu2021lora}. Since the backbone is never updated, its general-purpose representation is preserved, and restricting learning to these modules substantially reduces memory and computation over long task streams. As these modules accumulate across tasks, however, the model faces a new decision at inference: how to retrieve the most suitable adaptation module for each test input?

{\bf Challenge and Prior Approaches.}~This is a key challenge in class-incremental learning (CIL) settings where each new task introduces new classes and the task identities of test inputs are unknown during inference~\citep{kim2022theoretical}.~Existing methods have approached this challenge in two different ways:

{\bf Retrieval-free} methods often learn prediction mechanisms over a fixed pre-trained representation space, avoiding forgetting induced by representation updates at the cost of adaptability.~For example, RanPAC~\citep{DBLP:conf/nips/McDonnellGPAH23} learns a low-rank adaptation (LoRA) unit using data from the first task and tuses it to compute input and class embeddings across all tasks.~It then adopts a class prototype (CP)-based classification, which is employed to map unseen inputs to their most probable classes.~This retrieval step incurs no additional learnable parameters and is therefore immune to catastrophic forgetting.~However, its use of a pre-tuned set of adaptation parameters (using data from the first task) for inputs across all tasks will limit its adaptability in real-world applications involving significant semantic gaps across tasks~\cite{DBLP:conf/icml/KimYPLSBL24}.~Alternatively, InfLoRA~\citep{liang2024inflora} and SD-LoRA~\citep{wu2025sd} perform classification using a shared representation space induced by the accumulated fine-tuning modules, regardless of the task identity of each test input.

{\bf Retrieval-based} methods, such as L2P~\citep{wang2022learning}, DualPrompt~\citep{wang2022dualprompt}, CODA-Prompt~\citep{Smith_2023_CVPR}, and AdaPromptCL~\citep{DBLP:conf/icml/KimYPLSBL24}, instead retrieve fine-tuning prompts at inference via continually learned retrieval functions.~These retrievers are however vulnerable to catastrophic forgetting themselves.~Moreover, S-Prompt~\citep{wang2022sprompt} performed retrieval based on fixed pre-trained representations, which are not explicitly optimized to distinguish among tasks.~Instead, HiDe-Prompt~\citep{wang2023hide} propose to mitigate retriever forgetting via continually re-training the task-identity predictor using synthetic embeddings sampled from the stored class-conditional statistics of previously encountered tasks.~While effective, this requires repeated retraining and statistical replaying, and its retrieval accuracy remains dependent on the task-discriminative capacity of fixed pre-trained representations.

{\bf Solution Approach.}~To overcome both the vulnerability to retriever forgetting in retrieval-based methods and the limited representation adaptability of retrieval-free methods, we develop a novel continual fine-tuning framework that exploits the distributional structures of task-adaptive representations for input-dependent retrieval without learning a shared retrieval function.

Our key insight is that fine-tuning modules can be learned not only to acquire task-specific knowledge but also to induce representations that facilitate their own retrieval.~Concretely, each task-specific adapter can be structured to induce distinctive embedding distributions that reflect the semantics of its task, thereby facilitating the identification of the adapter itself.~We associate each fine-tuning module with a compact statistical signature of its induced embedding distribution, enabling retrieval by evaluating the likelihood of an input's embedding under each module's corresponding signature.~Since each signature is fitted using data from its own task alone, no task identifier needs to be retrained, no past statistics are replayed, and previously learned signatures remain unchanged as new tasks arrive.~At inference, test inputs arrive in single-task batches, and each input votes for the fine-tuning module whose signature best explains its induced embedding.~The module receiving the most votes is then used to make predictions for all inputs in the batch.

This also motivates a theoretical characterization of the conditions under which the proposed retrieval mechanism succeeds, establishing how retrieval accuracy depends on the separation between statistical signatures, the dimensionality of representations, and the number of inputs available for retrieval.~Guided by this characterization and the above insight, we develop a retrieval-aware fine-tuning method that balances cross-task knowledge transfer and representation distinguishability.~This allows each new adapter to selectively reuses relevant tuning directions learned from earlier tasks while discovering new directions that are encouraged to be orthogonal to existing ones.~The former facilitates knowledge transfer across related tasks, whereas the latter encourages the representation separation conducive to accurate retrieval.~Empirically, we show that (1) disabling transfer or removing its sparsity markedly lowers both final and retrieval accuracy; and (2) when use together, these components achieve the best final accuracy on all six benchmarks.

{\bf Technical Contribution.}~To substantiate the above solution approach, we develop and empirically validate a new continual fine-tuning method named \textsc{proteus} via the following contributions:

\begin{enumerate}
    \item \textbf{Parameter-free retrieval for CL with pre-trained models.}~We introduce a novel retrieval framework that uses the distributional structure of task-adaptive representations as statistical signatures to retrieve fine-tuning modules, without continually learning a shared retrieval function or replaying past statistics (Section~\ref{subsec:param-free}).

    \item \textbf{Characterization of retrieval success.}~Under mild regularity conditions, we show that the per-input retrieval error decays exponentially with signature separation, that the error on single-task batches decays exponentially with batch size, and that the separation sufficient for a target accuracy grows at most logarithmically with the number of tasks (Theorem~\ref{main_thm:1}, Corollary~\ref{corollary}, and Theorem~\ref{main_thm:2}).

    \item \textbf{Retrieval-Aware Fine-Tuning.}~Guided by our theory, we develop a novel retrieval-aware fine-tuning method that promotes cross-task knowledge transfer and representation distinguishability by selectively reusing prior tuning directions and adding new, orthogonal ones (Section~\ref{sec:adaptive_ft}).

    \item \textbf{Extensive Empirical Validation.}~On six class-incremental benchmarks with varying semantic shifts, our framework achieves state-of-the-art accuracy, up to 14.19 percentage points above the strongest baseline, and the lowest average forgetting, with competitive retrieval accuracy without retraining or replay (Section~\ref{sec:exprs}).
\end{enumerate}

\section{Problem Setup and Related Works}
\label{sec:problem}
\subsection{Problem Setup}
\label{subsec:setup}
Our research primarily focuses on class-incremental continual learning with a pre-trained backbone $h(\cdot\,; \boldsymbol{\omega}_0)$ that is frozen throughout. In this setting, tasks with datasets $\mathcal{D}_1, \ldots, \mathcal{D}_m$ arrive sequentially, where each $\mathcal{D}_k = \{(\boldsymbol{x}_i, y_i)\}$ contains labeled examples from a set of classes $\mathcal{C}_k$, with $\mathcal{C}_j \cap \mathcal{C}_k = \emptyset$ for $j \neq k$. 
At each step of the learning process, the learner has access only to the training data of the current task.~While training examples from previously encountered tasks cannot be stored or revisited, the learner is allowed to retain compact statistical summaries of these data to facilitate knowledge retention and subsequent adaptation.~The goal is to make accurate classification for each test input $\boldsymbol{x}_*$ among all seen classes $\mathcal{C}_1 \cup \cdots \cup \mathcal{C}_m$ without knowing the task it belongs to.

Let $\Delta\boldsymbol{\omega}_k$ denote the adaptation module (e.g., prompts or LoRA matrices) learned for the $k$-th task, $\boldsymbol{M}_k = (\Delta\boldsymbol{\omega}_1, \ldots, \Delta\boldsymbol{\omega}_k)$ the pool of modules up to task $k$, and $\boldsymbol{h}_k(\boldsymbol{x}) = h(\boldsymbol{x}; \Delta\boldsymbol{\omega}_k)$ the embedding of $\boldsymbol{x}$ induced by $\Delta\boldsymbol{\omega}_k$. Prediction uses an \textit{active adaptation set} chosen by a selection rule $g_{\boldsymbol{\psi}}$ with parameters $\boldsymbol{\psi}$,
\vspace{-1mm}
\begin{eqnarray}
\label{eq:selection-rule}
\boldsymbol{M}'(\boldsymbol{x}) &=& g_{\boldsymbol{\psi}}\big(\boldsymbol{x}, \boldsymbol{M}_k\big) \ \subseteq\ \boldsymbol{M}_k \ ,
\end{eqnarray}
which depends on the input in retrieval-based methods and is fixed in retrieval-free methods. Learning task $k$ involves training the new module and fitting the selection rule:
\vspace{-1mm}
\begin{eqnarray}
\label{eq:module-obj}
\min_{\Delta\boldsymbol{\omega}_k} \ L\big(\mathcal{D}_k, \boldsymbol{M}'_k\big) + R\big(\boldsymbol{M}'_k\big) \qquad \text{and} \qquad \min_{\boldsymbol{\psi}} \ \mathcal{J}\big(\boldsymbol{\psi}; \mathcal{D}_k, \boldsymbol{M}_k\big) \ ,
\end{eqnarray}
where $\boldsymbol{M}'_k \subseteq \boldsymbol{M}_k$, which includes $\Delta\boldsymbol{\omega}_k$, is the active set used during training, $L$ is the task loss, $R$ is an optional regularizer limiting interference among the modules in $\boldsymbol{M}'_k$, and $\mathcal{J}$ is a selection objective (e.g., task-identification loss, or a signature likelihood), with weighting coefficients absorbed into $R$ and $\mathcal{J}$.

Existing methods differ in how they schedule the two problems. \textbf{(a)} L2P~\cite{wang2022learning}, DualPrompt~\cite{wang2022dualprompt}, and CODA-Prompt~\cite{Smith_2023_CVPR} solve them jointly over $(\Delta\boldsymbol{\omega}_k, \boldsymbol{\psi})$. \textbf{(b)} HiDe-Prompt~\cite{wang2023hide} and S-Prompt~\cite{wang2022sprompt} fit $\boldsymbol{\psi}$ on frozen pre-trained embeddings, independently of the new module, so HiDe-Prompt trains its task identifier before $\Delta\boldsymbol{\omega}_k$. \textbf{(c)} \textsc{proteus} trains $\Delta\boldsymbol{\omega}_k$ first and then fits $\boldsymbol{\psi}$ in its embedding space $\boldsymbol{h}_k$. \textbf{(d)} Retrieval-free methods use a fixed rule, such as the first, the accumulated, or all modules, and have no selection objective. Moreover, $\boldsymbol{\psi}$ is either shared across tasks and updated at every task, as in (a) and HiDe-Prompt, so that selection for earlier tasks can drift unless explicitly preserved, a failure mode known as retriever forgetting, or task-local, $\boldsymbol{\psi} = (\boldsymbol{\psi}_1, \ldots, \boldsymbol{\psi}_k)$ with only $\boldsymbol{\psi}_k$ fit at task $k$, as in S-Prompt and \textsc{proteus}. At inference, we adopt a batch-wise setting in which test inputs arrive in single-task batches (Section~\ref{subsec:param-free}).

\subsection{Related Work}
\label{subsec:related}

\textbf{Retrieval-free methods for CL with pre-trained models.}~This line of work uses a fixed inference pathway without per-input retrieval (type (d) in Section~\ref{subsec:setup}), which avoids retriever forgetting at the cost of input-adaptive selection. For example,
RanPAC~\citep{DBLP:conf/nips/McDonnellGPAH23} adapts a parameter-efficient module on the first task only and keeps it frozen afterward, projecting the resulting features through a random nonlinear layer and classifying them by class-prototype (CP)-based classification with decorrelated prototypes, accumulated without revisiting past data.~InfLoRA~\citep{liang2024inflora} constrains each new LoRA module to a subspace designed to eliminate interference with earlier tasks, while SD-LoRA~\citep{wu2025sd} decouples the magnitude and direction of LoRA components, keeping earlier directions fixed and continuing to learn their magnitudes.~Both perform inference with a single model that accumulates adaptation over the task sequence.~O-LoRA~\citep{wang2023orthogonal}, designed for language models, learns each task's module in a subspace orthogonal to previous ones and uses all modules at inference. While these strategies differ in how they build representations, they share a fixed inference pathway that does not condition on the input, reducing flexibility under task heterogeneity.~As a consequence, RanPAC's first-task representations become misaligned as task diversity grows, while accumulated models risk biasing toward recent tasks and degrading performance on earlier tasks.

\textbf{Retrieval-based methods for CL with pre-trained models.}~This line of work maintains a pool of task-specific modules and selects among them at inference. L2P~\citep{wang2022learning} selects prompts from a shared pool by matching learnable keys to features of the frozen backbone.~DualPrompt~\citep{wang2022dualprompt} separates task-shared and task-specific prompts retrieved in the same way, CODA-Prompt~\citep{Smith_2023_CVPR} composes prompt components with input-dependent attention weights.~AdaPromptCL~\citep{DBLP:conf/icml/KimYPLSBL24} groups similar tasks under common prompts.~Their retrieval parameters are shared across tasks and updated as new tasks arrive (type (a)), so the selection learned for earlier tasks can drift, causing forgetting.

S-Prompt~\citep{wang2022sprompt} avoids such updates with frozen per-task K-means centroids of pre-trained features, but matching in the shared pre-trained space leaves tasks with similar inputs hard to distinguish.~HiDe-Prompt~\citep{wang2023hide} builds on the decomposition of class-incremental learning into within-task and task-identity prediction~\cite{kim2022theoretical} and retrains a dedicated task identifier after every task on pseudo-features sampled from cached class-level statistics (type (b)).~This largely mitigates retriever forgetting and achieves state-of-the-art performance among prior methods, but requires repeated retraining and replay.~Since the task is inferred through class prediction, it does not directly extend to settings where classes recur across tasks.~In contrast, \textsc{proteus} fits task-local signatures once, in each adapter's own embedding space (type (c)), retaining input-adaptive retrieval without retraining or replay.~Since these signatures use no class labels, its retrieval extends naturally to settings where classes recur across tasks or the task is not classification.

\vspace{-3mm}
\section{Proposed Method}\label{sec:method}
\begin{figure*}[!t]
    \centering
    \scalebox{1}[0.9]{\includegraphics[width=1\linewidth]{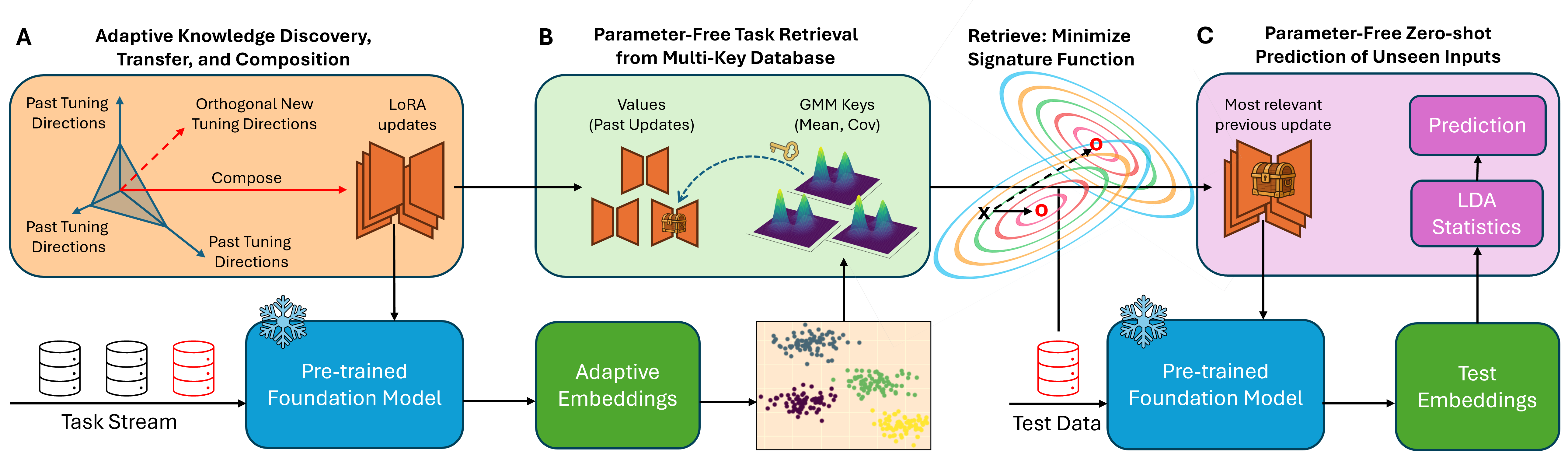}}
    \caption{Workflow of \textsc{Proteus}. {\bf A:} Each LoRA adapter combines reused prior directions with new orthogonal ones (Section~\ref{sec:adaptive_ft}). {\bf B:} Frozen GMM signatures, stored with the adapters, retrieve one adapter per test batch by likelihood-based plurality voting (Section~\ref{subsec:param-free}). {\bf C:} Class-prototype (CP)-based classification (Section~\ref{sec:adaptive_ft}).}
    \label{fig:Workflow SOLAR}
    \vspace{-2mm}
\end{figure*}

\textsc{proteus} has three components (Fig.~\ref{fig:Workflow SOLAR}). Section~\ref{subsec:param-free} pairs each adapter with a frozen statistical signature of its induced embeddings and retrieves adapters by likelihood matching with a plurality vote over test batches, without sharing retrieval parameters across tasks. Section~\ref{subsec:error-rate} characterizes when this retrieval succeeds, showing that the required separation between signatures grows at most logarithmically with the number of tasks. Section~\ref{sec:adaptive_ft} instantiates the framework with LoRA, composing each new adapter from sparsely reused prior directions and new orthogonal ones, and describes the CP-based classifier used for prediction.

\subsection{Parameter-Free Retrieval Algorithm}
\label{subsec:param-free}\vspace{-1mm}
\noindent {\bf Intuition.}~Each adapter $\Delta\boldsymbol{\omega}$ induces a distribution $D(\cdot)$ of embeddings for inputs of its task. If a test input $\boldsymbol{x}_\ast$ belongs to that task, its embedding $h(\boldsymbol{x}_\ast; \Delta\boldsymbol{\omega})$ is expected to have a high likelihood under $D$. Retrieval can therefore use a database $\mathrm{KB}$ that pairs each module of the pool $\boldsymbol{M}_k$ with its signature $D$, and select the module whose signature assigns the highest likelihood to the embedding it induces, without any learned retrieval function.

{\bf Design.}~For task $i$ and a Gaussian with mean $\boldsymbol{m}$ and covariance $\boldsymbol{\Lambda}$, let $\boldsymbol{h}_i(\boldsymbol{x}) = h(\boldsymbol{x}; \Delta\boldsymbol{\omega}_i) \in \mathbb{R}^d$ and define
\vspace{-1mm}
\begin{eqnarray}
\label{eq:score}
\hspace{-10mm}& & \ell_i\big(\boldsymbol{x}; \boldsymbol{m}, \boldsymbol{\Lambda}\big) \ \triangleq\ Z_i\big(\boldsymbol{x}; \boldsymbol{m}, \boldsymbol{\Lambda}\big) + \log\big|\boldsymbol{\Lambda}\big| \ , \quad
Z_i\big(\boldsymbol{x}; \boldsymbol{m}, \boldsymbol{\Lambda}\big) \ \triangleq\ \big(\boldsymbol{h}_i(\boldsymbol{x}) - \boldsymbol{m}\big)^\top \boldsymbol{\Lambda}^{-1} \big(\boldsymbol{h}_i(\boldsymbol{x}) - \boldsymbol{m}\big) \ ,
\label{eq:mahalanobis}
\end{eqnarray}
where $Z_i$ is the squared Mahalanobis distance and $\ell_i$ equals $-2\log \mathbb{N}(\boldsymbol{h}_i(\boldsymbol{x}); \boldsymbol{m}, \boldsymbol{\Lambda})$ up to the constant $d\log 2\pi$, so a lower score indicates a higher likelihood. A signature for $\Delta\boldsymbol{\omega}_k$ summarizes the embeddings of $\mathcal{D}_k$. Modeling them as a single Gaussian, its parameters minimize the negative log-likelihood:
\vspace{-1mm}
\begin{eqnarray}
\label{eq:retrieval_gaussian}
\left(\boldsymbol{m}_k, \boldsymbol{\Lambda}_k\right) &=& \underset{\boldsymbol{m} \in \mathbb{R}^d,\ \boldsymbol{\Lambda} \in \mathbb{S}^d,\ \boldsymbol{\Lambda} \succeq \lambda_0 \boldsymbol{I}}{\mathrm{argmin}} \ \sum_{\boldsymbol{x} \in \mathcal{D}_k} \ell_k\big(\boldsymbol{x}; \boldsymbol{m}, \boldsymbol{\Lambda}\big) \ ,
\end{eqnarray}
where $\mathbb{S}^d$ is the set of $d \times d$ symmetric matrices and $\lambda_0 > 0$ lower-bounds the eigenvalues of $\boldsymbol{\Lambda}$, which keeps the minimizer well defined even when $|\mathcal{D}_k| \leq d$. To capture multimodal embedding distributions, we extend this single-key signature to a multi-key signature.

\begin{definition}[Multi-Key Signature]~The multi-key signature of an adaptation module $\Delta\boldsymbol{\omega}_k$ for task $\mathcal{D}_k$ is a collection of $\tau$ Gaussian components $(\boldsymbol{m}_k^t, \boldsymbol{\Lambda}_k^t)_{t=1}^\tau$ such that the induced embedding $\boldsymbol{h}_k(\boldsymbol{x})$ of each input $\boldsymbol{x} \sim \mathcal{D}_k$ follows $\mathbb{N}(\boldsymbol{m}_k^t, \boldsymbol{\Lambda}_k^t)$ for some $t \in [\tau]$.
\label{def: multi-key}
\end{definition}

The multi-key signature replaces Eq.~\ref{eq:retrieval_gaussian} with a clustering objective over component assignments $\zeta_{it}$:
\vspace{-1mm}
\begin{eqnarray}
\hspace{-10mm}& & \underset{\boldsymbol{m}^{1:\tau}, \boldsymbol{\Lambda}^{1:\tau}, \{\zeta_{it}\}}{\minimize} \ \sum_{t=1}^\tau\sum_{i=1}^{|\mathcal{D}_k|} \zeta_{it}\, \ell_k\big(\boldsymbol{x}_i; \boldsymbol{m}^t, \boldsymbol{\Lambda}^t\big) \quad \text{s.t.} \quad \zeta_{it} \in \{0, 1\}, \ \ \sum_{t=1}^\tau \zeta_{it} = 1, \ \ \boldsymbol{\Lambda}^t \succeq \lambda_0 \boldsymbol{I} \ . \label{eq:18}
\end{eqnarray}
To avoid fixing $\tau$, we solve Eq.~\ref{eq:18} with a Dirichlet Process Gaussian Mixture Model (DP-GMM)~\cite{teh2010dirichlet}, which learns $\tau$ together with the component parameters. We denote the resulting signature of task $k$ by $\boldsymbol{\psi}_k \triangleq \{(\boldsymbol{m}_k^t, \boldsymbol{\Lambda}_k^t)\}_{t=1}^{\tau}$, which is fit once on the embeddings of $\mathcal{D}_k$ after $\Delta\boldsymbol{\omega}_k$ is trained and frozen afterward.

\noindent {\bf Batch-wise inference.}~At test time, inputs arrive in batches $\mathcal{X}_B = \{\boldsymbol{x}_1, \ldots, \boldsymbol{x}_B\}$ drawn independently from a single, unknown task $k \in [n]$, where $n \leq m$ is the number of tasks learned so far, and inputs may belong to different components, i.e., $\boldsymbol{x}_q \sim \mathcal{D}_k^{t_q}$. Each input first votes for the task of its best-matching component, scored in that task's own embedding space, and the batch is then assigned to the task with the most votes, with ties broken by the lowest task index:
\vspace{-1mm}
\begin{eqnarray}
\label{eq:per-input-retrieval}
\hat{k}_q \ = \ \underset{i \in [n]}{\mathrm{argmin}} \ \min_{t \in [\tau]} \ \ell_i\big(\boldsymbol{x}_q; \boldsymbol{m}_i^t, \boldsymbol{\Lambda}_i^t\big) \ , \qquad
\hat{k}_B \ = \ \underset{i \in [n]}{\mathrm{argmax}} \ \sum_{q=1}^{B} \mathbb{I}\big[\hat{k}_q = i\big] \ .
\label{eq:plurality-vote}
\end{eqnarray}
The retrieved adapter $\Delta\boldsymbol{\omega}_{\hat{k}_B}$ is applied to every input in the batch, and $B = 1$ recovers per-input retrieval. Eq.~\ref{eq:plurality-vote} thus defines the selection rule of \textsc{Proteus}, $g_{\boldsymbol{\psi}}(\mathcal{X}_B, \boldsymbol{M}_n) = \{\Delta\boldsymbol{\omega}_{\hat{k}_B}\}$, with $\boldsymbol{\psi} = (\boldsymbol{\psi}_1, \ldots, \boldsymbol{\psi}_n)$.

\subsection{Retrieval Error Analysis}\label{subsec:error-rate}
Retrieval fails when an input is better explained by the signature of another task than by its own, which happens when signatures from different tasks are insufficiently separated. We formalize this through a clustered data assumption and a cluster separation factor, and bound the retrieval error in terms of the latter.

\begin{assumption}[Clustered Data Distribution]
\label{assumption:3.1}
For each task $k$ with a multi-key signature $(\boldsymbol{m}_k^t, \boldsymbol{\Lambda}_k^t)_{t=1}^\tau$, the data distribution $\mathcal{D}_k$ decomposes into $\tau$ sub-distributions $\{\mathcal{D}_k^t\}_{t=1}^\tau$ such that $\boldsymbol{h}_k(\boldsymbol{x}) \sim \mathbb{N}(\boldsymbol{m}^t_k, \boldsymbol{\Lambda}^t_k)$ when $\boldsymbol{x} \sim \mathcal{D}_k^t$, with $\boldsymbol{\Lambda}^t_k$ full rank. This matches the mixture fitted by Eq.~\ref{eq:18}.
\end{assumption}

An input $\boldsymbol{x} \sim \mathcal{D}_k^t$ is retrieved incorrectly when a false component $(\boldsymbol{m}_i^s, \boldsymbol{\Lambda}_i^s)$ of another task $i \neq k$ assigns it a lower score $\ell_i(\boldsymbol{x}; \boldsymbol{m}_i^s, \boldsymbol{\Lambda}_i^s)$ in Eq.~\ref{eq:score} than its true component. Writing $Z_i^s(\boldsymbol{x}) \triangleq Z_i(\boldsymbol{x}; \boldsymbol{m}_i^s, \boldsymbol{\Lambda}_i^s)$, we quantify how far inputs lie from false components as follows.

\begin{definition}[Cluster Separation Factor]\label{def:delta}
The average distance between an input sampled from $\mathcal{D}_k^t$ and the mean of a different representation cluster $(\boldsymbol{m}_i^s, \boldsymbol{\Lambda}_i^s)$, measured under the representation function $\boldsymbol{h}_i(\boldsymbol{x})$ of that cluster, exceeds $d$ by at least a multiplicative {\em separation factor} $1 + \delta$:
\begin{eqnarray}
\mathbb{E}_{\boldsymbol{x} \sim \mathcal{D}_k^t}\left[Z_i^s(\boldsymbol{x})\right] &\geq& \big(1 + \delta\big)d \ , \label{eq:spread}
\end{eqnarray}
where $\delta \geq 0$ the largest factor for which Eq.~\ref{eq:spread} holds for every false component.
\end{definition}

\begin{theorem}[Bounding Retrieval Error Rate Per Input]\label{main_thm:1}
Under Assumption~\ref{assumption:3.1} and Assumption~\ref{assumption:2} in App.~\ref{retrieval-err}, which requires the separation in Definition~\ref{def:delta}, concentration of the embedding fluctuations, and covariance eigenvalues bounded below by $\lambda_0$, suppose we are given $n$ tasks with clustered data distributions $(\mathcal{D}_1^t)_{t=1}^\tau, \ldots, (\mathcal{D}_n^t)_{t=1}^\tau$. Let $\boldsymbol{E}_k^t$ denote the event that a test input $\boldsymbol{x} \sim \mathcal{D}_k^t$ is not matched with the correct embedding function $\boldsymbol{h}_k$ by Eq.~\ref{eq:per-input-retrieval}. Then
\begin{eqnarray}
\hspace{-8.75mm}\mathrm{Pr}\Big(\boldsymbol{E}_k^t\Big) &\leq& \exp\left(-\frac{d\delta^2}{8\delta + 16}\right) \ \ +\ \ 4(n-1)\tau\ \exp\left(-c\frac{(d\delta/2 + \kappa)^2}{\sigma^2d + b(d\delta/2 + \kappa)}\right) \ \ \text{with}\ \ \nonumber \\
\hspace{-8.75mm}\kappa &\triangleq& \min_{(i,s): i \ne k} \log\frac{|\boldsymbol{\Lambda}_i^s|}{|\boldsymbol{\Lambda}_k^t|}, \quad \sigma^2\triangleq\frac{K^4}{\lambda_0^2}+\frac{K^2(1+\delta)}{\lambda_0}, \quad b\triangleq \frac{K^2}{\lambda_0} \ ,\label{eq:bound}
\end{eqnarray}
where $c>0$ is a universal constant (independent of $d,\delta, n, \tau, K, \lambda_0$), $K$ is the concentration constant in Assumption~\ref{assumption:2}, and $\delta \geq \max(0, -2\kappa/d)$. Here, $\kappa$ is the minimum difference between the log-volumes of false clusters and that of the true cluster.
\end{theorem}

\noindent The proof is in App.~\ref{retrieval-err}. The bound separates two sources of error. The first term controls the fluctuation of an input around its own cluster, and the second the chance that one of the $(n-1)\tau$ false components explains it better, combined by a union bound. Treating $K$, $\lambda_0$, and $\kappa$ as constants after multi-key signatures are fit, the bound behaves as $n\tau \exp\big(-\Omega\big(d\min(\delta, \delta^2)\big)\big)$, which reveals three properties. First, \emph{dimension helps}: since $\delta$ measures separation relative to $d$, the same $\delta$ yields an exponentially smaller error as the embedding dimension grows, which makes the high-dimensional embeddings of pre-trained backbones well suited to signature-based retrieval. Second, \emph{finer signatures are inexpensive}: more components per task make Assumption~\ref{assumption:3.1} easier to satisfy, while adding competitors only linearly to the bound, which the exponential decay absorbs once $\delta$ is moderate. Third, \emph{volume matters}: because the score in Eq.~\ref{eq:score} includes $\log|\boldsymbol{\Lambda}|$, a false component with a smaller covariance volume ($\kappa < 0$) is more attractive, and the condition $\delta \geq -2\kappa/d$ requires separation to outweigh this advantage. The eigenvalue floor $\lambda_0$ in Eq.~\ref{eq:retrieval_gaussian} bounds this effect by ruling out degenerate, arbitrarily concentrated signatures.

To extend the bound to batches, let $p_t$ denote the right-hand side of Eq.~\ref{eq:bound} for component $t$ and $p \triangleq \max_{t \in [\tau]} p_t$. Since adapters and signatures are fixed and inputs are independent, the votes in Eq.~\ref{eq:per-input-retrieval} are independently, each with probability at most $p$, and the plurality vote returns the true task whenever it receives a strict majority of the $B$ votes. This results in the following corollary.

\begin{corollary}[Plurality-Vote Batch Retrieval Error]\label{corollary}
Suppose the assumptions of Theorem~\ref{main_thm:1} hold for every component of task $k$, and let $E_B^{\text{vote}}$ denote the event that a test batch $\{\boldsymbol{x}_1, \ldots, \boldsymbol{x}_B\}$ with $\boldsymbol{x}_q \sim \mathcal{D}_k^{t_q}$ is assigned a wrong task by Eq.~\ref{eq:plurality-vote}. If $p < 1/2$, then
\begin{eqnarray}
\mathrm{Pr}\big(E_B^{\text{vote}}\big) &\leq& \Big(4p(1-p) \Big)^{B/2} \ .
\end{eqnarray}
Hence, $\mathrm{Pr}(E_B^{\text{vote}}) \leq \epsilon_B$ for any batch tolerance $\epsilon_B > 0$ once $B \geq 2\log(1/\epsilon_B)/\log\big(1/(4p(1-p))\big)$.
\end{corollary}

Voting converts a per-input guarantee into a batch guarantee that improves exponentially with $B$, and the required batch size grows only logarithmically in $1/\epsilon_B$: a per-input error of $p = 0.2$ requires $B \geq 21$ for $\epsilon_B = 0.01$ and only $B \geq 31$ for $\epsilon_B = 0.001$. Its benefit, however, hinges on $p$: at $p = 0.4$, the tolerance $\epsilon_B = 0.01$ already requires $B \geq 226$, and for $p \geq 1/2$ the bound is vacuous, so no batch size can compensate for a retriever that is wrong on most inputs. Batch size and separation are thus complementary levers: $B$ is chosen at inference, whereas $\delta$, which controls $p$ through Theorem~\ref{main_thm:1}, is shaped during training.~As such, this raises the question on the sufficient condition for the separation parameter $\delta$: 

{\bf How large must $\delta$ be to drive the per-input error below a target tolerance? }

To answer this, we develop a new result in Theorem~\ref{main_thm:2} based on Theorem~\ref{main_thm:1} above.~This result essentially establishes the sufficient condition for $\delta$ so that the rate of per-input error is below a target tolerance.

\begin{theorem}[Sufficient Separation for Controlled Error -- Informal Statement]\label{main_thm:2}
For any tolerance $\epsilon > 0$, $\mathrm{Pr}(\boldsymbol{E}_k^t) \leq \epsilon$ whenever 
\begin{eqnarray}
    \delta &\geq& C_1 L/d \ \ +\ \ C_2\sqrt{L/d} \ \ +\ \  3\kappa^-/d \ ,
\end{eqnarray}
where $L \triangleq \log(8n\tau/\epsilon)$, $\kappa^- \triangleq \max(0, -\kappa)$, and $C_1, C_2$ depend only on $K^2/\lambda_0$ and $c$. The full statement is given in App.~\ref{retrieval-err}.
\end{theorem}

The threshold shows how retrieval scales. Since the number of tasks $n$ enters only through $L$, the separation sufficient for a target error grows at most logarithmically with the number of tasks: with $d = 768$, $\tau = 20$, and $\epsilon = 0.01$, increasing the number of tasks tenfold, from 10 to 100, raises the dominant $\sqrt{L/d}$ term by only about $9\%$. The threshold depends on the number of components and on the tolerance in the same logarithmic way, and it shrinks as $1/\sqrt{d}$ with the embedding dimension. These conclusions hold as long as competing signatures are not more concentrated than the true one ($\kappa \geq 0$); otherwise, the $\kappa^-$ term adds to the threshold and can grow as tasks accumulate. Because $\delta$ is defined through expected Mahalanobis distances, it can also be estimated from embeddings and compared with the threshold, which turns the guarantee into a diagnostic for a given adapter design. Finally, applying Theorem~\ref{main_thm:2} with $\epsilon < 1/2$ to every component of task $k$ gives $p \leq \epsilon$, so Corollary~\ref{corollary} bounds the batch error by $\big(4\epsilon(1-\epsilon)\big)^{B/2}$. Signature separation is therefore the quantity an adapter design should promote. 

While directly optimizing this separation criterion is intractable, our theory points to a clear and intuitive strategy: \textbf{enforcing representation-level orthogonality across adaptation modules naturally promotes cluster separation and suppresses retrieval ambiguity}. Guided by this insight, we subsequently develop a practical adaptive fine-tuning algorithm in Section~\ref{sec:adaptive_ft} that explicitly encourages such structure, and demonstrate through extensive experiments across multiple benchmarks in Section~\ref{sec:exprs} that this strategy yields both highly accurate retrieval and strong predictive performance.

\subsection{Retrieval-Aware Continual Fine-Tuning via Adaptive LoRA Composition}
\label{sec:adaptive_ft}

We instantiate \textsc{proteus} with LoRA~\cite{hu2021lora}, where each task's adapter is a low-rank matrix added to the frozen backbone weights. The challenge is to design these adapters so that their induced signature clusters are well-separated in representation space, satisfying the condition Theorems~\ref{main_thm:2} require.~Two complementary mechanisms achieve this.~First, \textit{knowledge transfer}: new adapters should reuse relevant directions from prior tasks, preventing redundant learning and keeping representations coherent.~Second, \textit{orthogonal discovery}: new adapters should introduce new directions, orthogonal to prior ones, ensuring that task-specific features occupy distinct regions of representation space and thus that signature clusters remain separable.

\textbf{Adaptive LoRA Composition.}~Each LoRA module of task $j$ consists of $r$ rank-1 directions $\boldsymbol{b}_{j,i}\boldsymbol{a}_{j,i}^\top$, with $\boldsymbol{b}_{j,i} \in \mathbb{R}^p$ and $\boldsymbol{a}_{j,i} \in \mathbb{R}^q$. For a new task $\mathcal{D}_{k+1}$, prior directions are frozen, and the new adapter combines them with $r$ new directions:
\vspace{-2mm}
\begin{eqnarray}
\Delta\boldsymbol{\omega}_{k+1} &\triangleq& \sum_{j=1}^k \boldsymbol{B}_j\boldsymbol{S}_j\boldsymbol{A}^\top_j \ + \underbrace{\boldsymbol{B}_{k+1}\boldsymbol{A}^\top_{k+1}}_{\Delta \boldsymbol{\omega}_{k+1}^\perp} \ ,
\label{eq:9}
\end{eqnarray}
where $\boldsymbol{A}_j \triangleq [\boldsymbol{a}_{j,1}, \ldots, \boldsymbol{a}_{j,r}]$, $\boldsymbol{B}_j \triangleq [\boldsymbol{b}_{j,1}, \ldots, \boldsymbol{b}_{j,r}]$, and $\boldsymbol{S}_j \triangleq \mathrm{diag}[s_{j,1}, \ldots, s_{j,r}]$ holds learnable coefficients. The \textit{knowledge transfer} term $\boldsymbol{B}_j\boldsymbol{S}_j\boldsymbol{A}_j^\top$ weights the directions of task $j$ by their relevance to task $k+1$, and the \textit{orthogonal discovery} term $\Delta \boldsymbol{\omega}_{k+1}^\perp$ is encouraged to be orthogonal to prior directions by $R_\perp$ in Eq.~\ref{eq:ortho}.

\textbf{Training Loss.}~Without constraints, $\boldsymbol{S}$ could spread weight over many weakly relevant directions. We therefore impose an elastic-net penalty that promotes sparse, selective transfer, and train task $k+1$ with
\vspace{-1mm}
\begin{eqnarray}
\hspace{-10mm}&\underset{\Delta\boldsymbol{\omega}_{k+1}^\perp, \boldsymbol{S}}{\minimize}& L\left(\mathcal{D}_{k+1}, \Delta\boldsymbol{\omega}_{k+1}\right) + R\left(\boldsymbol{A}_{k+1}\right) \ , \label{eq:15_only} \\
\hspace{-10mm}&\text{where}& L \triangleq L_{\mathrm{cls}}\left(\mathcal{D}_{k+1}, \Delta\boldsymbol{\omega}_{k+1}\right) + \lambda \Big( \alpha \left\|\boldsymbol{S}\right\|_1 + (1 - \alpha) \left\|\boldsymbol{S}\right\|_2 \Big) \ , \quad R \triangleq \beta\, R_\perp\left(\boldsymbol{A}_{k+1}\right) \ , \nonumber
\end{eqnarray}
with $L_{\mathrm{cls}}$ the classification loss, $\boldsymbol{S} = \mathrm{blkdiag}[\boldsymbol{S}_1,\ldots, \boldsymbol{S}_k]$, and $\beta > 0$.

\begin{wrapfigure}{r}{0.55\textwidth}
\vspace{-15pt}
\begin{minipage}{0.55\textwidth}
\begin{algorithm}[H]
   \caption{\textsc{Proteus} Training}
   \label{alg:training_algs}
\begin{algorithmic}[1]
   \State {\bfseries INPUT:} pre-trained backbone $h(\boldsymbol{x}; \boldsymbol{\omega}_0)$ and stream of $m$ datasets $\mathcal{D}_1, \ldots, \mathcal{D}_m$
   \State {\bfseries OUTPUT:} database $\mathrm{KB} = \{(\boldsymbol{\psi}_k, \Delta\boldsymbol{\omega}_k)\}_{k=1}^m$; statistics $\boldsymbol{G}$, $\{\boldsymbol{e}_c\}$
   \State initialize $\mathrm{KB} \leftarrow \emptyset$, $\boldsymbol{G} \leftarrow \boldsymbol{0}$, $\boldsymbol{e}_c \leftarrow \boldsymbol{0}\ \forall c$; sample $\boldsymbol{W} \in \mathbb{R}^{d \times d_{\mathrm{rp}}}$
   \For{$k=1$ {\bfseries to} $m$}
    \State train $\Delta\boldsymbol{\omega}_k$ by solving Eq.~\ref{eq:15_only}
    \State fit signature $\boldsymbol{\psi}_k = \{(\boldsymbol{m}_k^t, \boldsymbol{\Lambda}_k^t)\}_{t=1}^\tau$ via Eq.~\ref{eq:18}
    \State $\mathrm{KB} \leftarrow \mathrm{KB} \cup \{(\boldsymbol{\psi}_k, \Delta\boldsymbol{\omega}_k)\}$
    \For {$(\boldsymbol{x}, y) \in \mathcal{D}_k$}
        \State $\boldsymbol{z} \leftarrow \sigma\big(\boldsymbol{W}^\top \boldsymbol{h}_k(\boldsymbol{x})\big)$ with $\boldsymbol{h}_k(\boldsymbol{x}) = h(\boldsymbol{x}; \Delta\boldsymbol{\omega}_k)$
        \State $\boldsymbol{G} \leftarrow \boldsymbol{G} + \boldsymbol{z}\boldsymbol{z}^\top$, \quad $\boldsymbol{e}_y \leftarrow \boldsymbol{e}_y + \boldsymbol{z}$
    \EndFor
   \EndFor
   \State {\bfseries return} $\mathrm{KB}$, $\boldsymbol{G}$, $\{\boldsymbol{e}_c\}$
\end{algorithmic}
\end{algorithm}
\end{minipage}
\vspace{-15pt}
\end{wrapfigure}

\textbf{Enforcing Orthogonality.}~The regularizer $R$ penalizes non-orthogonality on the factors $\boldsymbol{A}$ only, adopted from O-LoRA work~\cite{wang2023orthogonal}. Specifically, for each attention block $l$ equipped with LoRA, let $\bar{\boldsymbol{A}}_{k+1}^{(l)}$ be the new factor with unit-norm columns, and $\bar{\boldsymbol{A}}_{1:k}^{(l)} \triangleq \big[\bar{\boldsymbol{A}}_1^{(l)}, \ldots, \bar{\boldsymbol{A}}_k^{(l)}\big]$ the frozen, column-normalized factors of previous tasks:
\vspace{-1mm}
\begin{eqnarray}
\label{eq:ortho}
R_\perp\big(\boldsymbol{A}_{k+1}\big) &\triangleq& \sum_{l} \Big\| \big(\bar{\boldsymbol{A}}_{k+1}^{(l)}\big)^{\top} \bar{\boldsymbol{A}}_{1:k}^{(l)} \Big\|_F^2 
\end{eqnarray}
i.e., the sum of squared cosine similarities between new and prior columns, with gradients flowing only through $\boldsymbol{A}_{k+1}$. Constraining $\boldsymbol{A}$ suffices: if $\boldsymbol{A}_j^\top \boldsymbol{A}_{k+1} = \boldsymbol{0}$, then $\langle \Delta\boldsymbol{\omega}_{k+1}^\perp, \Delta\boldsymbol{\omega}_j^\perp \rangle_F = \mathrm{tr}(\boldsymbol{B}_{k+1}^\top \boldsymbol{B}_j \boldsymbol{A}_j^\top \boldsymbol{A}_{k+1}) = 0$, so the new component responds only to input directions that prior ones ignore. Column normalization makes the penalty scale-invariant, since $\boldsymbol{B}\boldsymbol{A}^\top = (c\boldsymbol{B})(\boldsymbol{A}/c)^\top$ would otherwise let it shrink $\boldsymbol{A}$ without changing update.

\textbf{Class-Prototype (CP)-based Classification.}~For prediction, we adopt the classifier of RanPAC~\cite{DBLP:conf/nips/McDonnellGPAH23}. Each embedding is mapped through a fixed random projection $\boldsymbol{W} \in \mathbb{R}^{d \times d_{\mathrm{rp}}}$ followed by an elementwise nonlinearity $\sigma$, and the Gram matrix $\boldsymbol{G}$ and class prototypes $\boldsymbol{e}_c$ of the projected features are accumulated during training without revisiting past data (Algorithm~\ref{alg:training_algs}). A test input $\boldsymbol{x}_\ast$ is then classified using its embedding under the retrieved adapter:
\vspace{-1mm}
\begin{eqnarray}
c_\ast &=& \underset{c}{\mathrm{argmax}} \ \boldsymbol{z}_\ast^\top \big(\boldsymbol{G} + \gamma\boldsymbol{I}\big)^{-1}\boldsymbol{e}_c \ , \quad \boldsymbol{z}_\ast \triangleq \sigma\big(\boldsymbol{W}^\top \boldsymbol{h}_{\hat{k}_B}(\boldsymbol{x}_\ast)\big) \ , \label{eq:23}
\end{eqnarray}
where $\gamma > 0$ is the ridge parameter. The inference pseudocode is given in App.~\ref{sec:inference-pseudocode}.

\vspace{-2mm}
\section{Empirical Studies} 
\label{sec:exprs}

\subsection{Experimental Setup}
\label{sec:exp-setup}

{\bf Datasets.}~We compare \textsc{Proteus} with state-of-the-art CL baselines, all built on the same backbone: ViT-B/16 pre-trained on ImageNet-21k~\cite{dosovitskiy2020image}\footnote{Original checkpoint: \url{https://storage.googleapis.com/vit_models/imagenet21k/ViT-B_16.npz}.}. Following~\citet{DBLP:conf/icml/KimYPLSBL24}, we evaluate under three scenarios that differ in the degree of semantic shift between tasks.\vspace{-1mm}

$\blacktriangleright$ {\bf Uniformly Mild, where tasks remain semantically similar.}~We split CIFAR-100~\cite{krizhevsky2009learning} (100 classes, 60,000 images) into 10 tasks of 10 classes, and ImageNet-R~\cite{hendrycks2021many} (200 classes, 30,000 images) into 10 tasks of 20 classes. Since tasks are disjoint class splits of a single dataset, they share a high degree of similarity~\cite{DBLP:conf/icml/KimYPLSBL24}.\vspace{-1mm}

$\blacktriangleright$ {\bf Uniformly Abrupt, where tasks are semantically distant.}~ImageNet-A~\cite{hendrycks2021natural} (200 classes, 7,500 images) consists of natural adversarial examples that standard ImageNet classifiers misclassify; we split it into 10 tasks of 20 classes. We also use two 5-task versions of VTAB~\cite{zhai2019visual}, in which each task is a different dataset. VTAB-1k-5T follows the VTAB-1k protocol~\cite{dong2023efficient}: 1,000 training and validation images per task, evaluated on the full test set. VTAB5T~\cite{zhou2023revisiting} is a downsampled version that keeps 10 classes from each of Resisc45, DTD, Pets, EuroSAT, and Flowers.\vspace{-1mm}

$\blacktriangleright$ {\bf Varying, where task shifts vary in semantic similarity.}~Adapting~\citet{DBLP:conf/icml/KimYPLSBL24}, we build VTAB-1k-Sim50 by selecting two datasets from each of VTAB's three groups (natural, specialized, structured) and re-partitioning their data into 6 tasks, each drawing 50\% of its classes from one source dataset and 50\% from another. This yields shifts of varying degree across the sequence while keeping class sets disjoint, preserving the class-incremental setting.

{\bf Baselines and Metrics.}~We compare \textsc{Proteus} against retrieval-based CL methods, namely L2P~\cite{wang2022learning}, AdaPromptCL~\cite{DBLP:conf/icml/KimYPLSBL24}, DualPrompt~\cite{wang2022dualprompt}, CODA-Prompt~\cite{Smith_2023_CVPR}, S-Prompt~\cite{wang2022sprompt}, and HiDe-Prompt~\cite{wang2023hide}, and retrieval-free CL methods, namely InfLoRA~\cite{liang2024inflora}, SD-LoRA~\cite{wu2025sd}, and RanPAC~\cite{DBLP:conf/nips/McDonnellGPAH23}. We additionally report joint training, which fine-tunes a single LoRA module on all tasks' data at once, using the same backbone, a total rank matching that of \textsc{Proteus} ($m \cdot r$), and the same number of epochs.  We measure the \textit{average accuracy} (higher is better) $\mathrm{ACC}_\kappa = (1/\kappa) \sum_{j=1}^{\kappa} \mathrm{acc}(j,\kappa)$, where $\mathrm{acc}(j,\kappa)$ is the model's accuracy on task $j$ after learning task $\kappa$; tables report the final value $\mathrm{ACC}_m$. Results are averaged over $3$ random seeds. For retrieval-based methods, we also report \textit{retrieval accuracy}. Specifically, at test time, inputs arrive in batches drawn from a single task, and every retrieval-based method, including \textsc{Proteus}, assigns each batch the task selected by a plurality vote over its inputs; retrieval accuracy is the percentage of batches assigned to their true task. See App.~\ref{sec: data and hyper} for additional metrics, including average forgetting, and experiment details.

\subsection{Main Results} 
\label{sec:main-results}\vspace{-2mm}

\begin{table}[t]
\begin{center}
\hspace{-1mm}
\begin{small}
\caption{
\label{tab: main-table}~Final average accuracy $\mathrm{ACC}_m$ (\%) achieved by \textsc{Proteus} and other baselines across three scenarios of task shift. Results are reported as mean$_{\pm\text{std}}$ over 3 random seeds. Joint training is shown as a reference and is excluded from the ranking. The best and second-best values among CL methods are highlighted in \textbf{bold} and \underline{underline}, respectively.}
\vspace{1mm}
\setlength{\tabcolsep}{3pt}
\begin{tabular}{|lcc ccc c|}
    \toprule
    \textbf{Methods} & \multicolumn{2}{c}{\textbf{Uniformly Mild}} & \multicolumn{3}{c}{\textbf{Uniformly Abrupt}} & \textbf{Varying} \\
    \cmidrule(l){2-3} \cmidrule(l){4-6} \cmidrule(l){7-7}
    & CIFAR-100 & ImageNet-R & ImageNet-A & VTAB5T & VTAB-1k-5T & VTAB-1k-Sim50 \\
    \midrule
    \textit{Joint Training} & \textit{$89.46_{\pm 0.16}$} & \textit{$75.27_{\pm 0.17}$} & \textit{$49.96_{\pm 0.77}$} & \textit{$92.17_{\pm 0.80}$} & \textit{$52.18_{\pm 0.32}$} & \textit{$58.11_{\pm 0.82}$} \\
    \midrule
    L2P & $83.10_{\pm 0.08}$ & $63.56_{\pm 0.58}$ & $32.78_{\pm 1.58}$ & $82.64_{\pm 1.03}$ & $18.71_{\pm 0.24}$ & $48.34_{\pm2.47}$ \\
    DualPrompt & $85.56_{\pm 0.24}$ & $69.40_{\pm 0.31}$ & $43.74_{\pm 0.30}$ & $85.69_{\pm 0.31}$ & $25.43_{\pm 0.53}$ & $52.48_{\pm1.46}$ \\
    CODA-Prompt & $80.55_{\pm 0.49}$ & $61.37_{\pm 0.10}$ & $39.32_{\pm 0.35}$ & $83.20_{\pm 0.95}$ & $35.55_{\pm 0.57}$ & $45.83_{\pm1.17}$ \\
    S-Prompt & $84.84_{\pm 0.26}$ & $69.68_{\pm 0.24}$ & $40.27_{\pm 2.00}$ & $86.23_{\pm 0.40}$ & $28.15_{\pm 0.24}$ & $56.27_{\pm2.12}$ \\
    AdaPromptCL & $84.78_{\pm 0.15}$ & $67.74_{\pm 0.21}$ & $38.52_{\pm 0.18}$ & $82.07_{\pm 0.54}$ & $27.15_{\pm 0.46}$ & $44.18_{\pm5.26}$ \\
    HiDe-Prompt & \underline{$93.63_{\pm 0.72}$} & \underline{$79.46_{\pm 1.39}$} & \underline{$51.23_{\pm 4.23}$} & $89.05_{\pm 0.16}$ & \underline{$49.70_{\pm 0.19}$} & \underline{$64.72_{\pm1.33}$} \\
    \midrule
    InfLoRA & $82.39_{\pm 0.40}$ & $70.45_{\pm 0.12}$ & $39.27_{\pm 0.47}$ & $88.97_{\pm 1.25}$ & $46.33_{\pm 2.13}$ & $49.92_{\pm1.74}$ \\
    SD-LoRA & $83.11_{\pm 0.58}$ & $59.26_{\pm 1.58}$ & $31.20_{\pm0.70}$ & $81.86_{\pm1.13}$ & $17.56_{\pm 14.85}$ & $43.99_{\pm2.87}$ \\
    RanPAC & $87.21_{\pm 0.05}$ & $66.05_{\pm 0.58}$ & $43.03_{\pm 0.47}$ & \underline{$89.71_{\pm 0.55}$} & $43.76_{\pm 3.54}$ & $54.63_{\pm3.66}$ \\
    \midrule
    \textsc{Proteus} & $\mathbf{94.00_{\pm 0.53}}$ & $\mathbf{82.06_{\pm 0.89}}$ & $\mathbf{65.42_{\pm 0.78}}$ & $\mathbf{93.10_{\pm 0.33}}$ & $\mathbf{54.02_{\pm 0.52}}$ & $\mathbf{67.32_{\pm 1.05}}$ \\
    \bottomrule
\end{tabular}
\end{small}
\end{center}
\vspace{-3mm}
\end{table}

Table~\ref{tab: main-table} reports the \textbf{overall performance} of \textsc{Proteus} and the CL baselines under the three scenarios described above. \textsc{Proteus} achieves the best accuracy on all six benchmarks. In the \textit{uniformly mild} setting, \textsc{Proteus} and HiDe-Prompt are the two leading methods: \textsc{Proteus} exceeds HiDe-Prompt by 0.37\% on CIFAR-100 and 2.60\% on ImageNet-R, and both lead all other CL methods by at least 6.42\% and 9.01\%, respectively. In the \textit{uniformly abrupt} setting, \textsc{Proteus} exceeds the strongest retrieval-based method, HiDe-Prompt, by 14.19\%, 4.05\%, and 4.32\%, and the strongest retrieval-free method by 22.39\% (RanPAC), 3.39\% (RanPAC), and 7.69\% (InfLoRA) on ImageNet-A, VTAB5T, and VTAB-1k-5T, respectively. In the \textit{varying} scenario, \textsc{Proteus} improves on HiDe-Prompt by 2.60\% and on RanPAC by 12.69\%. Overall, the margin of \textsc{Proteus} over HiDe-Prompt widens as task shifts grow, from 0.37--2.60\% in the mild and varying setting to 4.05--14.19\% under abrupt shifts. Unlike HiDe-Prompt, which retrains its task-identity classifier after every task on replayed class statistics, \textsc{Proteus} retrieves modules through signatures that are fit once and frozen. Notably, joint training falls below \textsc{Proteus} on all six benchmarks and below HiDe-Prompt on four. We attribute this to interference among tasks sharing a single LoRA module and to the absence of a mechanism that selects task-relevant parameters for each test input. Its accuracy below 60\% on ImageNet-A, VTAB-1k-5T, and VTAB-1k-Sim50 further indicates that these benchmarks remain challenging even when all data are available at once.

\begin{figure}[h]
    \vspace{5pt}
    \centering
    \begin{subfigure}[b]{0.32\linewidth}
        \centering
        \includegraphics[width=\linewidth]{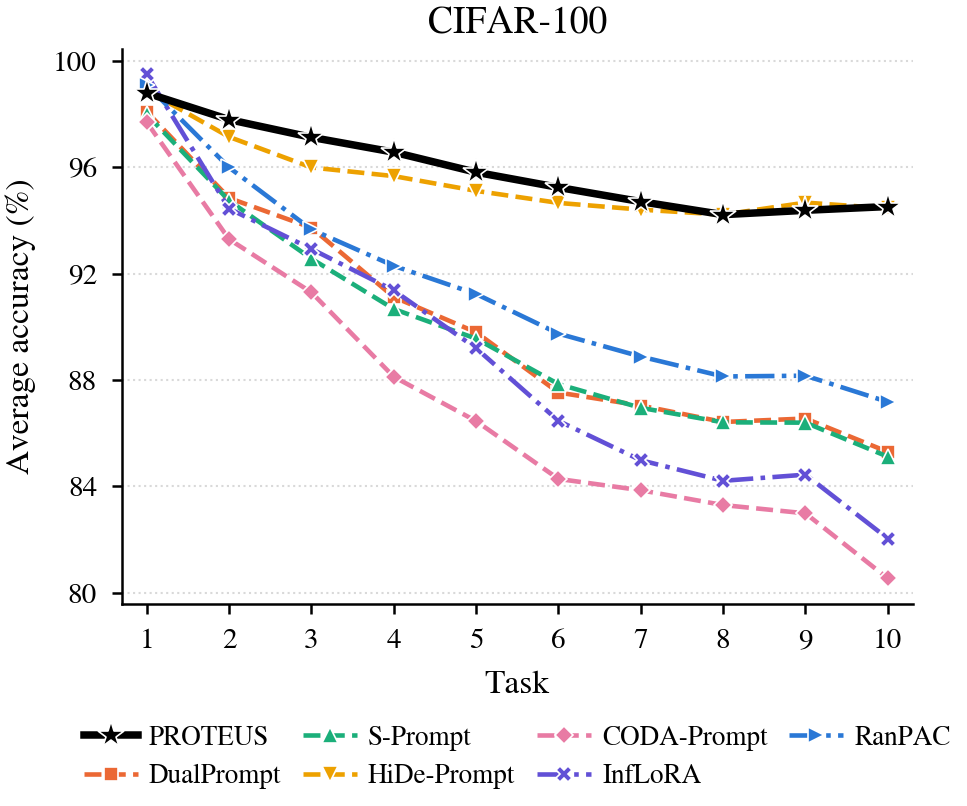}
        \caption{CIFAR-100}
        \label{fig:img1}
    \end{subfigure}\hfill
    \begin{subfigure}[b]{0.32\linewidth}
        \centering
        \includegraphics[width=\linewidth]{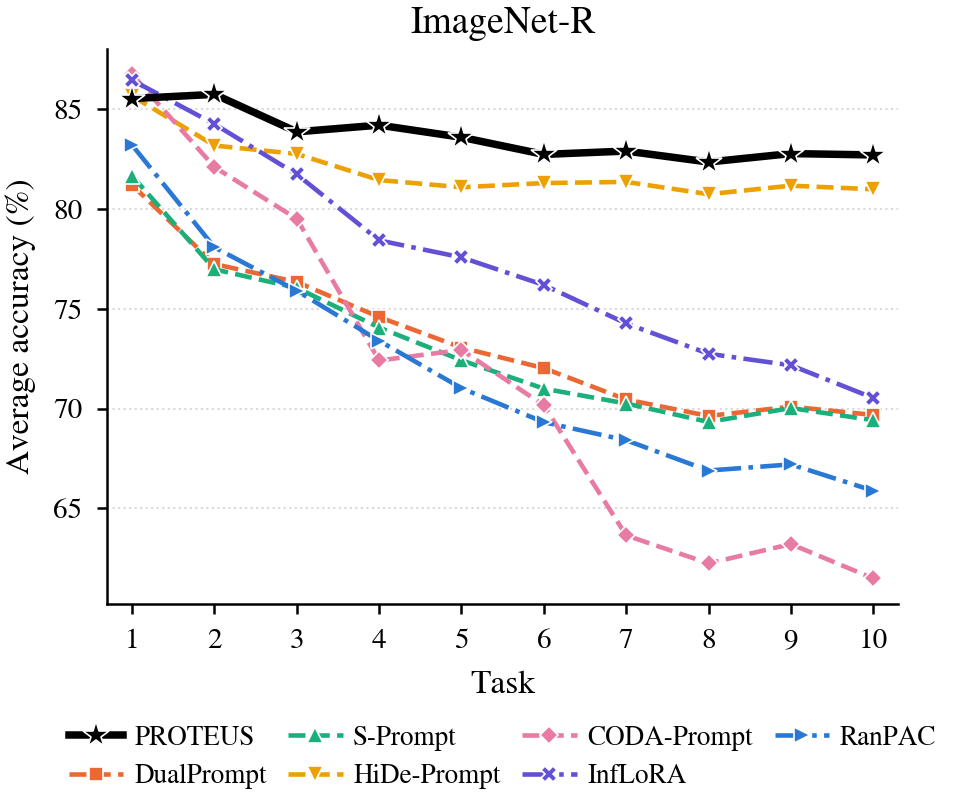}
        \caption{ImageNet-R}
        \label{fig:img2}
    \end{subfigure}\hfill
    \begin{subfigure}[b]{0.32\linewidth}
        \centering
        \includegraphics[width=\linewidth]{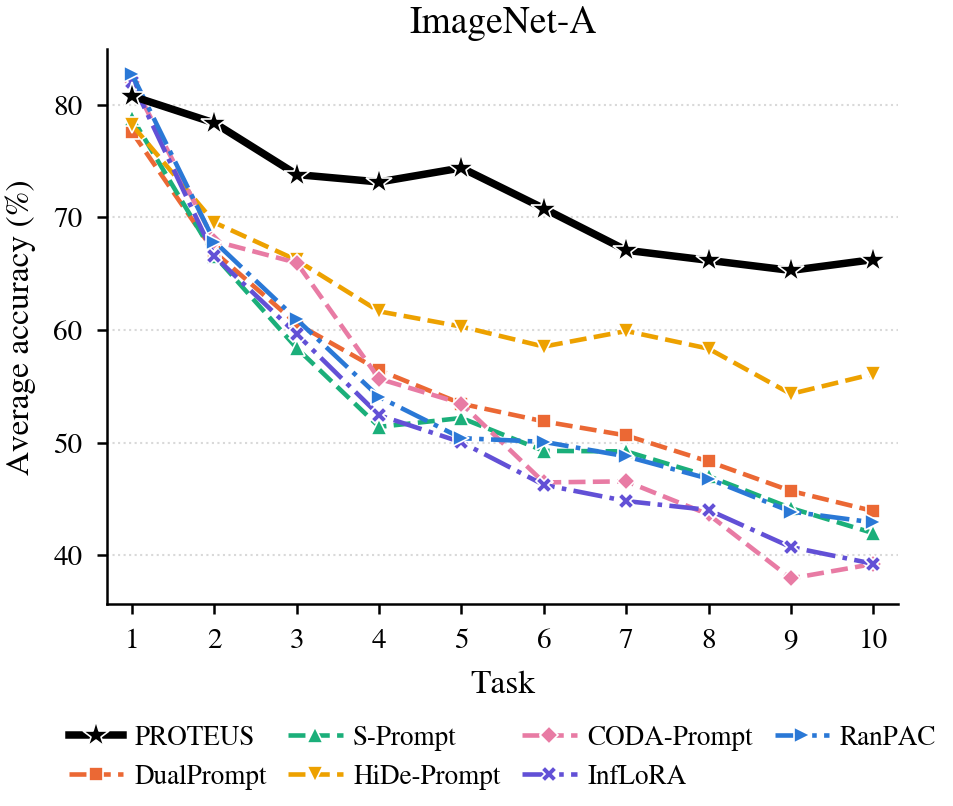}
        \caption{ImageNet-A}
        \label{fig:img3}
    \end{subfigure}

    \vspace{2mm}

    \begin{subfigure}[b]{0.32\linewidth}
        \centering
        \includegraphics[width=\linewidth]{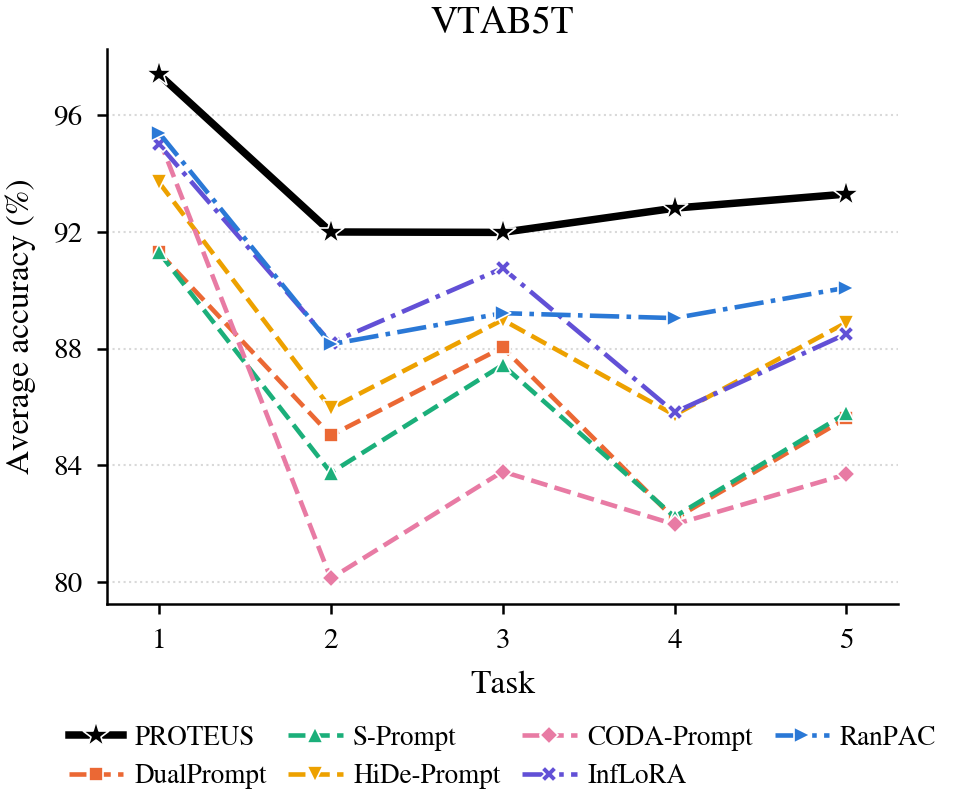}
        \caption{VTAB5T}
        \label{fig:img4}
    \end{subfigure}\hfill
    \begin{subfigure}[b]{0.32\linewidth}
        \centering
        \includegraphics[width=\linewidth]{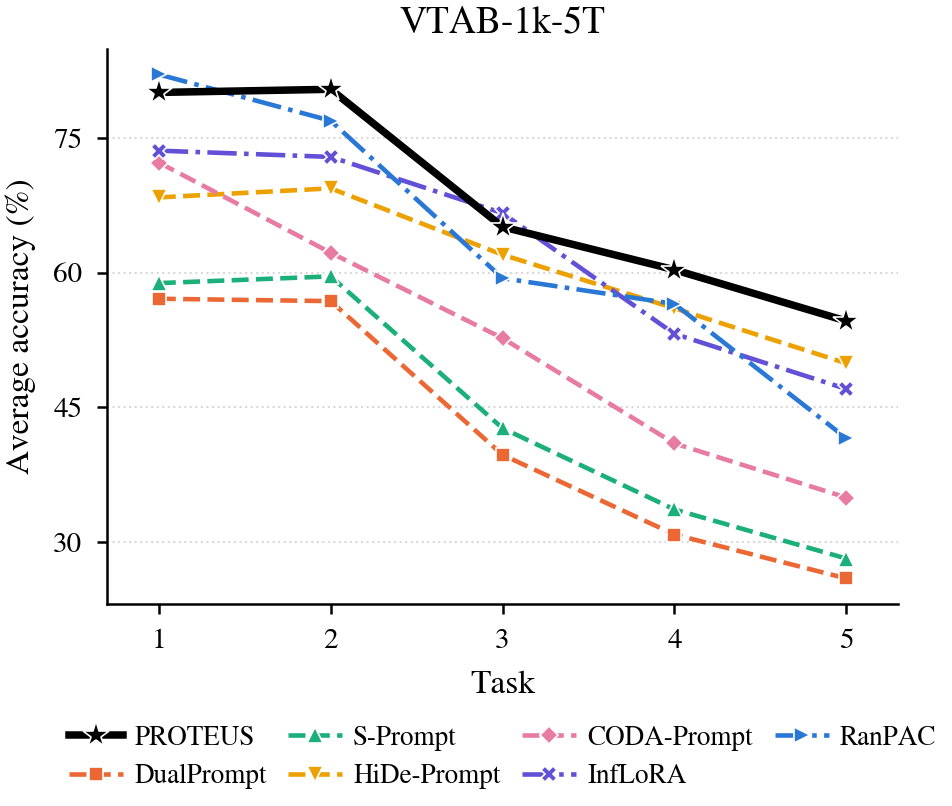}
        \caption{VTAB-1k-5T}
        \label{fig:img5}
    \end{subfigure}\hfill
    \begin{subfigure}[b]{0.32\linewidth}
        \centering
        \includegraphics[width=\linewidth]{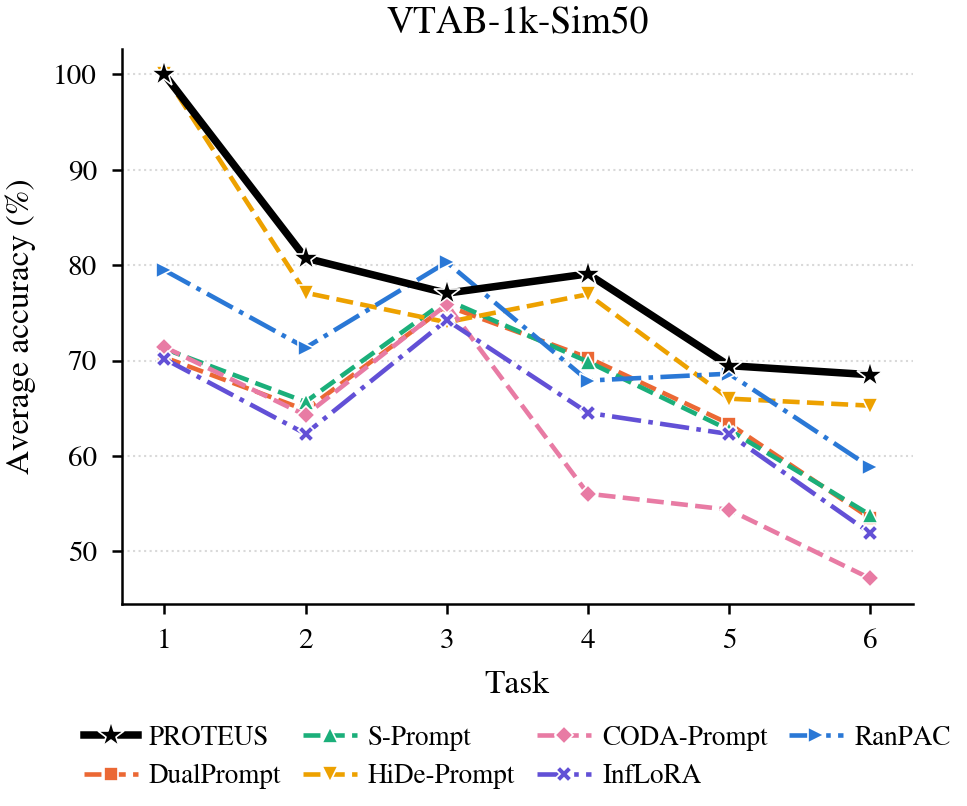}
        \caption{VTAB-1k-Sim50}
        \label{fig:img6}
    \end{subfigure}

    \caption{Average accuracy $\mathrm{ACC}_\kappa$ after learning each task $\kappa$ for \textsc{Proteus} and six baselines. Top row: CIFAR-100 and ImageNet-R (uniformly mild), ImageNet-A (uniformly abrupt); bottom row: VTAB5T and VTAB-1k-5T (uniformly abrupt), VTAB-1k-Sim50 (varying). Curves show a single run, whereas Table~\ref{tab: main-table} reports means over 3 seeds. Y-axis ranges differ across panels.}\vspace{-6mm}
    \label{fig:performance_progress}
\end{figure}

Figure~\ref{fig:performance_progress} shows the \textbf{performance progress} of \textsc{Proteus} and six baselines after each task on all six benchmarks; the numbers below refer to the illustrated single run. \textsc{Proteus} finishes at the top on all six benchmarks, tied with HiDe-Prompt on CIFAR-100. At the first task, where only one module exists and retrieval is trivial, some baselines start higher on CIFAR-100, ImageNet-R, ImageNet-A, and VTAB-1k-5T, and \textsc{Proteus} takes the lead as tasks accumulate. On CIFAR-100, ImageNet-R, and ImageNet-A (top row), \textsc{Proteus} and HiDe-Prompt form the two leading curves and decline the least over the sequence, with \textsc{Proteus} ahead at most tasks. The gap is largest on ImageNet-A, where \textsc{Proteus} leads HiDe-Prompt at every task and ends about 10.1\% ahead. On the VTAB benchmarks (bottom row), \textsc{Proteus} leads at every task on VTAB5T, while RanPAC is competitive at individual tasks, leading at the first task on VTAB-1k-5T and at the third task on VTAB-1k-Sim50. Because RanPAC keeps the embedding learned on the first task, it cannot adapt as later tasks shift and ends about 13.0\% and 9.6\% below \textsc{Proteus} on these two benchmarks, respectively. On VTAB-1k-Sim50, HiDe-Prompt tracks \textsc{Proteus} closely, staying within about 3.5\% at every task. Overall, \textsc{Proteus} remains at or near the top throughout each sequence, consistent with its high retrieval accuracy (Table~\ref{matching_comparison}) and the accuracy drop observed when retrieval is removed (Table~\ref{wo_matching}).

\vspace{-3mm}
\subsection{Retrieval Analysis}
\label{sec:ablation}

\begin{table}[b]
    \centering
    \begin{small}
    \caption{Retrieval accuracy (\%) of retrieval-based CL methods, defined as the percentage of single-task test batches assigned to their true task by plurality vote. All methods share the same test batch size for each dataset, listed in App.~\ref{sec: data and hyper}. Results are reported as mean$_{\pm\text{std}}$ over 3 random seeds.}
    \setlength{\tabcolsep}{3pt}
    \begin{tabular}{|lcccccc|}
        \toprule
        \textbf{Methods} & \textbf{CIFAR-100} & \textbf{ImageNet-R} & \textbf{ImageNet-A} & \textbf{VTAB5T} & \textbf{VTAB-1k-5T} & \textbf{VTAB-1k-Sim50} \\
        \midrule
        DualPrompt & $79.13_{\pm 2.56}$ & $76.51_{\pm 1.26}$ & $22.25_{\pm 5.14}$ & $86.39_{\pm 6.76}$ & $\mathbf{100.00_{\pm 0.00}}$ & $72.82_{\pm 3.54}$ \\
        S-Prompt & $76.91_{\pm 3.68}$ & $75.63_{\pm 1.16}$ & $19.62_{\pm 3.55}$ & $86.60_{\pm 4.97}$ & $\mathbf{100.00_{\pm 0.00}}$ & $75.32_{\pm 4.98}$ \\
        HiDe-Prompt & $\mathbf{99.97_{\pm 0.00}}$ & $\mathbf{99.97_{\pm 0.00}}$ & $79.70_{\pm 6.75}$ & $99.47_{\pm 0.23}$ & $\mathbf{100.00_{\pm 0.00}}$ & $87.36_{\pm 1.73}$ \\
        \midrule
        \textsc{Proteus} & $99.12_{\pm 0.00}$ & $98.26_{\pm 0.01}$ & $\mathbf{87.43_{\pm 0.00}}$ & $\mathbf{100.00_{\pm 0.00}}$ & $99.61_{\pm 0.57}$ & $\mathbf{88.00_{\pm 1.91}}$ \\
        \bottomrule
    \end{tabular}
    \label{matching_comparison}
    \end{small}
\end{table}

\noindent \textbf{Impact of task retrieval on performance.} We compare \textsc{Proteus} with a variant that removes retrieval and always uses the LoRA module $\Delta\boldsymbol{\omega}_m$ of the last task, which accumulates knowledge transferred from all previous tasks. As shown in Table~\ref{wo_matching}, removing retrieval reduces accuracy by 17.57--33.71\% across benchmarks. The two largest drops occur in the uniformly abrupt scenario, 30.26\% on ImageNet-A and 33.71\% on VTAB-1k-5T. On VTAB-1k-5T, the variant retains less than 40\% of the full model's accuracy. This is consistent with the intuition that a module composed for the last task transfers worst to earlier tasks that differ most from it. A single module therefore cannot replace selecting the module that matches each test input.

\begin{wraptable}[15]{r}{0.4\textwidth}
    \vspace{-4mm}
    \centering
    \begin{small}
    \setlength{\tabcolsep}{4pt}
    \caption{Final average accuracy (\%) of \textsc{Proteus} with and without retrieval. Without retrieval, every input is processed by the last task's module.}
    \label{wo_matching}
    \begin{tabular}{|lcc|}
        \toprule
        \textbf{CL Datasets} & \textbf{w/ Ret} & \textbf{w/o Ret} \\
        \midrule
        CIFAR-100 & $\textbf{94.53}$ & $76.96$ \\
        ImageNet-R & $\textbf{82.71}$ & $57.07$ \\
        \midrule
        ImageNet-A & $\textbf{65.29}$ & $35.03$ \\
        VTAB5T & $\textbf{93.28}$ & $70.07$ \\
        VTAB-1k-5T & $\textbf{54.60}$ & $20.89$ \\
        \midrule
        VTAB-1k-Sim50 & $\textbf{66.88}$ & $48.06$ \\
        \bottomrule
    \end{tabular}
    \end{small}
\end{wraptable} 

\noindent \textbf{Retrieval accuracy.}~Table~\ref{matching_comparison} compares \textsc{Proteus} with the retrieval-based methods that select task-specific modules, namely DualPrompt, S-Prompt, and HiDe-Prompt, all evaluated under the same batch protocol. \textsc{Proteus} and HiDe-Prompt are the two most accurate retrievers. \textsc{Proteus} is higher on ImageNet-A by 7.73\%, on VTAB5T by 0.53\%, and on VTAB-1k-Sim50 by 0.64\%, and within 1.71\% on CIFAR-100, ImageNet-R, and VTAB-1k-5T. Both exceed DualPrompt and S-Prompt by a wide margin, with \textsc{Proteus} ahead by 12.68--67.81\% on all benchmarks except VTAB-1k-5T, where tasks come from distinct datasets and every retrieval-based method is near-perfect. Unlike HiDe-Prompt, which retrains its task-identity classifier after every task on replayed class statistics, \textsc{Proteus} reaches this accuracy with signatures that are fit once and frozen. Retrieval accuracy alone, however, does not determine final accuracy. On VTAB-1k-5T, DualPrompt, S-Prompt, and HiDe-Prompt retrieve the correct task for every batch, yet reach only 25.43\%, 28.15\%, and 49.70\% final accuracy (Table~\ref{tab: main-table}), whereas \textsc{Proteus} reaches 54.02\% with 99.61\% retrieval accuracy. With retrieval accuracy on par with HiDe-Prompt, \textsc{Proteus} still attains higher final accuracy, suggesting that accurate retrieval pays off most when the retrieved adapters are specialized to their tasks, as the batch-size curves showed in Figure~\ref{fig:batch}. On ImageNet-A, the more accurate retrieval of \textsc{Proteus} adds to this advantage.

\subsection{Analysis of Adaptive LoRA Composition}
\label{sec:composition-analysis}

\begin{figure}[t]
    \centering
    \begin{subfigure}[b]{0.32\textwidth}
        \includegraphics[width=\linewidth]{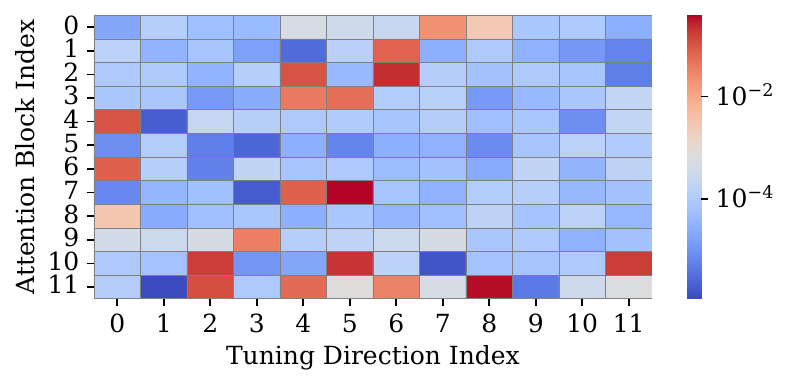}
        \caption{After task 4, $\lambda > 0$}
        \label{fig:a}
    \end{subfigure}
    \hfill
    \begin{subfigure}[b]{0.32\textwidth}
        \includegraphics[width=\linewidth]{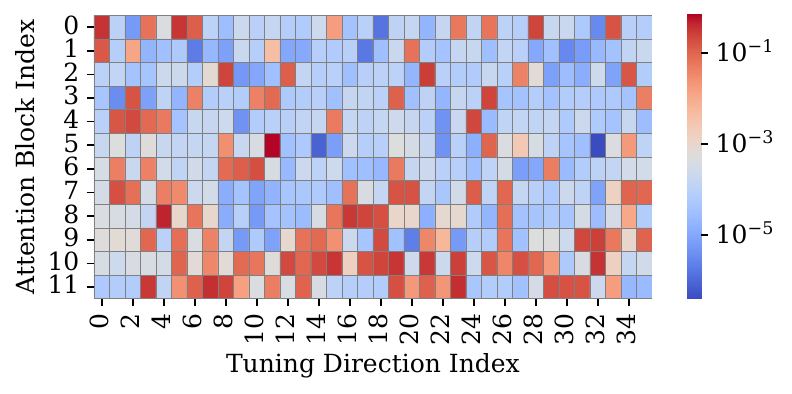}
        \caption{After task 10, $\lambda > 0$}
        \label{fig:b}
    \end{subfigure}
    \hfill
    \begin{subfigure}[b]{0.32\textwidth}
        \includegraphics[width=\linewidth]{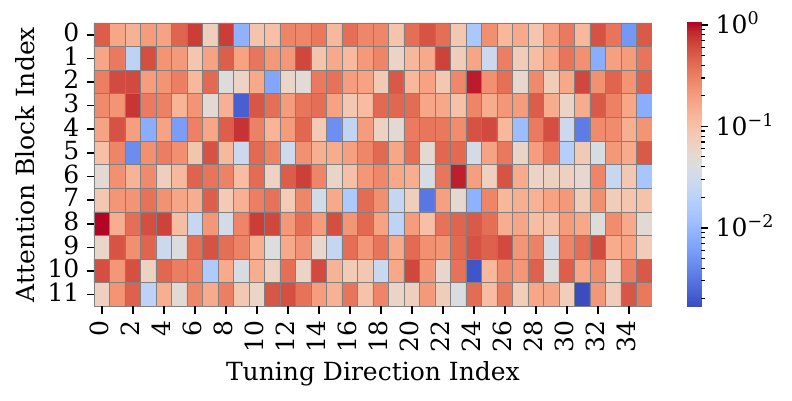}
        \caption{After task 10, $\lambda = 0$}
        \label{fig:c}
    \end{subfigure}
    \vspace{-2mm}
    \caption{Magnitudes of the diagonal entries of the knowledge-transfer matrix $\boldsymbol{S}$ on ImageNet-R. Rows index the 12 attention blocks, and columns index the LoRA units accumulated so far (4 per task).}\vspace{-6mm}
    \label{fig:heatmap}
\end{figure}

{\bf Sparsity of Knowledge Transfer.} The elastic-net penalty in Eq.~\ref{eq:15_only} encourages each new adapter to reuse only a few prior LoRA units. Fig.~\ref{fig:heatmap}(a--b) shows that $\boldsymbol{S}$ indeed remains highly sparse on ImageNet-R, with each task drawing on a small subset of prior units. Removing the penalty ($\lambda = 0$) yields a much denser $\boldsymbol{S}$ (Fig.~\ref{fig:heatmap}(c)) and severely degrades performance: final accuracy drops from 82.71\% to 30.49\% and retrieval accuracy from 98.26\% to 41.97\% (Table~\ref{tab:ablationS}, App.~\ref{sec:orthogonal}). The retrieval drop suggests that unfiltered reuse makes adapters overly similar, so that their signatures overlap. Weighting all prior units equally ($\boldsymbol{S} = \boldsymbol{I}$) degrades performance similarly, whereas disabling transfer altogether ($\boldsymbol{S} = \boldsymbol{0}$) lowers final and retrieval accuracy to 78.98\% and 87.05\%, indicating that reuse helps only when it is selective.

{\bf Efficiency.}~The inference cost of \textsc{Proteus} grows linearly with the number of tasks, since each input is embedded under every stored adapter. App.~\ref{sec: complexity} shows that these forward passes dominate the latency, and since they are independent across tasks and components, they can be parallelized (see App.~\ref{limit} for further analyses). During training, \textsc{Proteus} uses the least GPU memory among all compared methods on ImageNet-A, with higher accuracy than HiDe-Prompt (see App.~\ref{sec:gpu_footprint}). Under a common counting protocol, it trains the fewest parameters per task among retrieval-based methods on all six benchmarks, $6.5$--$34.8\times$ fewer than HiDe-Prompt (see App.~\ref{sec:param_compare}). An exponentially decaying LoRA rank further bounds the total memory while preserving accuracy (see App.~\ref{sec:mem_cost}). Limitations and future works are dicussed in App.~\ref{limit}.

{\bf Additional Ablations.}~Larger test batches improve both retrieval and final accuracy, and \textsc{Proteus} outperforms HiDe-Prompt even when both retrieve every batch correctly (see App.~\ref{sec:orthogonal}). Replacing the DP-GMM with K-Means lowers final and retrieval accuracy by up to 9.95\% and 23.80\%, confirming the value of adaptive, likelihood-based clustering. Scoring each task by its best-matching component outperforms top-$k$ aggregation, and the LoRA directions of different tasks remain nearly orthogonal across layers. Finally, a balanced $\ell_1$/$\ell_2$ penalty on $\boldsymbol{S}$ performs best, while a pure $\ell_1$ penalty ($\alpha = 1$) lowers accuracy (see App.~\ref{sec: alpha-lambda}).

\vspace{-10pt}
\section{Conclusion}

We presented \textsc{Proteus}, a parameter-free retrieval framework for continual learning with pre-trained models. It pairs each task-specific adapter with a statistical signature of its induced embeddings, fit once and frozen thereafter, and retrieves adapters by likelihood matching with a plurality vote over test batches, so that no retrieval parameters are shared across tasks and no task identifier needs retraining. Our analysis characterizes when this retrieval succeeds, both per input and per batch.  Guided by this analysis, our adaptive LoRA composition sparsely reuses prior directions while adding orthogonal new ones, which encourages separation and, as our results suggest, yields more specialized adapters. Across six class-incremental benchmarks, \textsc{Proteus} achieves the best final accuracy and the lowest average forgetting, with the largest gains under abrupt task shifts, while its retrieval accuracy is comparable to HiDe-Prompt without any retraining or replay. Its inference cost, however, grows with the number of tasks, and reducing this cost through subset voting, as well as extending \textsc{Proteus} to task-free settings, are natural next steps (App.~\ref{limit}). More broadly, we hope \textsc{Proteus} encourages further work on structure-aware, parameter-free retrieval, in which adaptation modules are designed to be retrievable by construction, as a foundation for scalable and reliable continual learning.

\bibliography{references}
\bibliographystyle{tmlr}

\newpage
\appendix
\textbf{Code Release.} Our experimental code is released and maintained at \url{https://anonymous.4open.science/r/PROTEUS-2C4B/}.

\section{Bounding the Task-Retrieval Error Rate}\label{retrieval-err}
This section analyzes the error rate of the parameter-free retrieval procedure in Section~\ref{subsec:param-free}. Section~\ref{app:assumptions} states the assumptions, Section~\ref{app:overview} summarizes the results, and Sections~\ref{app:per-input}--\ref{app:separation} prove them step by step.

\subsection{Assumptions}\label{app:assumptions}
\begin{assumption}[Clustered Data Distribution] 
\label{assumption:1}
Following our representation clustering algorithm in Eq.~\ref{eq:18}, we assume that for each task $k$ with a multi-key signature with $\tau$ components $(\boldsymbol{m}_k^t, \boldsymbol{\Lambda}_k^t)_{t=1}^\tau$, its data distribution $\mathcal{D}_k$ can be decomposed into $\tau$ sub-distributions $\{\mathcal{D}_k^t\}_{t=1}^\tau$ such that:
\begin{eqnarray}
\hspace{-5mm}\boldsymbol{h}_k(\boldsymbol{x}) \ \sim\ \mathbb{N}\big(\boldsymbol{m}^t_k, \boldsymbol{\Lambda}^t_k\big) \quad\text{when}\quad\boldsymbol{x} \sim \mathcal{D}_k^t(\boldsymbol{x}) \quad\text{with}\quad \boldsymbol{\Lambda}_k^t\quad\text{full rank}\ . \label{eq:a1}
\end{eqnarray}
This is a reasonable assumption as our clustering algorithm was developed to fit a $\tau$-component mixture of Gaussians $(\boldsymbol{m}_k^t, \boldsymbol{\Lambda}_k^t)_{t=1}^\tau$ to the input embeddings of each task $k$.
\end{assumption}

\begin{assumption}[Cluster Separation and Regularity] \label{assumption:2}
For every false component $(i,s)$ with $i\neq k$ and $\boldsymbol{x} \sim \mathcal{D}_k^t$, let $Z_i^s \triangleq \big(\boldsymbol{h}_i(\boldsymbol{x})-\boldsymbol{m}_i^s \big)^\top\big(\boldsymbol{\Lambda}_i^s \big)^{-1}\big(\boldsymbol{h}_i(\boldsymbol{x})-\boldsymbol{m}_i^s \big)$ denote the squared Mahalanobis distance of Eq.~\ref{eq:score}, and split
\begin{equation}
    \boldsymbol{h}_i(\boldsymbol{x})-\boldsymbol{m}_i^s = \boldsymbol{\mu}_i^s + \boldsymbol{Y}_i^s, \qquad \boldsymbol{\mu}_i^s\triangleq \mathbb{E}\big[\boldsymbol{h}_i(\boldsymbol{x}) \big] - \boldsymbol{m}_i^s, \qquad \boldsymbol{Y}_i^s\triangleq \boldsymbol{h}_i(\boldsymbol{x}) - \mathbb{E}\big[\boldsymbol{h}_i(\boldsymbol{x})\big]\ ,
\end{equation}
where $\boldsymbol{\mu}_i^s$ is the offset between the mean embedding of task-$k$ inputs under the false embedding function $\boldsymbol{h}_i$ and the mean of component $s$ of task $i$, and $\boldsymbol{Y}_i^s$ is the zero-mean fluctuation. We assume the following regularity conditions:

{\bf A2-1. Separation.}~For every false component $(i,s)$ with $i \neq k$, the expected squared Mahalanobis distance of an input $\boldsymbol{x} \sim \mathcal{D}_k^t$ to that component, measured under its own embedding function $\boldsymbol{h}_i$, exceeds $d$ by at least a multiplicative {\em separation factor} $1 + \delta$ (Definition~\ref{def:delta}):
\begin{eqnarray}
\mathbb{E}_{\boldsymbol{x} \sim \mathcal{D}_k^t}\left[Z_i^s(\boldsymbol{x})\right] \geq \big(1 + \delta\big)d  \ , \label{eq:spread-app}
\end{eqnarray}
i.e., $\delta \triangleq \min_{(i,s):\, i \neq k} \mathbb{E}_{\boldsymbol{x} \sim \mathcal{D}_k^t}[Z_i^s(\boldsymbol{x})]/d - 1$, the largest factor for which Eq.~\ref{eq:spread-app} holds for every false component, and we assume $\delta \geq 0$.

{\bf A2-2. Concentration.}~The fluctuation $\boldsymbol{Y}_i^s$ has the convex concentration property (CCP) with constant $K$. That means for every convex, 1-Lipschitz function $\varphi: \mathbb{R}^d \rightarrow \mathbb{R}$ and every $t>0$,
\begin{equation*}
    \text{Pr}\Big(\big|\varphi(\boldsymbol{Y}_i^s) - \mathbb{E}\varphi(\boldsymbol{Y}_i^s) \big| \geq t \Big) \leq 2\exp{(-t^2/K^2)}\ .
\end{equation*}

{\bf A2-3. Well-conditioning.}~The cluster covariances are uniformly well-conditioned: $\lambda_{\min}(\boldsymbol{\Lambda}_i^s) \geq \lambda_0 > 0$ for all $(i,s)$, i.e., no cluster covariance collapses toward singularity along any direction. For fitted signatures, this holds by construction through the eigenvalue floor in Eqs.~\ref{eq:retrieval_gaussian} and~\ref{eq:18}.
\end{assumption}

\subsection{Overview of Results}\label{app:overview}
Under these assumptions, we establish three results, first stated informally and then proved.

{\bf R1.}~The per-input retrieval error of \textsc{Proteus} decays exponentially at a rate governed by a rational function of the embedding dimension $d$, the separation factor $\delta$, and the variance factor $\sigma^2$, while growing linearly with the number of competing components $(n-1)\tau$ (Theorem~\ref{thm:1}).

{\bf R2.}~Plurality voting over a single-task batch drives the batch error down exponentially in the batch size $B$, provided the per-input error is below one half (Corollary~\ref{cor:vote}).

{\bf R3.}~A sufficiently large separation factor offsets the linear growth in $(n-1)\tau$ and drives the per-input error below any target tolerance, with a threshold that grows at most logarithmically with the number of tasks (Theorem~\ref{thm:2}).

Our adapter design aims to encourage such separation through orthogonality between adaptation directions (Section~\ref{sec:adaptive_ft}), which we evaluate empirically in Section~\ref{sec:exprs}.

\subsection{Per-Input Retrieval Error}\label{app:per-input}

\begin{theorem}[Bounding Retrieval Error Rate Per Input]\label{thm:1}
Under Assumptions~\ref{assumption:1}--\ref{assumption:2}, suppose we are given $n$ tasks with clustered data distributions $(\mathcal{D}_1^t)_{t=1}^\tau, (\mathcal{D}_2^t)_{t=1}^\tau, \ldots, (\mathcal{D}_n^t)_{t=1}^\tau$. Let $\boldsymbol{E}_k^t$ denote the event that a test input $\boldsymbol{x} \sim \mathcal{D}_k^t$ is not matched with the correct embedding function $\boldsymbol{h}_k(\boldsymbol{x})\in \mathbb{R}^d$ by Eq.~\ref{eq:per-input-retrieval}. Then
\begin{eqnarray}
\hspace{-8.75mm}\mathrm{Pr}\Big(\boldsymbol{E}_k^t\Big) &\leq& \exp\left(-\frac{d\delta^2}{8\delta + 16}\right) \ \ +\ \  4(n-1)\tau\ \exp\left(-c\frac{(d\delta/2 + \kappa)^2}{\sigma^2d + b(d\delta/2 + \kappa)}\right) \ \ \text{with}\ \ \\
\hspace{-8.75mm}\kappa &\triangleq& \min_{(i,s): i \ne k} \log\frac{|\boldsymbol{\Lambda}_i^s|}{|\boldsymbol{\Lambda}_k^t|}, \quad \sigma^2\triangleq\frac{K^4}{\lambda_0^2}+\frac{K^2(1+\delta)}{\lambda_0}, \quad b\triangleq \frac{K^2}{\lambda_0} \ .\label{eq:bound-app}
\end{eqnarray}
where $c>0$ is a universal constant (independent of $d,\delta, n, \tau, K, \lambda_0$), $K$ is the concentration constant in A2-2, and $\delta \geq \max\left(0, -\frac{2\kappa}{d}\right)$. Here, $\kappa$ is the minimum difference between the log-volumes of the false clusters $(\boldsymbol{m}_i^s, \boldsymbol{\Lambda}_i^s)$ with $i \ne k$ and that of the true cluster $(\boldsymbol{m}_k^t, \boldsymbol{\Lambda}_k^t)$.
\end{theorem}

{\bf Proof.}~To avoid cluttering the notation, we assume the same number $\tau$ of representation clusters per task; the proof extends directly to a different $\tau$ per task.

Under Assumption~\ref{assumption:1}, the embedding function $\boldsymbol{h}_k(\boldsymbol{x})$ transforms each sub-distribution $\mathcal{D}_k^t$ of task $k$ into a Gaussian $\mathbb{N}(\boldsymbol{m}_k^t, \boldsymbol{\Lambda}_k^t)$.~Hence, when $\boldsymbol{E}_k^t$ occurs, there must exist $(i, s): i \ne k$ for which, 

\begin{eqnarray}
\hspace{-15mm}\log \mathbb{N}(\boldsymbol{h}_i(\boldsymbol{x}) \mid \boldsymbol{m}_i^s, \boldsymbol{\Lambda}_i^s) &>& \log \mathbb{N}(\boldsymbol{h}_k(\boldsymbol{x}) \mid \boldsymbol{m}_k^t, \boldsymbol{\Lambda}_k^t) \ \ \Rightarrow \ \ Z_k^t \ \ >\ \ Z_i^s \ \ +\ \ \log\frac{\left|\boldsymbol{\Lambda}_i^s\right|}{\left|\boldsymbol{\Lambda}_k^t\right|} \ ,\label{eq:a2}
\end{eqnarray}
where $Z_k^t = (\boldsymbol{h}_k(\boldsymbol{x}) - \boldsymbol{m}_k^t)^\top(\boldsymbol{\Lambda}_k^t)^{-1}(\boldsymbol{h}_k(\boldsymbol{x}) - \boldsymbol{m}_k^t)$ and $Z_i^s = (\boldsymbol{h}_i(\boldsymbol{x}) - \boldsymbol{m}_i^s)^\top(\boldsymbol{\Lambda}_i^s)^{-1}(\boldsymbol{h}_i(\boldsymbol{x}) - \boldsymbol{m}_i^s)$.

This consequently implies
\begin{eqnarray}
\boldsymbol{E}^t_k & \Rightarrow&  Z_k^t \ \ >\ \ \min_{(i,s): i \ne k} \left(Z_i^s \ \ +\ \ \log\frac{\left|\boldsymbol{\Lambda}_i^s\right|}{\left|\boldsymbol{\Lambda}_k^t\right|}\right) \ ,\\
&\Rightarrow& Z_k^t \ \ >\ \ \min_{(i,s): i \ne k} \big(Z_i^s \ \ +\ \ \kappa\big) \ \ \equiv\ \ \boldsymbol{Q}_k^t \ . \label{eq:a3}
\end{eqnarray}
Hence, $\mathrm{Pr}(\boldsymbol{E}^t_k) \leq \mathrm{Pr}(\boldsymbol{Q}^t_k)$. Moreover, for any choice of $\theta$,
\begin{eqnarray}
\mathrm{Pr}\Big(\boldsymbol{E}_k^t\Big) &\leq& \mathrm{Pr}\Big(\boldsymbol{Q}_k^t\Big) \ \ \leq\ \  \mathrm{Pr}\Big(Z_k^t \geq \theta\Big) \ +\ \sum_{i\ne k}\sum_{s=1}^\tau \mathrm{Pr}\Big(Z_i^s < \theta -\kappa\Big) \ ,    \label{eq:a4}
\end{eqnarray}
which holds due to the union bound.

\textbf{Bound true-side tail $\mathrm{Pr}\Big(Z_k^t \geq \theta\Big)$}.
Under Assumption~\ref{assumption:1}, $Z_k^t$ is a chi-square random variable with $d$ degrees of freedom, i.e., $Z_k^t \sim \chi^2(d)$. By the Laurent--Massart inequality~\cite{laurent2000adaptive}, for any $x \geq 0$, $\text{Pr}\Big(Z_k^t - d \geq 2\sqrt{dx}+2x\Big) \leq \exp{(-x)}$ (i.e., with probability $1-\exp{(-x)}$, $Z_k^t$ exceeds its mean $d$ by less than $ 2\sqrt{dx}+2x$). Given $u\geq 0$, set $x=u^2\big/\big(4(d+u)\big)$, then $ 2\sqrt{dx}+2x \leq u$. Use the Laurent-Massart inequality with $u$, we have:
\begin{equation}
    \text{Pr}\Big(Z_k^t - d \geq u \geq 2\sqrt{dx}+2x \Big) \leq \exp{\bigg(-\frac{u^2}{4d + 4u} \bigg)} \label{eq:true-tail-end}
\end{equation}

\textit{To answer why $2\sqrt{dx}+2x \leq u$ when $x=u^2\big/\big(4(d+u)\big)$.} Substituting $x$ gives
\begin{equation*}
    ( 2\sqrt{dx}+2x) / u = \sqrt{\frac{d}{d+u}} + \frac{u}{2(d+u)}
\end{equation*}
Let $r=\frac{u}{d+u} \Rightarrow 1-r=\frac{d}{d+u}$, substitute $r$:
\begin{equation*}
    ( 2\sqrt{dx}+2x) / u = \sqrt{1-r} + r/2
\end{equation*}
On the other hand, $1-r\leq 1-r + r^2/4$, for $r \in [0,1)$. This leads us to $\sqrt{1-r} \leq 1-r/2 \Leftrightarrow  \sqrt{1-r} + r/2 \leq 1$. Or $( 2\sqrt{dx}+2x) / u \leq 1$. Hence $ 2\sqrt{dx}+2x \leq u$.

\textbf{Bound false-side tail $\text{Pr}\Big(Z_i^s < \theta - \kappa\Big)$}. 
We fix one false component, drop its indices, and let $\boldsymbol{A} = \boldsymbol{\Lambda}^{-1}$. Expand $Z=(\boldsymbol{\mu}+\boldsymbol{Y})^\top\boldsymbol{A}(\boldsymbol{\mu}+\boldsymbol{Y})$,
\begin{equation}
    Z=\boldsymbol{Y}^\top\boldsymbol{A}\boldsymbol{Y} + 2(\boldsymbol{A\mu})^\top\boldsymbol{Y} + \boldsymbol{\mu}^\top\boldsymbol{A}\boldsymbol{\mu}\ . \label{eq:false-tail}
\end{equation}
Considering:
\begin{itemize}
    \item \textit{Mean.} Since $\mathbb{E}[\boldsymbol{Y}]=0$ (Assumption~\ref{assumption:2}), Let $m \triangleq \mathbb{E}[Z]=\mathbb{E}\big[\boldsymbol{Y}^\top\boldsymbol{A}\boldsymbol{Y}\big] + \boldsymbol{\mu}^\top\boldsymbol{A}\boldsymbol{\mu} \geq (1+\delta)d$ (by A2-1).

    Let $Z-m=Q+L$ with $Q\triangleq \boldsymbol{Y}^\top\boldsymbol{A}\boldsymbol{Y} - \mathbb{E}\big[\boldsymbol{Y}^\top\boldsymbol{A}\boldsymbol{Y}\big]$ and $L\triangleq 2(\boldsymbol{A\mu})^\top\boldsymbol{Y}$.
    
    \item \textit{Norms.} By A2-3, $\|\boldsymbol{A} \|_{\text{op}}\leq 1/\lambda_0$, so $\|\boldsymbol{A} \|^2_F \leq d/\lambda_0^2$. Since $\boldsymbol{A}$ is symmetric, $\|\boldsymbol{A\mu} \|^2 = \boldsymbol{\mu}^\top\boldsymbol{A}^2\boldsymbol{\mu} \leq \frac{1}{\lambda_0}\boldsymbol{\mu}^\top\boldsymbol{A}\boldsymbol{\mu} \leq \frac{{m}}{\lambda_0}$.
    
    \item \textit{Quadratic part.} Applying the Hanson--Wright inequality for zero-mean random vectors with the convex concentration property~\cite{adamczak2015note},
    \begin{equation*}
        \text{Pr}(Q\leq -t) \leq 2\exp{\bigg(-c_1\min\Big\{\frac{t^2}{K^4\|\boldsymbol{A} \|_F^2},\frac{t}{K^2\|\boldsymbol{A} \|_\text{op}} \Big\} \bigg)} \leq 2\exp{\bigg(-c_1\min\bigg\{\frac{t^2\lambda_0^2}{K^4d},\frac{t\lambda_0}{K^2}\bigg\} \bigg)}\ .
    \end{equation*}

    \item \textit{Linear part.} The map $y \mapsto (\boldsymbol{A\mu})^\top y\big/ \| \boldsymbol{A\mu}\|$ is linear, hence convex and 1-Lipschitz. By A2-2,
    \begin{equation*}
        \text{Pr}(L\leq -t) \leq 2\exp{\bigg(-\frac{t^2}{4K^2\|\boldsymbol{A\mu} \|^2} \bigg)} \leq 2\exp{\bigg( -\frac{t^2\lambda_0}{4K^2m}\bigg)}
    \end{equation*}
\end{itemize}

Combining these bounds with~\eqref{eq:false-tail},
\begin{equation*}
    \text{Pr}(Z\leq m-u) \leq \text{Pr} (Q \leq -u/2) + \text{Pr}(L \leq -u/2)
\end{equation*}
Let $X_0=\frac{K^4d}{\lambda_0^2}, b=\frac{K^2}{\lambda_0},$ and $D(m)=X_0+bm$.
\begin{eqnarray}
    \text{Pr} (Q \leq -u/2) &\leq& 2\exp{\bigg(-c_1\min\bigg\{\frac{(u/2)^2\lambda_0^2}{K^4d},\frac{(u/2)\lambda_0}{K^2}\bigg\} \bigg)} \nonumber \\
    &=& 2\exp{\bigg(-c_1\min\bigg\{\frac{(u/2)^2}{X_0},\frac{(u/2)}{b}\bigg\} \bigg)} \nonumber \\
    &\leq& 2\exp{\bigg(-c_1\frac{u^2}{4}\frac{1}{X_0+bu/2}\bigg)}\quad \text{use } \min\{t^2/X, t/Y\}\geq t^2/(X+tY) \nonumber \\
    &\leq& 2\exp{\bigg(\frac{-c_1u^2}{4}\frac{1}{D(m)+bu}\bigg)} \quad \text{since}\quad X_0+bu/2\leq D(m)+bu
\end{eqnarray}

\begin{eqnarray}
    \text{Pr}(L\leq -u/2) &\leq& 2\exp{\bigg( -\big(\frac{u}{2}\big)^2 \Big/ \big(\frac{4K^2m}{\lambda_0}\big)\bigg)} \nonumber \\
    &=& 2\exp{\bigg(-\frac{u^2}{16} \Big/\frac{K^2m}{\lambda_0}\bigg)} \nonumber \\
    &\leq& 2\exp{\bigg(-\frac{u^2}{16}\frac{1}{D(m)+bu}\bigg)} \quad \text{since}\quad \frac{K^2m}{\lambda_0} \leq D(m) \leq D(m) +bu
\end{eqnarray}
Hence, with $c=\min(c_1/4, 1/16)$, we have:
\begin{equation}
    \text{Pr}(Z\leq m-u) \leq 4\exp{\bigg(-c\frac{u^2}{D(m)+bu}\bigg)}, \quad\text{where}\quad D(m)=\frac{K^4d}{\lambda_0^2} + \frac{K^2}{\lambda_0}m\ . \label{eq:comb_false_tail}
\end{equation}

Since this bound depends on $m$, which differs across clusters, we evaluate it at the worst case $m_0$, the smallest allowed mean, so that it holds for every false component.

\textbf{Finding the worst case of $m_0$.}~Let $a=\theta-\kappa$. If $a\leq 0$, then $\text{Pr}(Z\leq a)=0$ since $Z\geq0$. So assume $0<a\leq m$ and apply~\eqref{eq:comb_false_tail} with $u=m-a$. Writing $D(m)=\alpha+\beta m$ with $\alpha,\beta>0$, the exponent becomes $g(m)=(m-a)^2 \big/ (\alpha+\beta m+b(m-a))$. Take the first derivative gives
\begin{eqnarray}
    g'(m) \propto (m-a)\big[2\alpha+\beta m + \beta a + b(m-a) \big] \geq 0\ ,
\end{eqnarray}
so $g$ is nondecreasing and the bound is weakest at $m_0=(1+\delta)d$ (by A2-1), and $D(m_0)=d\sigma^2$ with $\sigma^2\triangleq\frac{K^4}{\lambda_0^2}+\frac{K^2(1+\delta)}{\lambda_0}$. Thus,
\begin{equation}
    \text{Pr}(Z_i^s < \theta - \kappa) \leq 4\exp{\bigg(-c\frac{(m_0-\theta+\kappa)^2}{d\sigma^2+b(m_0-\theta+\kappa)} \bigg)}\ . \label{eq:false-tail-end}
\end{equation}

\textbf{Choice of $\theta$.} Set $\theta=(1+\delta/2)d$, which satisfies $\theta\geq d$ since $\delta \geq 0$. Applying~\eqref{eq:true-tail-end} with $u=d\delta/2$ yields
\begin{eqnarray}
    \text{Pr}(Z_k^t \geq \theta
    ) \leq \exp{\bigg(-\frac{(d\delta)^2}{4(4d+2d\delta)}  \bigg)} = \exp{\bigg(-\frac{d\delta^2}{16+8\delta}\bigg)}
\end{eqnarray}
Moreover, $m_0-\theta+\kappa = (1+\delta)d - (1+\delta/2)d+\kappa = d\delta/2+\kappa \geq 0$ by the premise $\delta \geq -2\kappa/d$. Substituting into~\eqref{eq:false-tail-end},
\begin{eqnarray}
    \text{Pr}(Z_i^s < \theta - \kappa) \leq 4\exp{\bigg( -c\frac{(d\delta/2+\kappa)^2}{d\sigma^2+b(d\delta/2+\kappa)}\bigg)}
\end{eqnarray}

Plugging both bounds into~\eqref{eq:a4} and taking the union over the $(n-1)\tau$ false components gives
\begin{eqnarray}
\hspace{-8.75mm}\mathrm{Pr}\Big(\boldsymbol{E}_k^t\Big) &\leq& \exp\left(-\frac{d\delta^2}{8\delta + 16}\right) \ \ +\ \  4(n-1)\tau\ \exp\left(-c\frac{(d\delta/2 + \kappa)^2}{\sigma^2d + b(d\delta/2 + \kappa)}\right) \ \ \text{with}\ \ \\
\hspace{-8.75mm}\kappa &\triangleq& \min_{(i,s): i \ne k} \log\frac{|\boldsymbol{\Lambda}_i^s|}{|\boldsymbol{\Lambda}_k^t|}, \quad \sigma^2\triangleq\frac{K^4}{\lambda_0^2}+\frac{K^2(1+\delta)}{\lambda_0}, \quad b\triangleq \frac{K^2}{\lambda_0} \ .\label{eq:bound-app-end}
\end{eqnarray}

This completes our proof for Theorem~\ref{thm:1}. 

{\bf Remark 1.}~Once training and signature fitting are complete, $K$, $\lambda_0$, and $\kappa$ are fixed constants. Treating them as such, the bound behaves as $\mathrm{Pr}(\boldsymbol{E}_k^t) \lesssim n\tau \exp\big(-\Omega\big(d\min(\delta, \delta^2)\big)\big)$: the error decays exponentially in $d\delta$ once $\delta$ is bounded away from zero, and as $\exp(-\Omega(d\delta^2))$ for small $\delta$, while growing linearly with the $(n-1)\tau$ competing components. Two complementary levers counter this growth: plurality voting over test batches (Section~\ref{app:batch}) and sufficient separation (Section~\ref{app:separation}).

\subsection{Batch-Wise Retrieval Error}\label{app:batch}
\textbf{Setup.} As in Section~\ref{subsec:param-free}, a test batch $\boldsymbol{x}_1,\ldots,\boldsymbol{x}_B$ consists of independent inputs from task $k$, where each input belongs to one component, $\boldsymbol{x}_q \sim \mathcal{D}_k^{t_q}$ for some $t_q\in [\tau]$, and adapters and signatures are fixed after training.

\textbf{Retrieval rule.} Each input $\boldsymbol{x}_q$ is retrieved individually by Eq.~\ref{eq:per-input-retrieval}, i.e., by the maximum signature log-likelihood over all components $(i,s)$, giving a predicted task $\hat{k}_q$. The batch is assigned to the plurality task $\hat{k}_B$ of Eq.~\ref{eq:plurality-vote}. Let $E_B^{\text{vote}}$ denote the event that $\hat{k}_B \neq k$.

\textbf{Per-input error bound.} Let $p_t$ denote the right-hand side of Theorem~\ref{thm:1} for component $t$, with its own $\kappa_t$, and define
\begin{equation}
    p \triangleq \max_t p_t = \max_t \bigg[\exp{\Big(-\frac{d\delta^2}{8\delta+16} \Big)} +4(n-1)\tau\exp{\Big(-c\frac{(d\delta/2+\kappa_t)^2}{d\sigma^2+b(d\delta/2+\kappa_t)} \Big)}\bigg]\ .
\end{equation}

\begin{corollary}[Plurality-Vote Batch Retrieval Error]\label{cor:vote}
    Suppose the assumptions of Theorem~\ref{thm:1} hold for every component $t$ of task $k$, and consider the setup above, with a test batch $\{\boldsymbol{x}_1,\ldots,\boldsymbol{x}_B\}$ where $\boldsymbol{x}_q \sim \mathcal{D}_k^{t_q}$. If $p < 1/2$, then
    \begin{eqnarray}
        \text{Pr}(E_B^{\text{vote}}) \leq \Big(4p(1-p) \Big)^{B/2} \ .
    \end{eqnarray}
    Hence, the batch error decreases exponentially in the batch size $B$, where $p$ is the maximum per-input error over the components of task $k$. For any batch tolerance $\epsilon_B > 0$, $\text{Pr}(E_B^{\text{vote}}) \leq \epsilon_B$ once $B\geq \frac{2\log(1/\epsilon_B)}{\log\big(1/(4p(1-p))\big)}$, so the required batch size grows only logarithmically in $1/\epsilon_B$.
\end{corollary}

\textbf{Proof.} Let $I_q = \boldsymbol{1} \{\hat{k}_q \neq k \}$ indicate that input $\boldsymbol{x}_q \sim \mathcal{D}_k^{t_q}$ is assigned a wrong task. Theorem~\ref{thm:1} applies directly:
\begin{equation*}
    p_q \triangleq \text{Pr}(I_q = 1) = \text{Pr}\big(E_k^{t_q} \big) \leq p_{t_q} \leq p\ .
\end{equation*}
Since each predicted task $\hat{k}_q$ is a deterministic function of $\boldsymbol{x}_q$ and the inputs are independent, $I_1,\ldots,I_B$ are independent Bernoulli($p_q$) variables. Let $N\triangleq\sum_qI_q$ be the number of wrong votes in the batch.

Let $U_1,\ldots,U_B$ be i.i.d.\ uniform on $[0,1]$, and set $\tilde{I}_q=\boldsymbol{1}\{U_q\leq p_q \}$ and $J_q = \boldsymbol{1}\{U_q\leq p \}$. Then $\tilde{I}_q$ has the same joint law as $I_q$. Since $\sum_qJ_q \sim \text{Bin}(B,p)$ and $\tilde{I}_q\leq J_q$ for every $q$ because $p_q\leq p$, we have for every threshold $a$,
\begin{equation*}
    \text{Pr}(N\geq a) = \text{Pr}\Big(\sum_q \tilde{I}_q \geq a \Big) \leq \text{Pr}\Big(\sum_qJ_q \geq a \Big) = \text{Pr}\big(\text{Bin}(B,p) \geq a \big)\ .
\end{equation*}

Suppose task $k$ receives more than $B/2$ votes, i.e., $N\geq B/2$. Then every other task receives at most $N < B/2$ votes, strictly fewer than task $k$, so task $k$ wins the plurality outright, regardless of how ties are broken. Contrapositively,
$E_B^{\text{vote}} \subseteq \{N \geq B/2 \}$. Hence,
\begin{equation*}
    \text{Pr}(E_B^{\text{vote}} ) \leq \text{Pr}(N\geq B/2) \leq \text{Pr}\big(\text{Bin}(B,p)\geq B/2 \big)\ .
\end{equation*}

For $a\geq p$, the Chernoff--Hoeffding bound~\cite{hoeffding1963probability} gives $\text{Pr}\big(\text{Bin}(B,p)\geq aB \big) \leq \exp{\big(-B\text{KL}(a \| p) \big)}$, with $\text{KL}(a \| p)= a\log\frac{a}{p}+(1-a)\log\frac{1-a}{1-p}$. With $a = 1/2 > p$,
\begin{equation*}
    \text{KL}\Big(\frac{1}{2}\|p \Big) = \frac{1}{2}\log\frac{1}{2p} + \frac{1}{2}\log\frac{1}{2(1-p)} = -\frac{1}{2}\log\big(4p(1-p) \big)
\end{equation*}
Use $\exp{(c\log x)} = x^c$ to $\text{Pr}(E_B^{\text{vote}} )$,
\begin{equation*}
    \text{Pr}(E_B^{\text{vote}} ) \leq \exp{\Big(\frac{B}{2}\log\big(4p(1-p)\big) \Big)} = \big(4p(1-p)\big)^{B/2}\ .
\end{equation*}

\subsection{Sufficient Separation}\label{app:separation}
\begin{theorem}[Sufficient Separation for Controlled Error]\label{thm:2}
    Under the assumptions of Theorem~\ref{thm:1}, for any tolerance $\epsilon > 0$, the retrieval error satisfies $\text{Pr}(E_k^t) \leq \epsilon$ whenever we can choose a sufficient separation $\delta$ as a function of $\epsilon$:
    \begin{eqnarray}
        \delta \geq C_1\frac{L}{d} + C_2\sqrt{\frac{L}{d}} + \frac{3\kappa^-}{d} &,& \quad\text{with}\quad C_1 = 8\max\Big(1,\frac{b}{c} \Big),\quad\text{and}\quad C_2 = 2\max\Big(2,\sqrt{\frac{\rho}{c}} \Big) \nonumber \\
        \text{where}\quad L\triangleq\log\frac{8n\tau}{\epsilon}, \quad \rho &\triangleq& \frac{K^4}{\lambda_0^2} + \frac{K^2}{\lambda_0}, \quad \kappa^- \triangleq \max(0, -\kappa)\ . \nonumber
    \end{eqnarray}
    That is, a separation of order $\sqrt{\log(n\tau/\epsilon)/d} + \log(n\tau/\epsilon)/d$ suffices, with constants depending only on $K^2/\lambda_0$ and the universal constant $c$.
\end{theorem}

\textbf{Proof.} By Theorem~\ref{thm:1}, $\text{Pr}(E_k^t) \leq T_1 + T_2$ with
\begin{eqnarray}
    T_1=\exp{\Big(-\frac{d\delta^2}{8\delta + 16} \Big)}, \quad T_2 = 4(n-1)\tau\exp{\Big(-c\frac{z^2}{d\sigma^2+bz} \Big)}, \quad z\triangleq d\delta/2 + \kappa,
\end{eqnarray}
valid for $\delta \geq \max(0,-2\kappa/d)$. We split the error budget and require $T_1 \leq \epsilon/2$ and $T_2\leq \epsilon/2$ separately. Let $L_1=\log(2/\epsilon)$ and $L_2=\log\big(8(n-1)\tau/\epsilon\big)$; both are at most $L$.

\textbf{First term.} $T_1 \leq \epsilon/2$ iff 
\begin{eqnarray}
    \exp{\Big(-\frac{d\delta^2}{8\delta+16} \Big)} &\leq& \epsilon/2 \nonumber \\
    \Leftrightarrow \frac{d\delta^2}{8\delta+16} &\geq& \log(2/\epsilon) \nonumber \\
    \Leftrightarrow d\delta^2 &\geq& L_1(8\delta + 16) \quad\text{by the definition of }L_1 \nonumber\\
    \Leftrightarrow \delta &\geq& \delta_1\triangleq \frac{4}{d}\Big(L_1+\sqrt{L_1^2+dL_1} \Big) \label{eq:delta1}
\end{eqnarray}
Since $\sqrt{L_1^2+dL_1} \leq L_1 + \sqrt{dL_1}$, we have $\delta_1\leq 8L_1/d + 4\sqrt{L_1/d} \leq 8L/d + 4\sqrt{L/d}$.

\textbf{Rewrite the second exponent in $z$.} Since $d\delta=2(z-\kappa)$,
\begin{eqnarray}
    d\sigma^2+bz &=& \frac{dK^4}{\lambda_0^2} + bd(1+\delta) + bz \nonumber \\
    &=& \frac{dK^4}{\lambda_0^2} + bd + bz + bd\delta \nonumber \\
    &=& \underbrace{\frac{dK^4}{\lambda_0^2} + b(d-2\kappa)}_{\beta_0} + 3bz \ . \nonumber 
\end{eqnarray}
Note $\beta_0 = d\rho - 2b\kappa$.

\textbf{Second term.} $T_2 \leq \epsilon/2$ iff $cz^2 \geq L_2(\beta_0+3bz)$. Since $\beta_0 \leq \beta_0^+ \triangleq \max(\beta_0, 0)$, it suffices that
\begin{equation*}
    cz^2 - 3bL_2z - L_2\beta_0^+ \geq 0 
\end{equation*}
a quadratic in $z$ with positive leading coefficient and nonnegative larger root. It holds for all $z$ at or above that root, and using $\sqrt{x+y} \leq \sqrt{x}+\sqrt{y}$ the root is at most,
\begin{equation*}
    z^\star \triangleq \frac{3bL_2}{c} + \sqrt{\frac{L_2\beta_0^+}{c}}\ .
\end{equation*}
Converting with $\delta=(2/d)(z-\kappa)$, it suffices that $\delta \geq \delta_2\triangleq (2/d)(z^\star - \kappa)$. 

On the other hand, $\beta_0^+ \triangleq \max(\beta_0,0)=\max(d\rho - 2b\kappa,0) \leq d\rho + 2b\kappa^-$ since $-\kappa \leq \kappa^-$. 
Using $\sqrt{x+y} \leq \sqrt{x} + \sqrt{y}$, we have
\begin{eqnarray}
    \sqrt{\frac{L_2\beta_0^+}{c}} &\leq& \sqrt{\frac{d\rho L_2}{c}} + \sqrt{\frac{2bL_2}{c}\kappa^-} \nonumber \\
    &\leq& \sqrt{\frac{d\rho L_2}{c}} +  \frac{bL_2}{c} + \frac{\kappa^-}{2} \quad \text{use}\quad \sqrt{xy} \leq (x+y)/2 \nonumber\ .    
\end{eqnarray}
Use this inequality and substitute $z^\star$ to the definition of $\delta_2$, we have:
\begin{eqnarray}
    \delta_2 &\leq& \frac{2}{d}\Big(\frac{4bL_2}{c}+\sqrt{\frac{d\rho L_2}{c}} + \frac{3}{2}\kappa^- \Big) \nonumber \\
    &=& \frac{8b}{c}\frac{L_2}{d} + 2\sqrt{\frac{\rho}{c}}\sqrt{\frac{L_2}{d}} + \frac{3\kappa^-}{d} \nonumber \\
    &\leq& \frac{8b}{c}\frac{L}{d} + 2\sqrt{\frac{\rho}{c}}\sqrt{\frac{L}{d}} + \frac{3\kappa^-}{d}   \quad\text{since}\quad L_2 \leq L\ . \label{eq:delta2}
\end{eqnarray}

Combine conditions on $\delta_1$ and $\delta_2$ from~\eqref{eq:delta1} and~\eqref{eq:delta2} respectively, the stated threshold is at least the bounds on $\delta_1$ and $\delta_2$, or $\delta \geq C_1L/d + C_2\sqrt{L/d}+3\kappa^-/d$ where $C_1 = \max(8,8b/c); C_2 = \max(4, 2\sqrt{\rho/c})$, and at least $2\kappa^-/d=\max(0,-2\kappa/d)$. With this choice then $T_1\leq\epsilon/2$ and $T_2\leq\epsilon/2$, so $\text{Pr}(E_k^t)\leq\epsilon$.

\section{Inference Algorithm of \textsc{Proteus}}
\label{sec:inference-pseudocode}
\begin{algorithm}[h]
   \caption{\textsc{Proteus} Inference (batch-wise)}
   \label{alg:testing_algs}
\begin{algorithmic}[1]
    \State {\bfseries INPUT:} pre-trained backbone $h(\boldsymbol{x}; \boldsymbol{\omega}_0)$; test batch $\mathcal{X}_B = \{\boldsymbol{x}_1, \ldots, \boldsymbol{x}_B\}$ drawn from a single unknown task; database $\mathrm{KB} = \{(\boldsymbol{\psi}_k, \Delta\boldsymbol{\omega}_k)\}_{k=1}^n$ with $\boldsymbol{\psi}_k = \{(\boldsymbol{m}_k^t, \boldsymbol{\Lambda}_k^t)\}_{t=1}^{\tau}$; random projection $\boldsymbol{W}$, statistics $\boldsymbol{G}$ and $\{\boldsymbol{e}_c\}$, ridge parameter $\gamma$ (outputs of Algorithm~\ref{alg:training_algs})
    \State {\bfseries OUTPUT:} class labels $c_1, \ldots, c_B$
    \State {\bfseries precompute} $(\boldsymbol{\Lambda}_k^t)^{-1}$, $\log|\boldsymbol{\Lambda}_k^t|$ for all $(k, t)$, and $(\boldsymbol{G} + \gamma\boldsymbol{I})^{-1}$ \Comment{once after training}
    \For{$q = 1$ {\bfseries to} $B$} \Comment{each input votes for a task}
        \For{$i = 1$ {\bfseries to} $n$}
            \State compute embedding $\boldsymbol{h}_i(\boldsymbol{x}_q) = h(\boldsymbol{x}_q; \Delta\boldsymbol{\omega}_i)$
            \State $s_i \leftarrow \min_{t \in [\tau]} \ell_i(\boldsymbol{x}_q; \boldsymbol{m}_i^t, \boldsymbol{\Lambda}_i^t)$ via Eq.~\ref{eq:score}
        \EndFor
        \State $\hat{k}_q \leftarrow \arg\min_{i \in [n]} s_i$ \Comment{Eq.~\ref{eq:per-input-retrieval}}
    \EndFor
    \State $\hat{k}_B \leftarrow \arg\max_{i \in [n]} \sum_{q=1}^{B} \mathbb{I}[\hat{k}_q = i]$, ties broken by the lowest task index \Comment{Eq.~\ref{eq:plurality-vote}}
    \State retrieve $\Delta\boldsymbol{\omega}_{\hat{k}_B}$ from $\mathrm{KB}$
    \For{$q = 1$ {\bfseries to} $B$}
        \State $\boldsymbol{z}_q \leftarrow \sigma\big(\boldsymbol{W}^\top h(\boldsymbol{x}_q; \Delta\boldsymbol{\omega}_{\hat{k}_B})\big)$
        \State $c_q \leftarrow \arg\max_{c} \boldsymbol{z}_q^\top (\boldsymbol{G} + \gamma\boldsymbol{I})^{-1} \boldsymbol{e}_c$ \Comment{Eq.~\ref{eq:23}}
    \EndFor
    \State {\bfseries return} $c_1, \ldots, c_B$
\end{algorithmic}
\end{algorithm}

\section{Complexity Analysis of \textsc{Proteus}}\label{sec: complexity}
{\bf Notation.} We use the notation of the main text: $m$ is the number of tasks in the stream and $n \leq m$ the number of tasks learned at test time; $|\mathcal{D}_k|$ is the number of training inputs of task $k$, $|\mathcal{D}| = \sum_{k=1}^m |\mathcal{D}_k|$, and $N = 10{,}000$ is the maximum number of embeddings sampled to fit each signature; $d = 768$ is the embedding dimension and $d_{\mathrm{rp}} = 10{,}000$ the dimension of the random projection; each signature has $\tau \leq \tau_{\max}$ components, where $\tau_{\max} \leq 100$ is the truncation level of the DP-GMM, fit with $I$ variational iterations; $N_c$ is the number of seen classes and $B$ the test batch size. Finally, $C_{\mathrm{fwd}}$ denotes the cost of one forward pass of the backbone with a LoRA adapter, and $C_{\mathrm{ft}}(\Delta\boldsymbol{\omega}_k)$ the cost of fine-tuning on task $k$.

\subsection{Complexity of \textsc{Proteus} Training}
The main loop of Algorithm~\ref{alg:training_algs} (lines 4--12) runs once per task. For task $k$:
\begin{enumerate}
    \item Training $\Delta\boldsymbol{\omega}_k$ by solving Eq.~\ref{eq:15_only} (line 5) costs $\mathcal{O}\big(C_{\mathrm{ft}}(\Delta\boldsymbol{\omega}_k)\big)$. Since the composed adapter in Eq.~\ref{eq:9} includes all $k-1$ prior LoRA units, this cost grows mildly with $k$, but remains small relative to the backbone.
    \item Fitting the signature $\boldsymbol{\psi}_k$ via Eq.~\ref{eq:18} (line 6) uses a DP-GMM with full covariances on $N_k = \min(|\mathcal{D}_k|, N)$ embeddings. Each of the $I$ iterations scores every embedding against every component in $\mathcal{O}(N_k\tau_{\max}d^2)$ and updates the covariances, their inverses, and their log-determinants in $\mathcal{O}(\tau_{\max}d^3)$, for a total of $\mathcal{O}(I N_k \tau_{\max} d^2)$ when $N_k \geq d$. The inverses and log-determinants of the final components are kept for inference.
    \item The inner loop (lines 8--11) runs $|\mathcal{D}_k|$ times. Each iteration computes $\boldsymbol{h}_k(\boldsymbol{x})$ with one forward pass, projects it in $\mathcal{O}(d\,d_{\mathrm{rp}})$, updates $\boldsymbol{G}$ in $\mathcal{O}(d_{\mathrm{rp}}^2)$, and updates $\boldsymbol{e}_y$ in $\mathcal{O}(d_{\mathrm{rp}})$. Since $d \leq d_{\mathrm{rp}}$, the loop costs $\mathcal{O}\big(|\mathcal{D}_k|(C_{\mathrm{fwd}} + d_{\mathrm{rp}}^2)\big)$.
\end{enumerate}
After the last task, the classifier weights $(\boldsymbol{G} + \gamma\boldsymbol{I})^{-1}[\boldsymbol{e}_1, \ldots, \boldsymbol{e}_{N_c}]$ are computed once in $\mathcal{O}(d_{\mathrm{rp}}^3 + d_{\mathrm{rp}}^2 N_c) = \mathcal{O}(d_{\mathrm{rp}}^3)$, since $N_c \leq d_{\mathrm{rp}}$. Summing over tasks, the training complexity of \textsc{Proteus} is
\begin{eqnarray}
\mathcal{O}\Big( \textstyle\sum_{k=1}^m C_{\mathrm{ft}}(\Delta\boldsymbol{\omega}_k) \ +\ |\mathcal{D}|\,\big(C_{\mathrm{fwd}} + d_{\mathrm{rp}}^2\big) \ +\ m\, I N \tau_{\max} d^2 \ +\ d_{\mathrm{rp}}^3 \Big) \ ,
\end{eqnarray}
i.e., LoRA fine-tuning for each task, one forward pass and one Gram update per training input, one signature fit per task, and a single inversion after training.

\subsection{Complexity of \textsc{Proteus} Inference}
For a test batch of $B$ inputs, Algorithm~\ref{alg:testing_algs} proceeds as follows:
\begin{enumerate}
    \item Per-input retrieval (lines 4--10) computes, for each input, its embedding under each of the $n$ stored adapters and scores it against their components. With the inverses and log-determinants precomputed (line 3), scoring one component costs $\mathcal{O}(d^2)$, so this step costs $\mathcal{O}\big(B\,n\,(C_{\mathrm{fwd}} + \tau_{\max} d^2)\big)$.
    \item The plurality vote (line 11) costs $\mathcal{O}(B + n)$, and retrieving $\Delta\boldsymbol{\omega}_{\hat{k}_B}$ from $\mathrm{KB}$ (line 12) costs $\mathcal{O}(1)$.
    \item Prediction (lines 13--16) applies the retrieved adapter to each input with one forward pass, projects the embedding in $\mathcal{O}(d\,d_{\mathrm{rp}})$, and scores all classes with the precomputed classifier weights in $\mathcal{O}(d_{\mathrm{rp}} N_c)$, for a total of $\mathcal{O}\big(B\,(C_{\mathrm{fwd}} + d_{\mathrm{rp}}(d + N_c))\big)$.
\end{enumerate}
Per test input, the inference complexity of \textsc{Proteus} is therefore
\begin{eqnarray}
\mathcal{O}\Big( n\,\big(C_{\mathrm{fwd}} + \tau_{\max} d^2\big) \ +\ d_{\mathrm{rp}}\,(d + N_c) \Big) \ .
\end{eqnarray}
In practice, the $n$ forward passes dominate, which explains why the per-batch latency reported below grows linearly with the number of tasks. For example, fixing $T=100$ (the largest $T$ in all settings), we report below the per-batch retrieval overhead (in seconds) with batch size = 32 on ImageNet-R as a function of the number of tasks $m$ in Table~\ref{tab:overhead_10T}. In the context of long sequences of tasks, we further run PROTEUS on CIFAR-100 with $m = 25$ tasks, 4 classes each task, and $T=20$ . The retrieval overhead (in seconds) each batch (batch size = 16) is a function of the no. of tasks $m$ in Table~\ref{tab:overhead_25T}. However, the forward passes under different adapters and the scores of different components are independent, so inference time can be reduced by parallelizing the computation across tasks and across components and will be our future improvement.

\begin{table}[h]
    \caption{{Inference retrieval latency report on ImageNet-R per-batch (in seconds) as a function of the number of tasks $m$.}}
    \centering
    \begin{tabular}{|c|c|c|c|c|c|c|c|c|c|}
        \toprule
        Task 1&Task 2&Task 3&Task 4&Task 5&Task 6&Task 7&Task 8&Task 9&Task 10  \\
        \midrule
        0.1595&0.2963&0.4417&0.5783&0.7252&0.9166&1.0273&1.1707&1.3157&1.4848 \\
        \bottomrule
    \end{tabular}
    \label{tab:overhead_10T}
\end{table}

\begin{table}[h]
    \caption{{Inference retrieval latency report on CIFAR-100 per-batch (in seconds) as a function of the number of tasks $m$.}}
    \centering
    \begin{tabular}{|c|c|c|c|c|c|c|c|c|c|}
        \toprule
        Ta. 3&Ta. 5&Ta. 7&Ta. 10&Ta. 13&Ta. 15&Ta. 17&Ta. 20&Ta. 23&Ta. 25  \\
        \midrule
        0.1648&0.2625&0.3831&0.5581&0.7382&0.8687&1.0073&1.2181&1.4552&1.6363 \\
        \bottomrule
    \end{tabular}
    \label{tab:overhead_25T}
\end{table}

\section{Additional Results on the forgetting metric}\label{sec:forgetting}

Table~\ref{tab: forgetting} reports the final average forgetting $\mathrm{F}_m$ of all methods (lower is better). \textsc{Proteus} attains the lowest average forgetting across the six benchmarks, $1.90\%$, followed by HiDe-Prompt at $2.57\%$, while all other baselines exceed $6\%$. \textsc{Proteus} forgets the least on ImageNet-A, VTAB5T, and VTAB-1k-Sim50, i.e., under abrupt and varying task shifts, and ranks second on CIFAR-100, ImageNet-R, and VTAB5T-1k, where it trails HiDe-Prompt by at most $1.02\%$. The two methods with the lowest forgetting are also the two most accurate retrievers (Table~\ref{matching_comparison}), consistent with the view that preserving task identification is key to limiting forgetting: HiDe-Prompt achieves this by retraining its task identifier on replayed class statistics, whereas \textsc{Proteus} keeps each signature frozen after it is fit. In contrast, methods whose shared retrievers are updated across tasks, such as L2P and CODA-Prompt, and retrieval-free methods that accumulate adaptation in a single model, such as SD-LoRA, forget the most, with average forgetting of $11.04\%$, $17.26\%$, and $20.21\%$, respectively.

\begin{table*}[h]
\begin{center}
\hspace{1mm}
\begin{small}
\caption{
\label{tab: forgetting}
Final average forgetting (\%, lower is better) of \textsc{Proteus} and other baselines across three scenarios of task shift. The last column reports the mean over the six benchmarks, with the resulting rank in parentheses. The best and second-best values are highlighted in \textbf{bold} and \underline{underline}, respectively.}
\vspace{3mm}
\setlength{\tabcolsep}{3pt}
\resizebox{1.0\linewidth}{!}{
\begin{tabular}{|lcc ccc c| c|}
    \toprule
    \textbf{Methods} & \multicolumn{2}{c}{\textbf{Uniformly Mild}} & \multicolumn{3}{c}{\textbf{Uniformly Abrupt}} & \textbf{Varying} & \textbf{Average}\\
    \cmidrule(l){2-3} \cmidrule(l){4-6} \cmidrule(l){7-7}
    & CIFAR-100 & ImageNet-R & ImageNet-A & VTAB5T & VTAB5T-1k & VTAB-1k-Sim50 & \textbf{Forgetting} \\
    \midrule
    L2P & $6.10_{\pm0.40}$ & $7.03_{\pm0.44}$ & $11.13_{\pm2.07}$ & $8.73_{\pm1.07}$ & $18.09_{\pm0.65}$ & $15.16_{\pm3.51}$ & $11.04$ (8) \\
    DualPrompt & $5.47_{\pm0.45}$ & $5.36_{\pm0.40}$ & $8.54_{\pm0.89}$ & $6.92_{\pm0.45}$ & $6.88_{\pm0.79}$ & $14.62_{\pm3.65}$ & $7.96$ (7) \\
    CODA-Prompt & $10.12_{\pm0.58}$ & $21.56_{\pm0.37}$ & $19.43_{\pm0.92}$ & $11.89_{\pm1.21}$ & $9.70_{\pm0.85}$ & $30.86_{\pm0.96}$ & $17.26$ (9) \\
    S-Prompt & $5.64_{\pm0.18}$ & $4.47_{\pm0.14}$ & $7.81_{\pm1.27}$ & $6.04_{\pm0.40}$ & $4.68_{\pm0.17}$ & $7.38_{\pm1.48}$ & $6.00$ (3) \\
    AdaPromptCL & $5.82_{\pm0.30}$ & $3.85_{\pm0.12}$ & \underline{$4.38_{\pm1.21}$} & $3.91_{\pm0.64}$ & $3.40_{\pm0.66}$ & $17.59_{\pm10.38}$ & $6.49$ (5) \\
    HiDe-Prompt & $\mathbf{0.85_{\pm0.03}}$ & $\mathbf{1.00_{\pm0.32}}$ & $4.79_{\pm2.19}$ & $3.94_{\pm0.42}$ & $\mathbf{0.26_{\pm0.13}}$ & \underline{$4.57_{\pm0.76}$} & \underline{$2.57$} (2) \\
    InfLoRA & $9.02_{\pm0.73}$ & $9.41_{\pm0.15}$ & $7.20_{\pm0.64}$ & $4.65_{\pm1.50}$ & $3.43_{\pm0.23}$ & $9.37_{\pm2.76}$ & $7.18$ (6) \\
    SD-LoRA & $12.44_{\pm0.78}$ & $27.91_{\pm1.57}$ & $5.49_{\pm1.56}$ & $11.98_{\pm2.76}$ & $32.12_{\pm9.28}$ & $31.33_{\pm1.11}$ & $20.21$ (10) \\
    RanPAC & $4.82_{\pm0.10}$ & $4.62_{\pm0.24}$ & $8.15_{\pm0.26}$ & \underline{$2.52_{\pm0.61}$} & $7.23_{\pm4.50}$ & $10.69_{\pm1.71}$ & $6.34$ (4) \\
    \midrule
    \textsc{Proteus} & \underline{$1.87_{\pm0.14}$} & \underline{$1.05_{\pm0.35}$} & $\mathbf{3.57_{\pm1.59}}$ & $\mathbf{1.20_{\pm0.30}}$ & \underline{$0.65_{\pm0.44}$} & $\mathbf{3.07_{\pm1.52}}$ & $\mathbf{1.90}$ (1) \\
    \bottomrule
    \end{tabular}
    }
\end{small}
\end{center}
\end{table*}

\section{Datasets, Hyperparameters, and Pre-Trained Models} \label{sec: data and hyper}
\textbf{Forgetting metric.}~Forgetting is defined as $\mathrm{F}_\kappa=(1/(\kappa-1))\sum_{\tau=1}^{\kappa-1}\max_{\tau'\in\{ 1,...,\kappa-1\}}(\mathrm{acc}(\tau,\tau')-\mathrm{acc}(\tau, \kappa))$ where $\mathrm{acc}(\tau,\kappa)$ denotes model's prediction accuracy on task $\tau$ after learning task $\kappa$. 

\textbf{Datasets.}~Table~\ref{tab:dataset_stats} summarizes the datasets used in our experiments. CIFAR-100 is accessed directly through torchvision, while ImageNet-R, ImageNet-A, and VTAB5T are taken from \url{https://github.com/zhoudw-zdw/RevisitingCIL}~\cite{zhou2023revisiting}. VTAB5T keeps 10 classes from each of Resisc45, DTD, Oxford-IIIT Pets, EuroSAT, and Flowers102, with one dataset per task. VTAB5T-1k follows the VTAB-1k protocol~\cite{zhai2019visual, dong2023efficient}: its five tasks are SVHN, Resisc45, Clevr-Count, KITTI-Distance, and sNORB-Elevation, in this order, each with the 1,000 training and validation images of VTAB-1k (800 and 200) and its full test set. VTAB-1k-Sim50 follows the construction of~\citet{DBLP:conf/icml/KimYPLSBL24}. We pair two datasets from each of VTAB's three groups, namely DTD and CIFAR-100 (natural), Resisc45 and PatchCamelyon (specialized), and Clevr-Count and sNORB-Elevation (structured). For each pair, we take $m$ classes from each dataset, where $m$ is the number of classes in the smaller one, and form two tasks, each drawing an equal number of classes from its own dataset and from its partner. For example, the DTD task contains 23 DTD and 23 CIFAR-100 classes. The two tasks of a pair thus share their source datasets while their class sets remain disjoint. The six tasks arrive in the order Resisc45, sNORB-Elevation, DTD, PatchCamelyon, Clevr-Count, and CIFAR-100, where each task is named after its own dataset, and classes are split with a fixed random seed.

\begin{table*}[h]
    \setlength{\tabcolsep}{12pt}
    \centering
    \caption{Dataset Statistics}
    \begin{tabular}{|llll|}
        \toprule
        \textbf{CL Datasets} & \textbf{\# train samples} & \textbf{\# val samples} & \textbf{\# classes} \\
        \midrule
        CIFAR100 & 50000 & 10000 & 100 \\
        Imagenet-R & 24000 & 6000 & 200 \\
        Imagenet-A & 5981 & 1519 & 200 \\
        VTAB5T & 1796 & 8619 & 50 \\
        VTAB-1k-5T & 5000 & 60193 & 76 \\
        VTAB-1k-Sim50 & 4363 & 65501 & 114 \\
        \bottomrule
    \end{tabular}
    \label{tab:dataset_stats}
\end{table*}

{\bf Pre-Trained Backbone and Configurations.}For all experiments, we use the ViT-B/16 model~\cite{dosovitskiy2020image} pre-trained on ImageNet-21k; all baselines use the same checkpoint. The rank of the LoRA units added for each new task in \textsc{Proteus} is set to $r=4$ for CIFAR-100, and ImageNet-R, $r=12$ for VTAB5T-small, VTAB5T-1k and VTAB-1k-Sim50, and $r=16$ for ImageNet-A. The elastic loss regularization strength $\lambda$ in Eq.~\ref{eq:15_only} is initialized at $0.006$ and is reduced by $20\%$ after each task, since the LoRA pool grows as tasks arrive.

The weight $\alpha$ in Eq.~\ref{eq:15_only} is ablated in App.~\ref{sec: alpha-lambda}. The noise level $\gamma$ in Eq.~\ref{eq:23} is selected as in RanPAC, by searching $\{10^{-8},\dots,10^{8}\}$ on held-out training data. All models are trained with Adam and a constant learning rate of $10^{-3}$ for CIFAR-100, ImageNet-R, VTAB5T, VTAB-1k-5T and VTAB-1k-Sim50, and $5\times10^{-4}$ for ImageNet-A. The test batch size is $8$ for CIFAR-100, VTAB5T, and VTAB-1k-5T, $32$ for ImageNet-R, $24$ for ImageNet-A and VTAB-1k-Sim50. Each task is trained for $5$ epochs on CIFAR-100, $20$ epochs on ImageNet-R and the VTAB benchmarks, and $30$ epochs on ImageNet-A. All experiments are run on a single NVIDIA A40 GPU with $48$GB of memory.

\section{Additional ablation studies}\label{sec:orthogonal}

\begin{figure}[h]
    \centering
    \includegraphics[width=\linewidth]{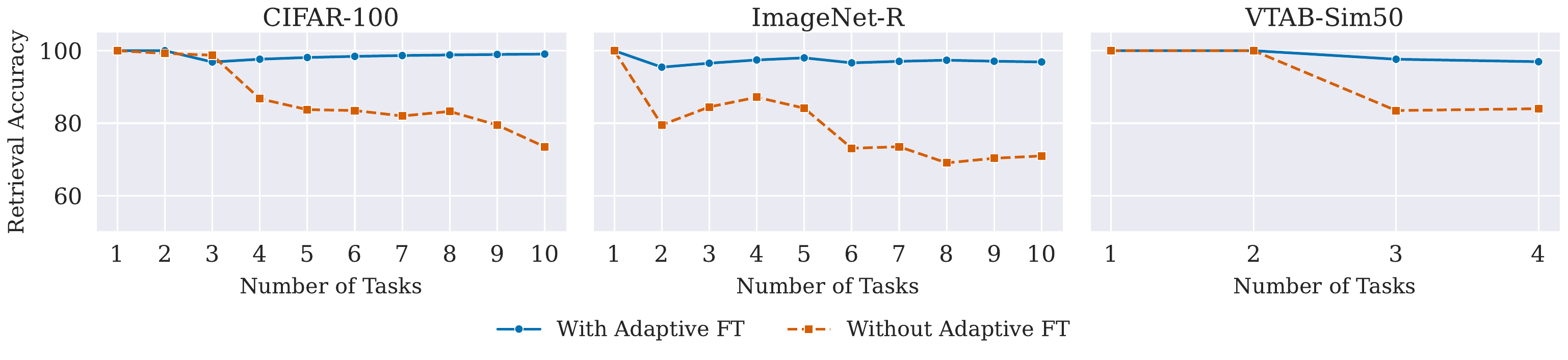}
    \caption{Task retrieval accuracies achieved by \textsc{Proteus} and its variant w/o adaptive fine-tuning and enforcing orthogonality among fine-tuning (FT) directions. It can be observed that w/o such conditioning, the task retrieval severely decreases.}
    \label{fig:synergy-pooling-matching}
\end{figure}

{\bf Effect of test batch size.} Figure~\ref{fig:batch} varies the test batch size $B$ and reports retrieval accuracy (top row) and final accuracy (bottom row), which lets us compare final accuracy at matched retrieval accuracy. Larger batches give the plurality vote more evidence, so the retrieval accuracy of \textsc{Proteus} and HiDe-Prompt generally rises with $B$ on all four benchmarks, and final accuracy rises with it. From $B = 4$ to $B = 128$, \textsc{Proteus} gains 8.75\%, 16.23\%, and 2.51\% on ImageNet-R, ImageNet-A, and CIFAR-100, compared with 5.19\%, 13.19\%, and 1.16\% for HiDe-Prompt, while all methods change by at most 0.37\% on VTAB5T. At $B = 1$, i.e., per-input retrieval, the retrieval accuracy of \textsc{Proteus} already matches or exceeds that of HiDe-Prompt on all four benchmarks, with 66.49\% vs.\ 64.67\% on ImageNet-R, 43.32\% vs.\ 37.07\% on ImageNet-A, 96.52\% vs.\ 93.71\% on VTAB5T, and 76.14\% vs.\ 76.27\% on CIFAR-100. More importantly, the advantage of \textsc{Proteus} persists when retrieval is matched. At $B = 128$, both methods retrieve the correct task for every batch on ImageNet-R, VTAB5T, and CIFAR-100, yet \textsc{Proteus} still exceeds HiDe-Prompt by 4.06\%, 4.02\%, and 1.62\%, respectively. On ImageNet-A, HiDe-Prompt even retrieves more accurately at this batch size (98.10\% vs.\ 89.47\%), yet \textsc{Proteus} still leads by 14.00\%, suggesting that most of its advantage there also arises after retrieval. Likewise, DualPrompt and S-Prompt reach 100\% retrieval accuracy on CIFAR-100 at $B = 128$, yet their final accuracy improves by at most 0.88\% over $B = 4$. Batching also helps only when per-input retrieval already favors the correct task: on ImageNet-A, DualPrompt and S-Prompt retrieve only 16.34\% and 15.80\% of inputs correctly per sample, barely above the 10\% chance level, and remain below 26\% at every batch size, consistent with Corollary~\ref{corollary}. Together, these results show that retrieval accuracy alone does not determine final accuracy, and suggest that once retrieval is accurate, the specialization of the retrieved adapters is what sets \textsc{Proteus} apart from HiDe-Prompt.

\begin{figure}[h]
    \centering
    \begin{subfigure}[b]{0.24\linewidth}
        \centering
        \includegraphics[width=\linewidth]{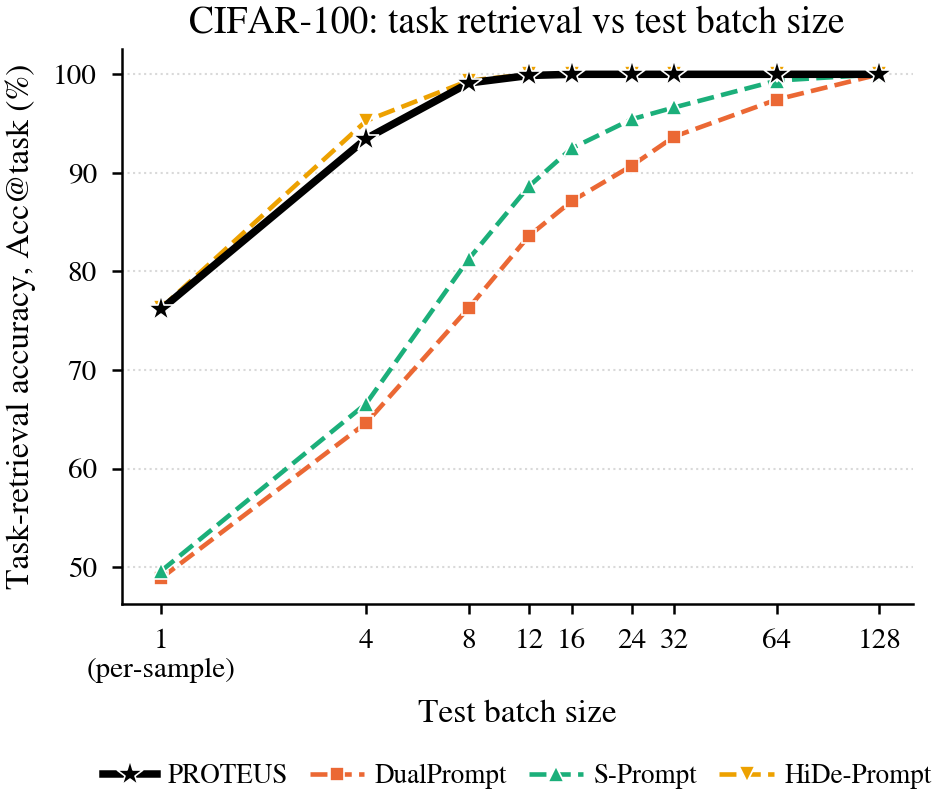}
        \caption{CIFAR-100}
        \label{fig:bs-ret-cifar}
    \end{subfigure}\hfill
    \begin{subfigure}[b]{0.24\linewidth}
        \centering
        \includegraphics[width=\linewidth]{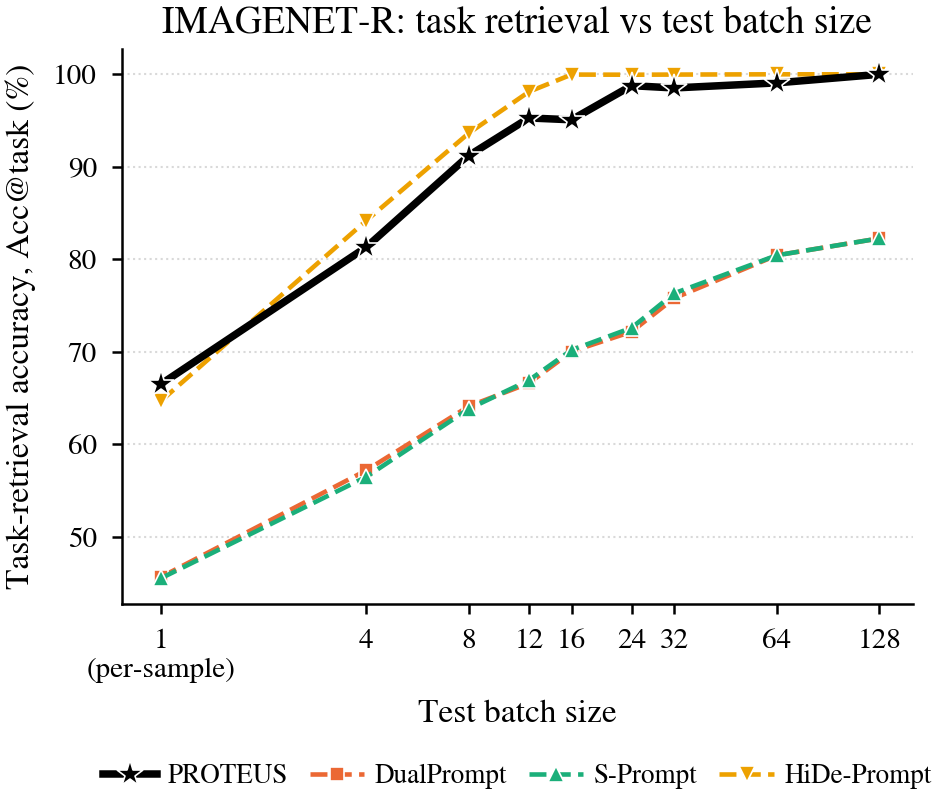}
        \caption{ImageNet-R}
        \label{fig:bs-ret-imr}
    \end{subfigure}\hfill
    \begin{subfigure}[b]{0.24\linewidth}
        \centering
        \includegraphics[width=\linewidth]{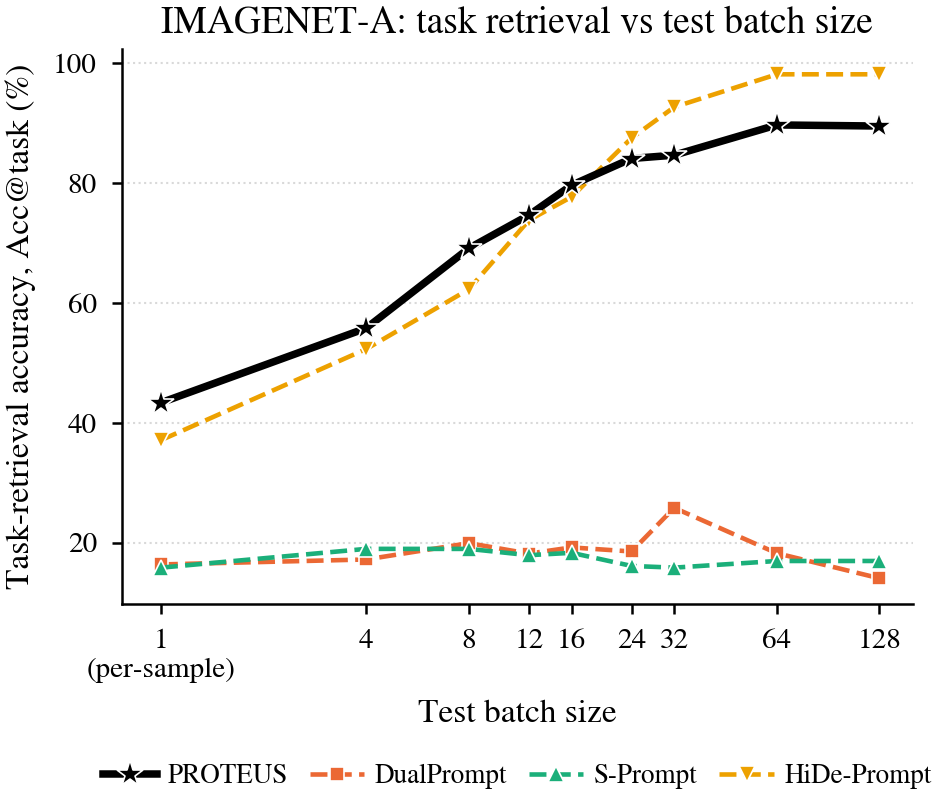}
        \caption{ImageNet-A}
        \label{fig:bs-ret-ima}
    \end{subfigure}\hfill
    \begin{subfigure}[b]{0.24\linewidth}
        \centering
        \includegraphics[width=\linewidth]{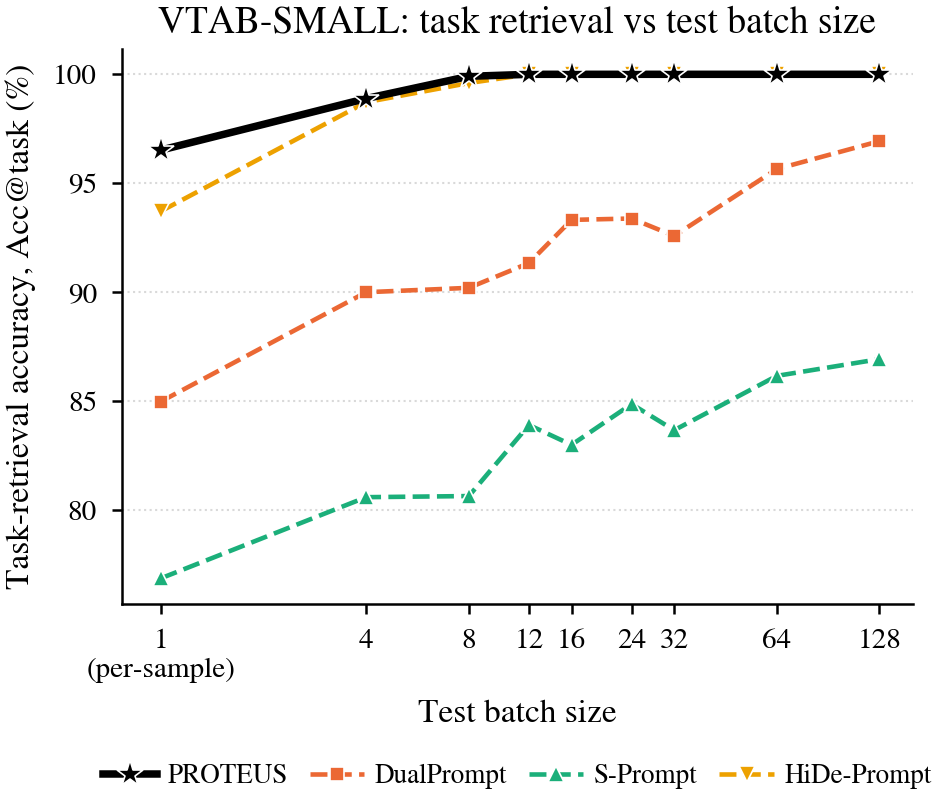}
        \caption{VTAB5T}
        \label{fig:bs-ret-vtab}
    \end{subfigure}

    \vspace{2mm}

    \begin{subfigure}[b]{0.24\linewidth}
        \centering
        \includegraphics[width=\linewidth]{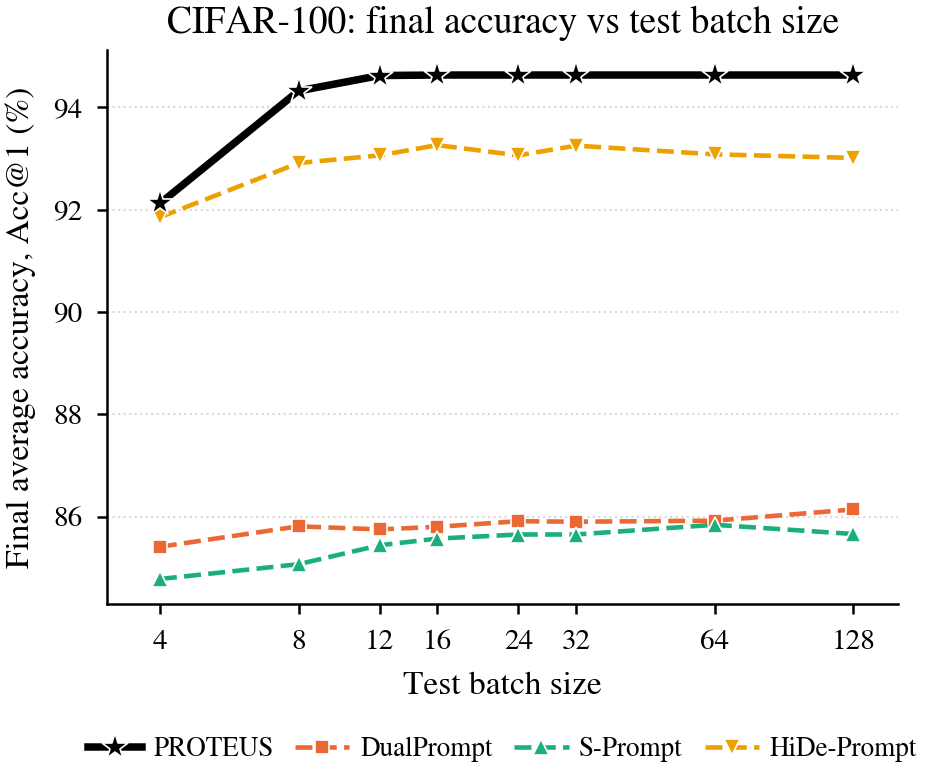}
        \caption{CIFAR-100}
        \label{fig:bs-acc-cifar}
    \end{subfigure}\hfill
    \begin{subfigure}[b]{0.24\linewidth}
        \centering
        \includegraphics[width=\linewidth]{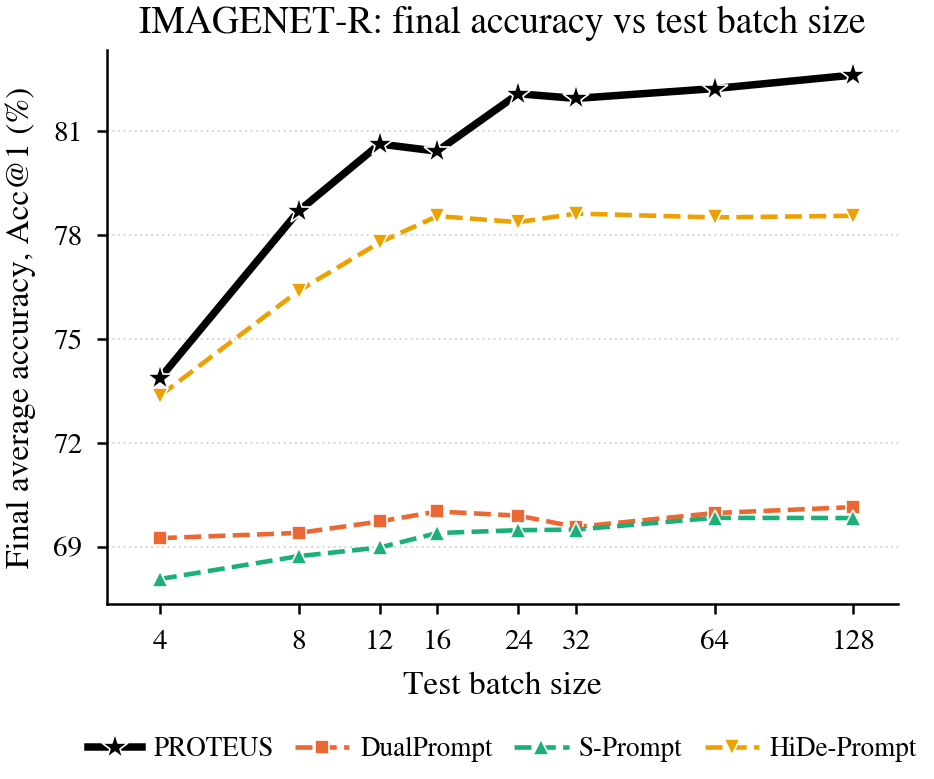}
        \caption{ImageNet-R}
        \label{fig:bs-acc-imr}
    \end{subfigure}\hfill
    \begin{subfigure}[b]{0.24\linewidth}
        \centering
        \includegraphics[width=\linewidth]{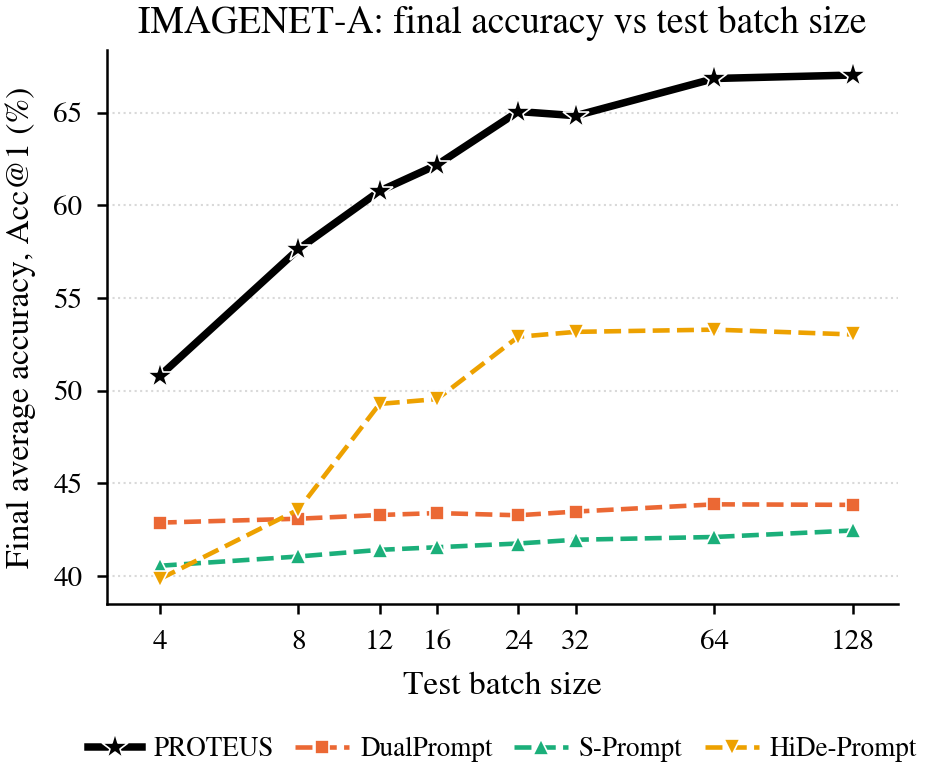}
        \caption{ImageNet-A}
        \label{fig:bs-acc-ima}
    \end{subfigure}\hfill
    \begin{subfigure}[b]{0.24\linewidth}
        \centering
        \includegraphics[width=\linewidth]{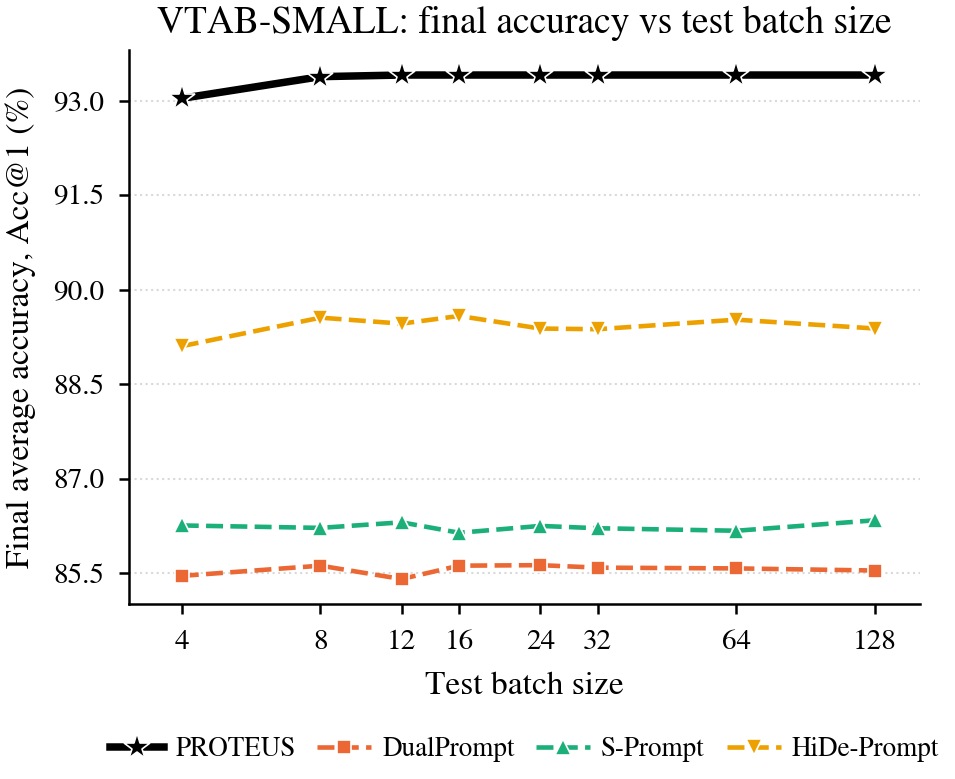}
        \caption{VTAB5T}
        \label{fig:bs-acc-vtab}
    \end{subfigure}

    \caption{Effect of test batch size on single seed. Top row (a--d): task-retrieval accuracy, the percentage of single-task test batches assigned to their true task by plurality vote (batch size 1 is per-sample retrieval). Bottom row (e--h): final average accuracy $\mathrm{ACC}_m$. All methods use the same batching protocol.}
    \label{fig:batch}
\end{figure}

{\bf Additional Results on the Orthogonality between Old and New LoRA Units.~}This section provides additional empirical verification that the orthogonality constraint in our adaptive fine-tuning loss in Eq.~\ref{eq:15_only} is indeed satisfied in the numerically optimized solution. In particular, we visualize the cosine distance between the corresponding LoRA updates across all tasks in VTAB5T-Large (Fig.~\ref{fig:sim_vtab_large_qv}) and ImageNet-A (Fig.~\ref{fig: similarity_INA_QV_Lora}). In both experiments, we observe that the cosine distances between the corresponding tuning directions $\Delta \boldsymbol{\omega}_\tau^\perp$ of different tasks are vanishingly small, confirming their orthogonality. This observation also remains consistent across various experiments where the adaptive LoRA fine-tuning is applied at different attention layers. The enforced orthogonality in turn helps improve the task retrieval as previously demonstrated in Section~\ref{sec:method}.

\begin{table}[!t]
    \centering
    \makebox[\textwidth][c]{
        \begin{minipage}{0.44\textwidth}
        \setlength{\tabcolsep}{5pt}
            \caption{\textsc{Proteus}'s classification and retrieval accuracy achieved with various settings of $\boldsymbol{S}$ and $\lambda$ on ImageNet-R dataset.}
            \begin{tabular}{|lcc|}
                \toprule
                \textbf{Configs} & \textbf{Classification} & \textbf{Retrieval} \\
                  \midrule
                  $\boldsymbol{S}=\mathbf{0}$ & $78.98$ & $87.05$ \\
                  $\boldsymbol{S}=\boldsymbol{I}$ & $30.38$ & $41.45$ \\
                  $\boldsymbol{S}_\ast, \lambda=0$ & $30.49$ & $41.97$\\
                  $\boldsymbol{S}_\ast, \lambda>0$ & $82.17$ & $98.26$  \\
                  \bottomrule
            \end{tabular}
            
            \label{tab:ablationS}
        \end{minipage}
        \hspace{10mm}
        \begin{minipage}{0.45\textwidth}
        \setlength{\tabcolsep}{5pt}
            \caption{\textsc{Proteus}'s classification accuracy achieved with Top-K Gaussians retrieval where $K=1, 3, 5$ on two VTAB datasets.} 
            \begin{tabular}{|lcc|}
                \toprule
                $\textbf{K}$ & \textbf{VTAB5T-Small} & \textbf{VTAB5T-Large} \\
                \midrule
                $1$ & \textbf{94.76} & \textbf{89.32} \\ [1.95mm]
                 $3$ & $93.72$ & $76.50$ \\ [1.95mm] 
                 $5$ & $92.34$ & $76.25$ \\ 
                \bottomrule
            \end{tabular}
            \label{tab:topkG}
        \end{minipage}
    }
    \vspace{-3mm}
\end{table}

{\bf Comparing to Top-K Gaussian signature function.}
Since selecting the closest Gaussian component yields strong performance, we perform an ablation study on the use of Top-K Gaussians instead. Given a new input $x$, we compute the likelihood of $h(x)$ under all Gaussian components for each task and aggregate the Top-K likelihoods by summing them, producing a task-specific signal. The task with the largest signal is then inferred as the task identity. We show the results for $K=1,3,5$ on two VTAB5T datasets in Table~\ref{tab:topkG}, which confirms that the accuracy at $K = 1$ is the best.

`
\begin{figure}[ht]
    \centering
    \includegraphics[width=0.8\linewidth]{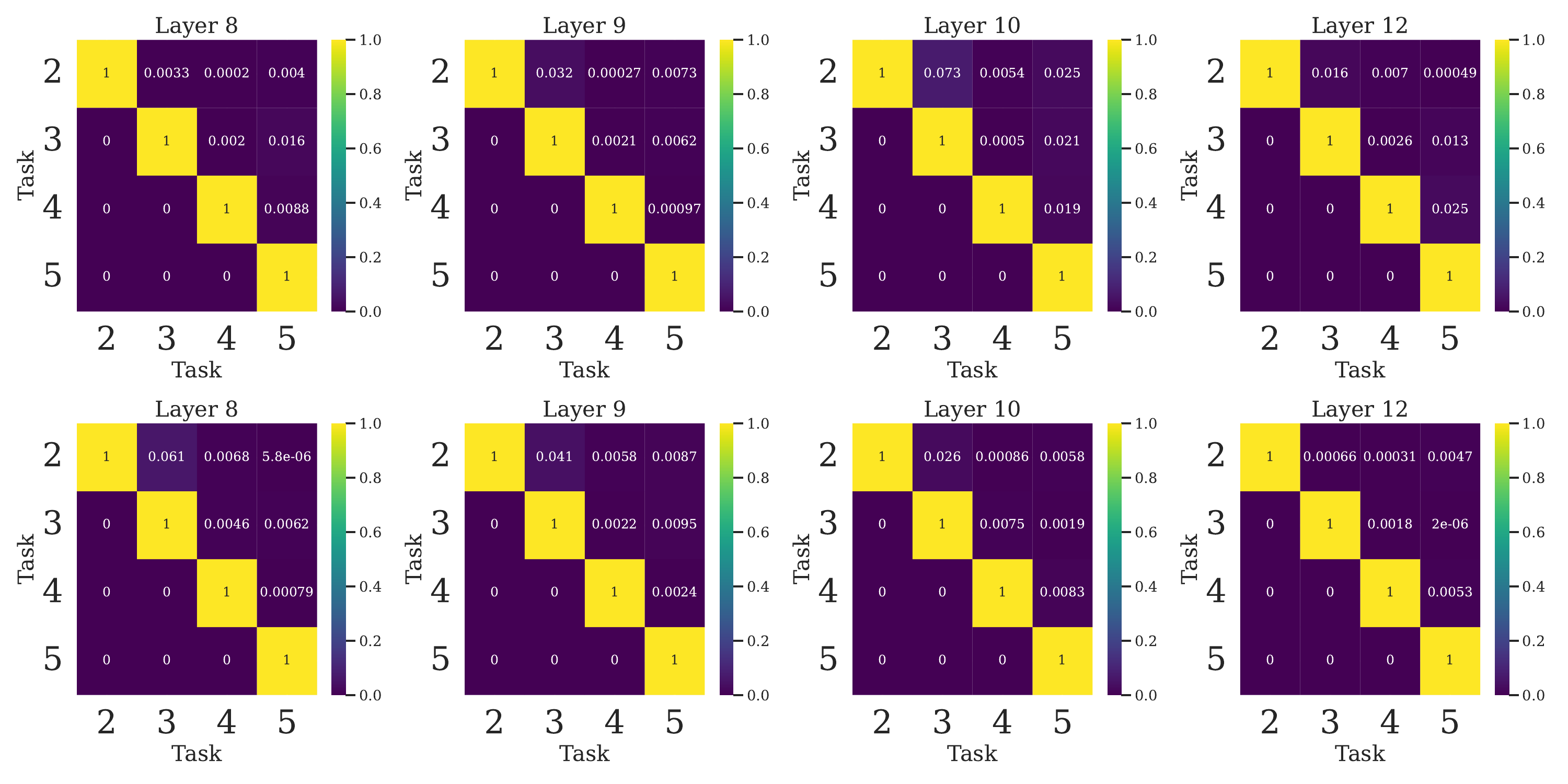}
    \caption{Pairwise cosine similarities of LoRA fine-tuning directions $\Delta \boldsymbol{\omega}_\tau^\perp$ across tasks $\tau$ of VTAB5T-Large for the query (top) and value (bottom) matrices of the multi-head self-attention (MHSA) at various layers. For clearer visualization, we display only the upper half of the cosine similarity matrices since they are symmetric.}
    \label{fig:sim_vtab_large_qv}
\end{figure}

\begin{figure*}
    \centering
    \includegraphics[width=\linewidth]{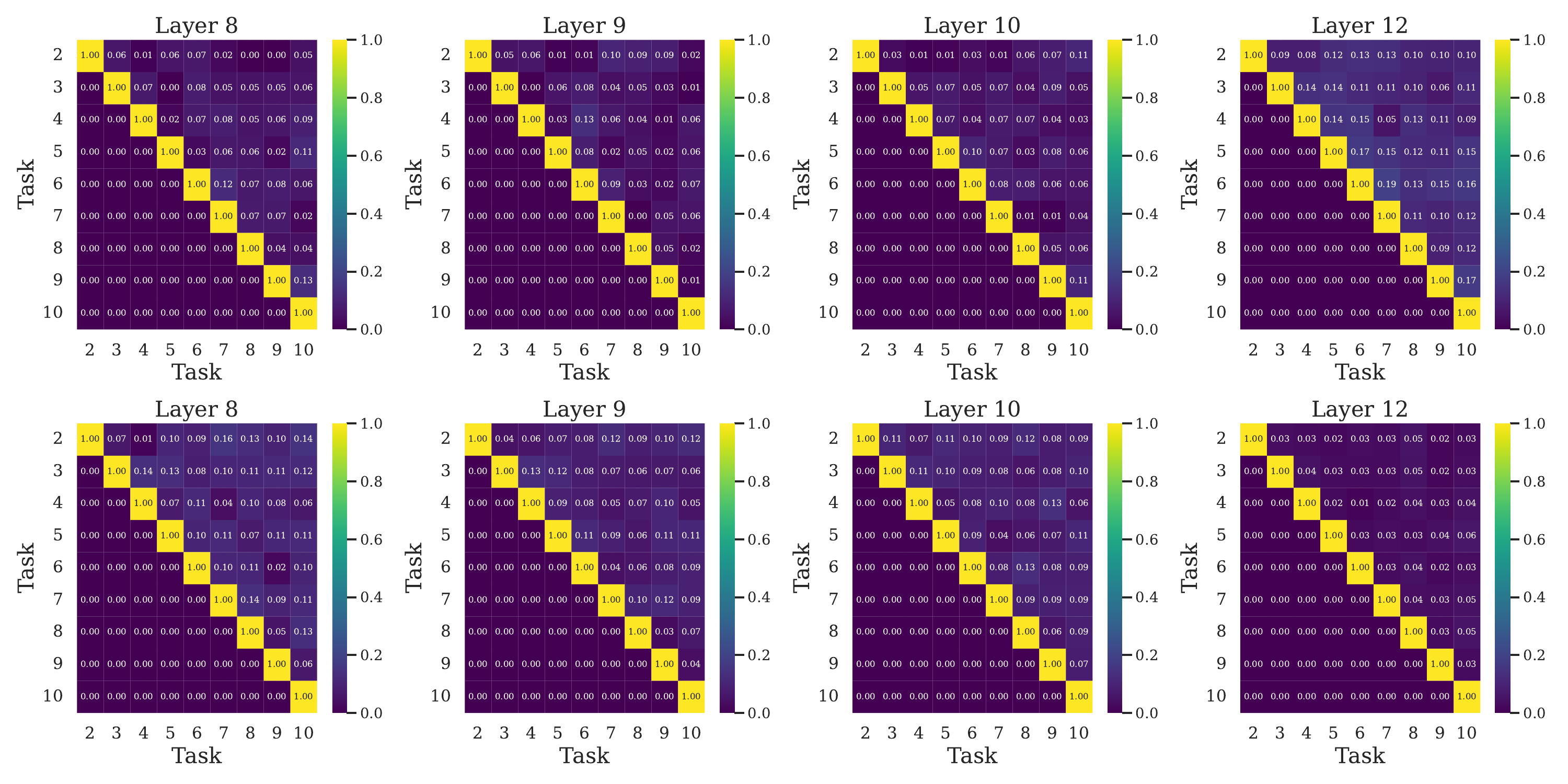}
    \caption{Pairwise cosine similarities of LoRA fine-tuning directions $\Delta \boldsymbol{\omega}_\tau^\perp$ across tasks $\tau$ of ImageNet-A for the query (top) and value (bottom) matrices of the multi-head self-attention (MHSA) at various layers.  For clearer visualization, we display only the upper half of the cosine similarity matrices since they are symmetric.}
    \label{fig: similarity_INA_QV_Lora}
\end{figure*}

{\bf {Advantage of DP-GMM over K-Means.}}
{Conceptually, DP-GMM offers two key advantages over K-means clustering. First, it automatically adapts the number of clusters, which is crucial in continual learning where tasks vary widely in complexity and semantics. Second, DP-GMM’s probabilistic modeling allows \textsc{Proteus} to compute principled likelihood scores for test inputs, and leverage those for retrieval. To empirically demonstrate the importance of the DP-GMM module, we provide  an additional experiment  in which DP-GMM is replaced  with K-Means (using $K=20$) in our pipeline. This simple change resulted in up to \textbf{9.95\%/23.80\%} drops in classification/retrieval accuracy. (see the Table~\ref{table:kmeans} below) }

\begin{table}[htbp]

\centering
\caption{{Performance Comparison of DP-GMM and K-Means on ImageNet-R and ImageNet-A}}
\label{table:kmeans}

\begin{tblr}{
  width     = \linewidth,
  colspec   = {l *{4}{c}},
  column{1} = {leftsep=2pt},
  cell{1}{1} = {r=2}{c},
  cell{1}{2} = {c=2}{c},
  cell{1}{4} = {c=2}{c},
  vlines,
  hlines,
  rows      = {m},
}
  Datasets           & DP-GMM          &                 & K-Means         &                 \\
                     & Classifier      & Retrieval       & Classifier      & Retrieval       \\
  ImageNet-R   & \textbf{82.71}  & \textbf{98.26}  & 81.15           & 92.75           \\
  ImageNet-A   & \textbf{65.29}  & \textbf{87.43}  & 55.34           & 63.63           \\
\end{tblr}

\end{table}

\section{GPU Footprint During Training}\label{sec:gpu_footprint}
Fig.~\ref{fig:complexity} compares \textsc{Proteus}'s GPU memory overhead during training to those of other baseline methods on ImageNet-A. Since PEFT-based CL methods expand their memory to accomodate increasing tasks, we report the maximum memory incurred during training of each method. The results show that \textsc{Proteus} requires less GPU memory than the competing approaches while delivering superior performance, as highlighted in Table \ref{tab: main-table}. Importantly, \textsc{Proteus} achieves higher accuracy than HiDe-Prompt while maintaining lower memory consumption. We note that the DP-GMM signature learning stage is executed on CPU; therefore, the reported GPU memory footprint reflects only the adaptive LoRA fine-tuning process.

\begin{figure}[t]
     \centering
    \includegraphics[width=0.5\linewidth]{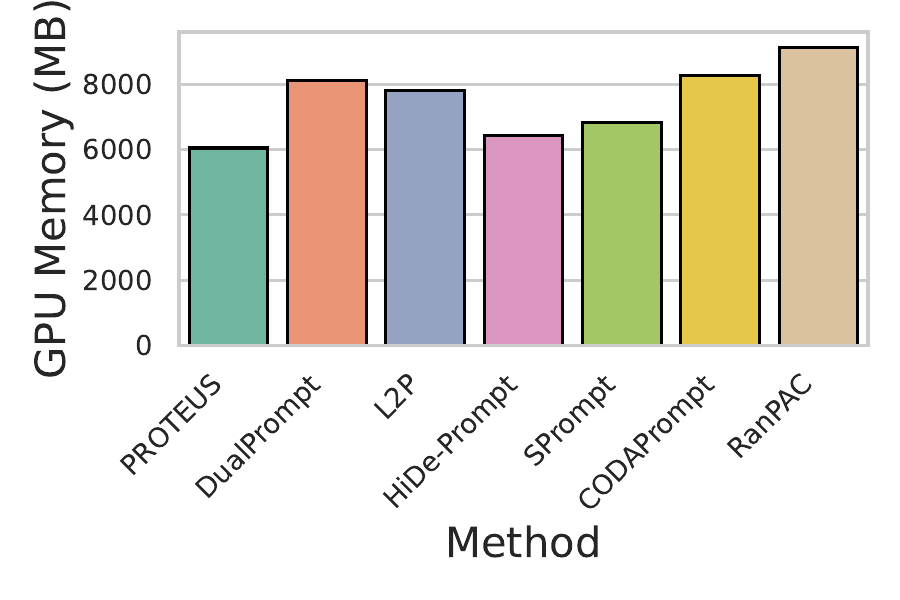}\vspace{-5mm}
    \caption{GPU consumption (MB) of various baselines.~\textsc{Proteus} consumes least GPU.\vspace{-5mm}}
    \label{fig:complexity}
\end{figure}

{\section{Detailed Analysis on Memory Cost.}}\label{sec:mem_cost}
{To further provide analysis about total memory cost induced during training process, we first provide an example on common settings. Choosing a low LoRA rank $r = 4$ per task, i.e., 4 rank-1 LoRA units. As the embedding dimension of each transformer block in ViT is 768 and there are 12 such blocks, each rank-1 LoRA unit comprises 24 vectors of size 768 x 1. In addition, the sizes of the mean vector and covariance matrix of each resulting input embedding cluster under each rank-1 LoRA unit are 768 x 1 and 768 x 768 respectively. As our non-parametric clustering algorithm is configured with at most 20 such clusters, each task incurs at most 24 * 4 * 768 * 4 (bytes) = 288KB and $(768 + 768 * 768) * 4 * 20 < 46$MB to store 4 rank-1 LoRA units and the (mean, covariance) signature for each resulting input embedding cluster, respectively. This amounts to 48MB of overhead per task which is quite affordable given the capacity of modern computers. Hence, a large sequence of 100 tasks will incur no more than 4.8GB of overhead. }

{Alternatively, we can also enhance the current approach via using a schedule of exponentially decaying LoRA rank. Here, we assume that with an increasing number of observed tasks, the incremental information provided by each newly arriving task diminishes. Hence, the LoRA rank for latter tasks can be reduced accordingly. This ensures the total parameter cost remains strictly bounded. In particular, let $r_m$ denote the LoRA rank at task $m$, we choose an exponential decay strategy such as: $r_m = r_o\exp{(-\alpha(m-1)}) $ so the sequence $r_1, …, r_m$ forms a geometric series or $\sum_{m=1}^\infty r_m = \frac{r_0}{1 - \exp{-\alpha}} < \infty$. Thus, even as $M \rightarrow \infty$, the total LoRA memory remains finite.}

{Let $\tau$ be the number of clusters of each task , we introduce $\tau*M$ clusters after observing M tasks.
Let c be the memory cost per rank-1 LoRA unit, and v be the memory cost per cluster. Following the above calculations:
\begin{itemize}
    \item c = 24 * 768 * 4 * 4 (bytes) = 288KB = 0.288MB
    \item v = $(768 + 768 * 768 ) * 4 (bytes) < 2.4$ MB
\end{itemize}
The total memory cost is then upper-bounded via:
$\sum_{m=1}^\infty c*r_m + v*\tau*M \leq H$.}

{Choosing $r_0 = 10, \alpha = 0.05, \tau = 20, M = 20$, we have H = 1019.05MB $<$ 1GB. Using this budget-constrained configuration on a 20-task simulation of the CIFAR100 dataset, we obtained promising results as reported below. Despite the tight budget, the performance of the budget-constrained \textsc{Proteus} matches that of its unconstrained version (fixed r = 10) as Table~\ref{tab:budget_cifar100}:}

\begin{table}[htbp]
    
    \caption{{Performance comparison on CIFAR100 20-task setting between \textsc{Proteus} under budget-constrained strategy and original unconstrained version.}}
    \centering
    \setlength{\tabcolsep}{3pt}
    \begin{tabular}{|c|c|c|c|c|c|c|c|c|c|c|}
        \toprule
        Methods&Ta. 2&Ta. 4&Ta. 6&Ta. 8&Ta. 10&Ta. 12&Ta. 14&Ta. 16&Ta. 18&Ta. 20  \\
        \midrule
       Budget-Constrained&99.1&98.6&97.9&97.57&96.82&96.46&96.22&96.1&96.21&95.43 \\
        \midrule
        Unconstrained& 99.1&98.4&97.7& 97.17&96.74&96.50&96.54&96.51&96.60&95.65 \\
        \bottomrule
    \end{tabular}
    \label{tab:budget_cifar100}
\end{table}

{In a more challenging setting on ImageNet-R with $r_0 = 10, \alpha = 0.1, \tau = 50$ which enforces a maximum memory budget H $<$ 2.5GB, the budget-constrained \textsc{Proteus} still manages to retain strong performance compared to its unconstrained version (fixed r = 10) as reported in Table~\ref{tab:budget_imgnr} below.}
\begin{table}[htbp]
    \caption{{Performance comparison on ImageNet-R 20-task setting between \textsc{Proteus} under budget-constrained strategy and original unconstrained version.}}
    \centering
    \setlength{\tabcolsep}{3pt}
    \begin{tabular}{|c|c|c|c|c|c|c|c|c|c|c|}
        \toprule
        Methods&Ta. 2&Ta. 4&Ta. 6&Ta. 8&Ta. 10&Ta. 12&Ta. 14&Ta. 16&Ta. 18&Ta. 20  \\
        \midrule
       Budget-Constrained&87.55&86.27&82.43&82.57&82.85&82.86&81.31&79.53&80.04&79.13 \\
        \midrule
        Unconstrained& 90.28&90.03&89.36&88.39&87.53&86.92&87.20&86.46&86.74&86.65 \\
        \bottomrule
    \end{tabular}
    \label{tab:budget_imgnr}
\end{table}

{These results consistently demonstrate that our exponential-decay LoRA strategy can rigorously bound memory growth while preserving the optimal performance of the unconstrained version (with mild degradation). This confirms that \textsc{Proteus} remains practical, affordable (in memory cost), and scalable in long-term continual learning scenarios.}

\section{Related Works}
\label{related}

{\bf Continual Learning.}~CL aims to learn new tasks effectively without sacrificing the knowledge learned in previous tasks.~Conventional CL methods fall into 3 main categories: rehearsal-based, regularization-based and architecture-based. Rehearsal-based CL methods ~\cite{rolnick2019experience, shin2017continual, bang2021rainbow, buzzega2020dark, aljundi2019gradient, aljundi2019online, hayes2019memory, hou2019learning, lopez2017gradient, von2019continual, van2020brain} utilize memory buffers to store past examples of previous tasks in order to reduce forgetting, while regularization based methods ~\cite{titsias2019functional, pan2020continual, chaudhry2018riemannian, kirkpatrick2017overcoming} add regularization terms to constrain the weight updates such that they do not interfere significantly with the knowledge learned from previous tasks. 

On the other hand, architecture-based methods ~\cite{mallya2018piggyback, ebrahimi2020adversarial, lee2020neural} introduce new parameters for each new task, allowing the model to specialize for new tasks without disrupting the learned representations of previous ones. However, traditional methods for CL are not suitable for adapting large pretrained models, which is the focus of this paper, due to their need to fully fine-tune the pre-trained models.

{\bf Parameter Efficient Fine-tuning (PEFT).} Full fine-tuning of modern large language models incurs a substantial amount of computational and storage resources and runs the risk of overfitting. To mitigate these, PEFT methods have been proposed, which attain or even outperform the performance of traditional fine-tuning methods while learning only a small number of parameters. Notably, prompt-tuning \cite{qin2021learning, liu2021p} and prefix-tuning \cite{li2021prefix} prepend learnable parameters to the input of the transformer. 

Likewise, adapters introduces learnable modules to the layers of transformers for each new task. On the other hand, low-rank adaptation (LoRA) \cite{hu2021lora} decomposes the weight update matrix into lower-ranked matrices and fine-tunes these low rank components. PEFT methods have also been proposed for vision tasks, such as visual prompt tuning \cite{jia2022visual}, which achieves comparable or better performance than full fine-tuning. However, PEFT methods are usually applied in static data settings while real-world scenarios require models to continuously adapt to new changes, necessitating approaches that allows PEFT to be applied dynamically over time. 

{\bf Parameter Efficient Continual Learning.}~With the advent of foundation models, various parameter efficient finetuning methods have been applied to efficiently adapt foundation models to new unseen task with a small set of learnable parameters in Continual Fine-tuning (CFT) settings. Previous works on prompt-based methods have focused on how to further design and retrieve prompts to reduce catastrophic forgetting. Particularly, L2P ~\cite{wang2022learning}, DualPrompt ~\cite{wang2022dualprompt} and CODAPrompt ~\cite{Smith_2023_CVPR} leverage prompt tuning for ViT in CL setting. Based on these works, HiDe-Prompt ~\cite{wang2023hide} enhances the performance by storing sample features alongside with learning the prompts. 

On the other hand, LoRA-based methods such as InfLoRA ~\cite{liang2024inflora}, Online-LoRA ~\cite{wei2024online}, SD-LoRA ~\cite{wu2025sd} leverage LoRA units to adapt to new tasks in CFT. InfLoRA and SD-LoRA apply the final accumulated LoRA weight to all tasks, hence they do not have any retrieval components, compromising task-specific performance.~More recently, O-LoRA \cite{wang2023orthogonal} also proposes a different approach to continually learn new LoRA units which are orthogonal to previous units for each new task.~However, unlike our approach, O-LoRA uses all LoRA units during inference rather than carefully retrieving the most relevant units for each new input.~This leads to a lack of representation adaptability which is similar to RanPAC~\cite{DBLP:conf/nips/McDonnellGPAH23} and consequently results in worse performance than our method -- see Table~\ref{tab:lora_position}.

Alternatively, another recent method, MELO~\cite{yu2024melo}, focuses instead on a different scheme of parameter-free retrieval approach based on clustering. However, it relies on carefully tuned clustering hyperparameters which fluctuates wildly across different pre-trained backbones. Furthermore, MELO was developed specifically for NLP applications using T5/BERT/GPT2 backbone, leaving it uncertain how to adapt its hyperparameters for vision tasks with ViT backbone. It also lacks mechanisms to mitigate knowledge interference (e.g., orthogonality constraints). In contrast, \textsc{Proteus} mitigates knowledge interference via both adaptive knowledge selection and orthogonal constraints, which results in significant performance improvement over both O-LoRA and MELO -- see Table \ref{tab:lora_position}. 

We evaluate the performance of O-LoRA and MELO on vision tasks by adapting their designs to the \textsc{Proteus} framework for a fair comparison.~For O-LoRA, we remove the retrieval component and disable knowledge reuse by setting $\boldsymbol{S} = 0$ while replacing our regularization with O-LoRA's. At inference time, all LoRA units are used rather than being retrieved selectively for each input. In contrast, for MELO, we replace its retrieval mechanism with our own and remove both the orthogonality constraints and the knowledge reuse component during training (setting $\boldsymbol{S} = 0$), thus reproducing MELO's behavior for a fair evaluation.~The reported results in Table~\ref{tab:lora_position} show that \textsc{Proteus} achieves much better performance than both O-LoRA and MELO across multiple benchmarks. 
\begin{table}[t]
    \centering
    \setlength{\tabcolsep}{12pt}
    \vspace{0.5cm}
    \caption{Performance comparison between our proposed method \textsc{Proteus} and its variants incorporating and adapting the continual fine-tuning approaches in O-LoRA~\cite{wang2023orthogonal} and MELO~\cite{yu2024melo}.~We note that these approaches were developed for NLP settings with T5/BERT/GPT2 backbone and need to be adapted to our CV settings with the ViT backbone.}
    \begin{tabular}{|lccr|}
        \toprule
        \textbf{CL Datasets }& \textbf{Proteus} & \textbf{O-LoRA} & \textbf{MELO}  \\
        \midrule
         CIFAR-100 & \textbf{94.53} & 54.55 & 92.98 \\
         ImageNet-R  & \textbf{82.17} & 53.93 & 78.98 \\
         \midrule
         ImageNet-A  & \textbf{65.29} & 32.11 & 38.47 \\
 
         VTAB5T  & \textbf{93.28} & 87.69 & 87.67 \\
         \bottomrule
    \end{tabular}
    
    \label{tab:lora_position}
\end{table}

\section{Parameter Comparison}
\label{sec:param_compare}
{\bf Definition.} We count a parameter as \textit{trainable} if it has \texttt{requires\_grad = True} while task $t$ is trained. The frozen ViT-B/16 backbone (85.8M parameters) is shared by all methods and never counted. Since code bases allocate classifier heads differently, we count for every method only the head rows of the current task, $769 \cdot C_t$ parameters (768 weights and one bias per class), where $C_t$ is the number of classes in task $t$; RanPAC uses a cosine classifier without bias, i.e., $768 \cdot C_1$. The per-task count is the sum of the method's trainable parameters at task $t$ and this head, and since it can change across tasks, e.g., with $C_t$ or with the growing pool of \textsc{Proteus}, we report its \textbf{maximum over all tasks}. All counts are computed from the configurations used in our experiments. Table~\ref{tab:param-count} reports the resulting counts.

\begin{table}[h]
\centering
\begin{small}
\caption{Maximum number of trainable parameters per task (in thousands), over all tasks of each benchmark.}
\label{tab:param-count}
\setlength{\tabcolsep}{2pt}
\begin{tabular}{|l| rrrrrr|}
\toprule
\textbf{Method} & \textbf{CIFAR-100} & \textbf{ImageNet-R} & \textbf{ImageNet-A} & \textbf{VTAB5T} & \textbf{VTAB-1k-5T} & \textbf{VTAB-1k-Sim50} \\
\midrule
L2P & 192.0 & 199.7 & 629.8 & 468.5 & 495.4 & 497.7 \\
DualPrompt & 952.3 & 960.0 & 960.0 & 464.6 & 491.6 & 499.2 \\
S-Prompt & 1,551.4 & 1,559.1 & 1,559.1 & 1,547.5 & 1,574.4 & 1,884.7 \\
CODA-Prompt & 3,847.7 & 3,855.4 & 3,855.4 & 3,847.7 & 3,874.6 & 3,723.3 \\
AdaPromptCL & 18,670.1 & 18,677.8 & 18,677.8 & 18,670.1 & 18,697.0 & 18,699.3 \\
HiDe-Prompt & 2,762.5 & 3,929.9 & 3,929.9 & 15,721.7 & 15,775.6 & 16,087.4 \\
\quad WTP stage & 391.7 & 1,551.4 & 1,551.4 & 1,543.7 & 1,570.6 & 1,880.1 \\
\quad TII stage & 2,370.8 & 2,378.5 & 2,378.5 & 14,178.1 & 14,205.0 & 14,207.3 \\
\midrule
InfLoRA & 192.0 & 199.7 & 199.7 & 192.0 & 218.9 & 221.2 \\
SD-LoRA & 376.4 & 384.0 & 384.0 & 376.4 & 403.3 & 405.6 \\
RanPAC & 1,197.3 & 221.2 & 221.2 & 213.5 & 213.5 & 241.2 \\
\midrule
\textsc{Proteus} & 156.0 & 163.7 & 608.7 & 451.2 & 477.3 & 479.6 \\
\bottomrule
\end{tabular}
\end{small}
\end{table}

{\bf Analysis.} Among retrieval-based methods, \textsc{Proteus} trains the fewest parameters per task on all six benchmarks. HiDe-Prompt trains $6.5$--$34.8\times$ more, mostly because its TII stage retrains an MLP task identifier at every task, CODA-Prompt trains $6.3$--$24.7\times$ more, and S-Prompt $2.6$--$9.9\times$ more. AdaPromptCL trains $30.7$--$119.7\times$ more under our definition, which counts its whole prompt pool since all of it has \texttt{requires\_grad = True}, although only a few prompt groups are used at each task; this figure is therefore an upper bound on its effective budget. The prompt pools of L2P and DualPrompt were enlarged to match the budget of \textsc{Proteus}, and their counts are within 3--4\% of \textsc{Proteus} on ImageNet-A and the VTAB benchmarks, so that they are compared at a comparable budget. Among retrieval-free methods, InfLoRA, SD-LoRA, and RanPAC train fewer parameters than \textsc{Proteus} on ImageNet-A and the VTAB benchmarks, where \textsc{Proteus} uses ranks 12 and 16, whereas \textsc{Proteus} trains the fewest parameters of all methods on CIFAR-100 and ImageNet-R, where it uses rank 4. In all cases, \textsc{Proteus} achieves substantially higher accuracy (Table~\ref{tab: main-table}) at a comparable or modestly larger budget than the retrieval-free methods.

{\bf Memory for non-gradient state.} The counts above exclude state that is computed in closed form or estimated from features and kept across tasks, which is not trained by gradient descent but does use memory. After each task, \textsc{Proteus} fits a DP-GMM with full covariances to at most 10{,}000 sampled embeddings of that task and keeps the active components (weight above $10^{-3}$). HiDe-Prompt additionally stores centroids of dimension 768 per class for classifier alignment. Signature storage is therefore the main memory cost of \textsc{Proteus} and is analyzed in App.~\ref{sec:mem_cost}, then compressing the covariances is a natural direction for future work (App.~\ref{limit}). Finally, \textsc{Proteus} and RanPAC use the same random-projection classifier with $d_{\mathrm{rp}} = 10{,}000$, which stores a fixed random matrix of $7.68$M values (regenerable from its seed), a Gram matrix of $5.0 \times 10^7$ unique values, and class prototypes of size $10{,}000$ per seen class, identically for both methods.

\section{Ablation on regularization parameter $\alpha$ in Eq.~\ref{eq:15_only}}\label{sec: alpha-lambda}

\begin{figure}[ht]
    \centering
    \includegraphics[width=0.6\linewidth]{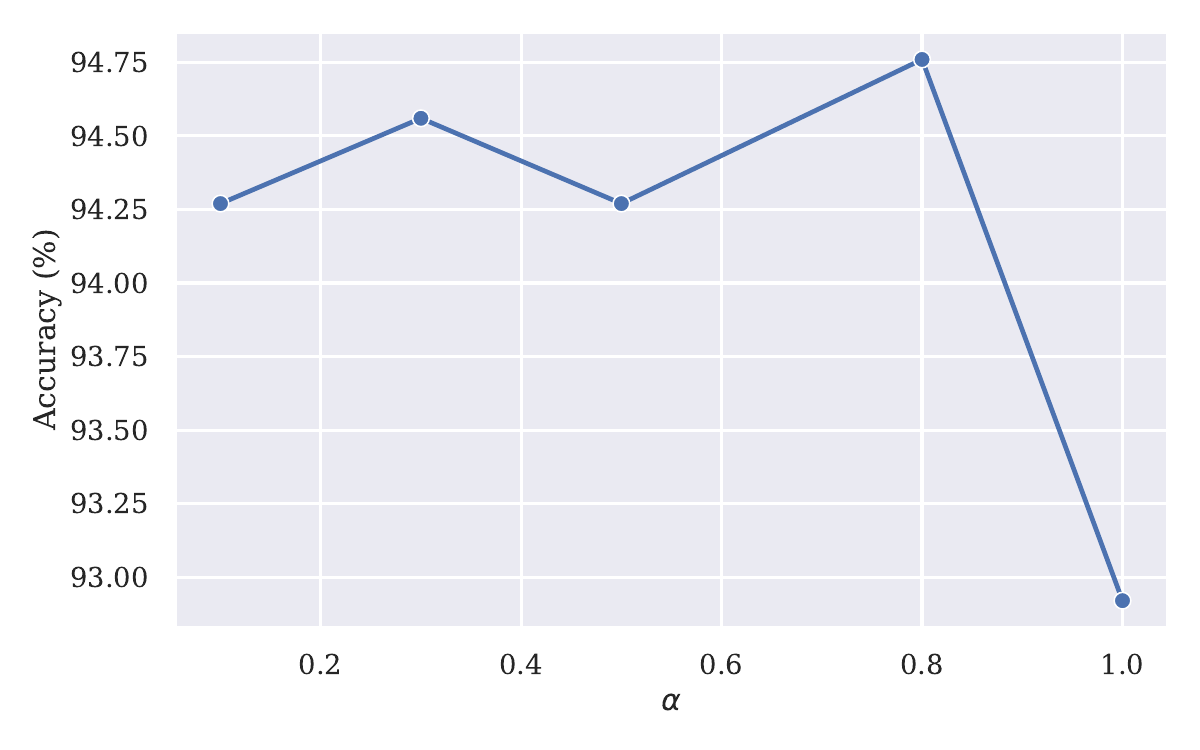}
    \caption{Performance of \textsc{Proteus} as $\alpha$ varies from 0.1 to 1.0.}
    \label{fig:alpha ablation VTAB}
\end{figure}

\par In this section, we investigate the sensitivity of \textsc{Proteus}'s performance to the regularization weight $\alpha$.~As defined in \eqref{eq:15_only}, $\alpha$ controls the trade-off between the $\ell_1$ and $\ell_2$ norm penalties where larger values of $\alpha$ emphasize sparsity while smaller values encourage stability and smoothness in the selection matrix $\boldsymbol{S}$. To ablate the impact of this trade-off, we fix all other hyperparameters as specified in Appendix~\ref{sec: data and hyper} and vary $\alpha$ across the set $\{0.1, 0.3, 0.5, 0.8, 1.0\}$. We conduct this ablation study on the VTAB5T benchmark.

As shown in Fig.~\ref{fig:alpha ablation VTAB}, the best performance is achieved at $\alpha = 0.8$, reaching a peak accuracy of \textbf{94.76}\%. In contrast, setting $\alpha = 1.0$, which corresponds to using only the $\ell_1$ norm, leads to a drop in accuracy.~This underscores the importance of incorporating the $\ell_2$ term to preserve stability and smoothness in $\boldsymbol{S}$. The results demonstrate that a balanced combination of sparsity and stability ($0 < \alpha < 1$), such as elastic loss, leading to more effective and generalizable representations.

\color{black}
\section{Limitations and Future Works}\label{limit}
{\bf Inference cost.}~Retrieval in \textsc{Proteus} scores each test input under every stored adapter, which requires one forward pass per adapter. Its inference cost therefore grows linearly with the number of tasks $n$, whereas methods with a single task-identity network, such as HiDe-Prompt, incur a roughly constant cost; App.~\ref{sec: complexity} reports the resulting per-batch latency. These passes are independent and can run in parallel, and the inverses and log-determinants of the signature covariances, as well as the classifier statistics, are precomputed once after training. A promising direction is to score only a subset of $B' < B$ inputs per batch under all adapters and to process the remaining inputs with the retrieved adapter alone, which reduces the cost from $nB$ to $nB' + B$ forward passes while, by Corollary~\ref{corollary}, keeping the batch error at most $\big(4p(1-p)\big)^{B'/2}$. An alternative is to first shortlists candidate tasks with a cheaper score could reduce the cost further.

{\bf Scope of batch-wise inference.}~Our batch-wise retrieval assumes that each test batch is drawn from a single task and that its inputs are independent, and Corollary~\ref{corollary} additionally requires a per-input error below one half. When inputs arrive individually, \textsc{Proteus} falls back to per-input retrieval, which remains competitive with HiDe-Prompt but is less accurate on difficult benchmarks, e.g., $43.3\%$ on ImageNet-A (see Figure~\ref{fig:batch}). Batches that mix several tasks are not covered by the current voting rule. A natural extension is to group the inputs of a batch by their per-input votes before aggregating, so that a mixed batch is split across tasks rather than assigned to a single one.

{\bf Storage.}~Each task adds a LoRA unit and a signature of $\tau$ Gaussian components, each with a full $d \times d$ covariance. Since a covariance has $d(d+1)/2$ free entries, about $296$k for $d = 768$, signatures can outweigh the adapters themselves, and total storage grows linearly with the number of tasks. Diagonal or low-rank-plus-diagonal covariances, or a single covariance shared across the components of a task, would reduce this cost while preserving the likelihood-based score.

{\bf Task boundaries.}~\textsc{Proteus} fits each signature after its task ends, so it requires task boundaries during training and does not directly apply to task-free or online continual learning. Updating signatures with streaming variational inference for DP-GMMs, and opening a new signature when incoming embeddings are poorly explained by the existing ones, could extend it to these settings.

{\bf Theoretical assumptions.}~Our guarantees are sufficient rather than necessary conditions, and they rely on clustered embeddings, a concentration condition with constant $K$, and covariance eigenvalues bounded below by $\lambda_0$ (Assumption~\ref{assumption:2}). The required separation also depends on the relative log-volumes of the signatures through $\kappa$, which, as a minimum over competing components, can only decrease as tasks accumulate. Moreover, \textsc{Proteus} does not optimize the separation factor $\delta$ directly, but encourages it through orthogonality between adaptation directions, a heuristic proxy. Designing training objectives that target $\delta$ directly, and deriving necessary conditions for reliable retrieval, are promising directions for future work.

\end{document}